\documentclass[11pt]{article}
\usepackage{graphicx}
\usepackage[margin=2cm]{geometry}
\usepackage{amsmath,amssymb,amsthm}
\usepackage{mathtools} 
\usepackage{mleftright}  
\usepackage{bm,bbm} 
\usepackage{thmtools}
\usepackage{thm-restate}
\usepackage{siunitx}
\usepackage{subcaption}
\usepackage{hyperref}
\usepackage{natbib}
\usepackage{url}

\usepackage{booktabs}
\usepackage{multirow}
\usepackage[table]{xcolor}
\usepackage{makecell}

\definecolor{cbaseline}{RGB}{198,239,206}   
\definecolor{crisk}{RGB}{235,140,150}        
\definecolor{ccost}{RGB}{255,220,110}        
\definecolor{ctradeoff}{RGB}{230,220,255}    
\definecolor{criskdown}{RGB}{252,224,227}    
\definecolor{ccostdown}{RGB}{255,245,200}    
\definecolor{ourblue}{RGB}{210,230,255}
\definecolor{cneutral}{RGB}{200,200,200}    
\definecolor{ourgrey}{RGB}{211,211,211}
\let\classAND\AND
\let\AND\relax
\usepackage{algorithmic}

\let\AND\classAND
\AtBeginEnvironment{algorithmic}{\let\AND\algoAND}

\let\classOR\OR
\let\OR\relax

\let\OR\classOR
\AtBeginEnvironment{algorithmic}{\let\OR\algoOR}
\usepackage{algorithm}

\usepackage{soul}
\usepackage[dvipsnames]{xcolor}

\usepackage{graphicx}
\graphicspath{{Figures/}}  
\usepackage{epstopdf}
\usepackage{pgfplots}
\pgfplotsset{compat=1.18} 

\usepackage{float}
\usepackage{wrapfig}
\usepackage{subcaption}

\newcommand{\cX}{{\mathcal X}}
\newcommand{\cY}{{\mathcal Y}}

\def\mc{\mathcal}

\def\mbf{\mathbf}

\DeclarePairedDelimiter\parens{\lparen}{\rparen}  
\DeclarePairedDelimiter\abs{\lvert}{\rvert}
\DeclarePairedDelimiter\babs{\bigg\lvert}{\bigg\rvert}
\DeclarePairedDelimiter\norm{\lVert}{\rVert}

\DeclarePairedDelimiter\bracks{\lbrack}{\rbrack}

\newcommand{\R}{\mathbb{R}}

\newcommand{\Indic}[1]{\mathbbm{1}\parens*{#1}} 

\DeclareMathOperator{\rank}{\mathrm{rank}}

\DeclareFontFamily{U}{matha}{\hyphenchar\font45}
\DeclareFontShape{U}{matha}{m}{n}{
      <5> <6> <7> <8> <9> <10> gen * matha
      <10.95> matha10 <12> <14.4> <17.28> <20.74> <24.88> matha12
      }{}
\DeclareSymbolFont{matha}{U}{matha}{m}{n}
\DeclareFontSubstitution{U}{matha}{m}{n}
\DeclareFontFamily{U}{mathx}{\hyphenchar\font45}
\DeclareFontShape{U}{mathx}{m}{n}{
      <5> <6> <7> <8> <9> <10>
      <10.95> <12> <14.4> <17.28> <20.74> <24.88>
      mathx10
      }{}
\DeclareSymbolFont{mathx}{U}{mathx}{m}{n}
\DeclareFontSubstitution{U}{mathx}{m}{n}
\DeclareMathDelimiter{\vvvert}{0}{matha}{"7E}{mathx}{"17}
\DeclarePairedDelimiterX{\matnorm}[1]
  {\vvvert}
  {\vvvert}
  {\ifblank{#1}{\:\cdot\:}{#1}}

\NewDocumentCommand{\expect}{ e{_} s o >{\SplitArgument{1}{|}}m }{%
  \operatorname{\mathbb{E}}
  \IfValueT{#1}{{\!}_{#1}}
  \IfBooleanTF{#2}{
    \expectarg*{\expectvar#4}%
  }{
    \IfNoValueTF{#3}{
      \expectarg{\expectvar#4}%
    }{
      \expectarg[#3]{\expectvar#4}%
    }%
  }%
}
\NewDocumentCommand{\expectvar}{mm}{%
  #1\IfValueT{#2}{\nonscript\;\delimsize\vert\nonscript\;#2}%
}
\DeclarePairedDelimiterX{\expectarg}[1]{[}{]}{#1}
\newcommand{\E}{\ensuremath{\mathbb{E}}}
\NewDocumentCommand{\prob}{ e{_} s o >{\SplitArgument{1}{|}}m }{%
  \operatorname{\mathbb{P}}
  \IfValueT{#1}{{\!}_{#1}}
  \IfBooleanTF{#2}{
    \probarg*{\probvar#4}%
  }{
    \IfNoValueTF{#3}{
      \probarg{\probvar#4}%
    }{
      \probarg[#3]{\probvar#4}%
    }%
  }%
}
\NewDocumentCommand{\probvar}{mm}{%
  #1\IfValueT{#2}{\nonscript\;\delimsize\vert\nonscript\;#2}%
}
\DeclarePairedDelimiterX{\probarg}[1]{(}{)}{#1}
\renewcommand{\P}{\ensuremath{\mathbb{P}}}
\newcommand{\condon}{\,\ifnum\currentgrouptype=16 \middle\fi|\,} 

\newcommand{\iid}{\overset{\mathrm{i.i.d.}}{\sim}}

\newcommand{\Unif}{\mathsf{Unif}}

\newcommand{\largeset}{warehouse}
\newcommand{\smallset}{retail}

\newcommand{\datadist}{\pi}
\newcommand{\traindist}{\tau}
\newcommand{\param}{\theta}
\newcommand{\rparam}{\boldsymbol{\param}}
\newcommand{\hypclass}{\mc{F}}

\newcommand{\rtestlabel}{\mbf{y}}
\newcommand{\Ntrain}{N}
\newcommand{\dtrain}{\mc{D}_{\mathrm{train}}}
\newcommand{\Ntest}{N'}
\newcommand{\dtest}{\mc{D}_{\mathrm{test}}}
\newcommand{\rtest}{\mbf{x}}

\newcommand{\rf}{\mbf{f}}

\newcommand{\Nretail}{M}
\newcommand{\Nlib}{T}
\newcommand{\nogaussianpts}{S}
\newcommand{\nogaussianptslib}{P}
\newcommand{\nsampretail}{Q}
\newcommand{\nretailind}{k}

\newcommand{\localtest}{v}

\newcommand{\rlocaltest}{\mbf{v}}
\newcommand{\localtestmat}{{V}}
\newcommand{\rlocaltestmat}{\mbf{V}}

\newcommand{\gapmargin}{\mathrm{g}_{\mathrm{marg}}}
\newcommand{\gaplmargin}{\mathrm{g}_{\mathrm{lmarg}}}
\newcommand{\gapvar}{\mathrm{g}_{\mathrm{var}}}
\newcommand{\gaplvar}{\mathrm{g}_{\mathrm{lvar}}}
\newcommand{\rgapvar}{\mathrm{\mbf{g}}_{\mathrm{var}}}
\newcommand{\gap}{\mathrm{g}_{\mathrm{measure}}}
\newcommand{\gapl}{\mathrm{g}_{\mathrm{lmeasure}}}
\newcommand{\gapcondr}{\mathrm{g}_{\mathrm{A}}}
\newcommand{\gaplcondr}{\mathrm{g}_{\mathrm{C}}}
\newcommand{\gaprandr}{\mathrm{g}_{\mathrm{B}}}
\newcommand{\gaplrandr}{\mathrm{g}_{\mathrm{D}}}

\newcommand{\rashexpctparam}{\epsilon_2}
\newcommand{\rashlexpctparam}{\epsilon_5}
\newcommand{\slackparam}{\epsilon_{3}}
\newcommand{\slackgap}{\epsilon_{4}}
\newcommand{\slackmeasure}{\epsilon_{1}}
\newcommand{\testdiff}{\delta}
\newcommand{\rtestdiff}{\bm{\testdiff}}
\newcommand{\allavg}{\mbf{\Delta}}
\newcommand{\expallavg}{\overline{\Delta}}
\newcommand{\condavg}{h_{\mathrm{A}}}
\newcommand{\lcondavg}{h_{\mathrm{B}}}
\newcommand{\explallavg}{\overline{\Gamma}}
\newcommand{\lallavg}{\mbf{\Gamma}}

\newcommand{\baseline}{f_{\mathrm{base}}}
\newcommand{\modelerrorlabel}{e}

\newcommand{\riskassess}{r}

\DeclareMathOperator{\ratio}{\mathrm{ratio}}
\DeclareMathOperator{\portion}{\mathrm{frac}}
\DeclareMathOperator{\median}{\mathrm{median}}
\newcommand{\auditcostratio}{\ratio_{\mathrm{cost}}}
\newcommand{\auditriskratio}{\ratio_{\mathrm{risk}}}
\newcommand{\auditdivfrac}{\portion_{\mathrm{div}}}
\newcommand{\auditriskfrac}{\portion_{\mathrm{risk}}}
\newcommand{\dtestconf}{\mc{D}^{\mathrm{cons}}_{\mathrm{test}}}
\newcommand{\dtestlocalconf}{\mc{D}^{\mathrm{localcons}}_{\mathrm{test}}}

\newcommand{\randomalgo}{\mc{T}}
\DeclareMathOperator{\NN}{\mathrm{NN}}

\newcommand{\rashomonparam}{\epsilon}

\newcommand{\consistency}{\mathrm{c}_{\sigma}}
\DeclareMathOperator{\margin}{\mathrm{mg}}
\DeclareMathOperator{\lmg}{\mathrm{lmg}}
\DeclareMathOperator{\lmc}{\mathrm{lmc}}
\newcommand{\localmargin}{\lmg_{\sigma}}
\newcommand{\localmarginconsistency}{\lmc_{\sigma}}
\newcommand{\localconsistency}{\lmc_{\sigma}}
\newcommand{\pairwisedisag}{\mathrm{pd}}
\newcommand{\disagens}{\mathrm{d}}
\newcommand{\localdisagens}{\mathrm{ld}_{\sigma}}
\newcommand{\predvar}{\mathrm{pv}}
\newcommand{\predrange}{\mathrm{pr}}

\newcommand{\localvar}{\mathrm{lv}_{\sigma}}
\newcommand{\hammancons}{\mathrm{hc}_{\sigma}}

\newcommand{\perfmet}{\hat{\operatorname{err}}} 

\newcommand{\normtwo}[1]{\norm*{#1}_{2}}

\newtheorem{lemma}{Lemma}

\newtheorem{definition}{Definition}
\newtheorem{assumption}{Assumption}

\newcommand{\ltwohsphere}{\mathbb{S}^{(d)}_{2}}

\newcommand{\emprashomonset}{\mc{R}}
\newcommand{\retailset}{{R}}
\newcommand{\rretailset}{\mbf{\retailset}}

\newcommand{\ensmretset}{\bar{{f}}}

\newcommand{\loss}{\ell}

\usepackage[color=red!30,colorinlistoftodos]{todonotes}

\usepackage{comment}

\title{Measuring consistency via ensemble margin and local prediction variability: Auditing decision systems in the presence of predictive multiplicity}
\author{Sinjini Banerjee$^{1}$, Tim Marrinan$^{2}$, Anand D. Sarwate$^{1}$ \\ $^{1}$Department of Electrical and Computer Engineering, Rutgers University \\ $^{2}$Pacific Northwest National Lab \\ \texttt{sb1977@rutgers.edu}, \texttt{timothy.marrinan@pnnl.gov}, \texttt{anand.sarwate@rutgers.edu}}

\begin{document}

\maketitle

\begin{abstract}
The Rashomon effect is a machine learning phenomenon where equally accurate models produce different predictions for the same inputs (predictive multiplicity). Existing work primarily focuses on multiplicity within individual models, but in more complex decision systems, the impact of the Rashomon effect is less well understood.  In this work, we study multiplicity from the perspective of auditing incorrect ensemble predictions, where the decision to divert an instance for human review is based on a consistency criterion that combines the ensemble margin with a measure of local prediction variability for each constituent model. With mild assumptions about stability and smoothness, we show that the consistency scores of finite ensembles converge to the corresponding consistency score of the expected model from the Rashomon set as the ensemble size and the number of samples used to measure local prediction variability increase. To demonstrate the efficacy of the proposed criterion, we evaluate the framework with respect to transformer models applied to natural language understanding tasks and parameter-efficient fine-tuning of large language models used for tabular data classification tasks. Our experiments show that ensembling models from the Rashomon set substantially reduces the risk of incorrect predictions going unchecked compared with auditing a single model, while incurring only a moderate increase in the number of diversions. Moreover, the auditing behavior of the full Rashomon set can be closely approximated by finite ensembles of relatively modest size, with the risk approaching zero for some datasets. We further demonstrate that the proposed measure exhibits stronger agreement with established predictive multiplicity metrics than existing consistency measures, providing a more reliable way to capture multiplicity in the Rashomon set.
\end{abstract}

\section{Introduction} \label{sec:intro}

Centralized machine learning marketplaces have become a critical component of the modern AI ecosystem, providing practitioners with access to a wide range of pre-trained models for diverse applications. Platforms such as Hugging Face Hub~\citep{wolf2019huggingface}, AWS Marketplace~\citep{awsmarketplace}, and NVIDIA NGC~\citep{nvidia_ngc_2026} host thousands of models, allowing users to rapidly deploy state-of-the-art systems without training models from scratch. An underlying assumption in model selection from these marketplaces is often that selecting the ``best'' model for a task is straightforward, where ``best'' is defined by performance metrics such as accuracy or F1-score. This assumption overlooks a fundamental property, known as the \emph{Rashomon effect}, which describes the existence of multiple equally well-performing models achieving different predictions on a held-out test set~\citep{breiman2001statistical}. 

In over-parameterized networks, the ``black box" nature of these models makes it difficult to examine the flow of information through their billions of parameters and interpret their decisions. This lack of transparency can lead to low trust and confidence in these systems, creating challenges for their deployment in safety-critical domains such as healthcare \citep{peterson2019machine,he2019practical}, and autonomous vehicles \citep{perez2024artificial}. From a data-centric viewpoint, the Rashomon effect leads to a phenomenon known as \emph{predictive multiplicity}, where competing models can disagree on individual test instances~\citep{marx2020predictive}. In high-stakes domains such as healthcare, finance, and criminal justice, this instability can lead to inconsistent diagnoses, credit approvals, or risk assessments, undermining both fairness and accountability~\citep{dai2025intentional}. These costs motivate the need for tools that identify when predictions are unstable and should be treated with caution.

One way to mitigate predictive multiplicity is through ensembling. By aggregating predictions through multiple models, ensembles are expected to reduce variance, and, in the limit of large ensemble sizes, converge to the expected prediction under the model distribution. Deep ensembles are widely used in practice to improve accuracy and estimate uncertainty~\citep{lakshminarayanan2017simple, fort2020deep,ganaie2022ensemble,rahaman2021uncertainty}. In practical scenarios, ensembles are often small, so predictive inconsistency may persist across independently sampled ensembles, even when the constituent models in each ensemble are individually near-optimal.

This paper focuses on understanding and mitigating ensemble inconsistency arising from predictive multiplicity in small ensembles. While prior work has studied multiplicity across individual models, there is limited understanding of instance-level consistency guarantees for finite ensembles drawn from a Rashomon set. Moreover, consistency does not imply correctness: models may agree yet still be wrong, leading to missed diversions of erroneous predictions. We take \emph{auditing with diversion} as our motivating application: instances can either be classified by the model or be diverted to an expert. Such systems must balance \emph{risk} (the model's error) against the \emph{cost} of using the expert. We show that appropriately leveraging ensembles can reduce missed diversions and improve decision reliability.

\subsection{Auditing and diversion to human experts}\label{sec: Auditing and diversion}

Modern machine learning models are abundant; in many cases, they achieve state-of-the-art performance on complex tasks. However, even the best-performing models may eventually make mistakes on unseen data. Suppose that there exists an idealized \emph{baseline model}, $\baseline$, that is the ``gold-standard'' for a given deployment scenario. Despite being the best model, $\baseline$ will inevitably make a prediction that is contradicted by domain experts. To reduce the risk of making bad decisions due to model errors, we would like to implement an automated ``auditing'' system to flag untrustworthy predictions and divert those decisions to a human expert for review. Human intervention ensures that edge cases and instances where the model is uncertain are handled using expert judgment. However, deciding when to intervene can be difficult, and mistakes can be costly. Failing to intervene when the baseline model makes an incorrect prediction can harm people whose lives are affected by the decisions, but deferring to an expert is also expensive and may be infeasible given the number of available experts. Since ground-truth labels are typically not available in practice, proxy measures derived from the model may help us decide when to divert.

We consider a binary classification example, where $x \in \mathbb{R}^d$ has label $y \in \{-1,+1\}$. Let $f: \mathbb{R}^d \to [0,1]$ be a predictive model such that $f(x)$ is the estimated probability that $x$ has label $y=+1$, and let $h(x;f) = \Indic{f(x)\geq \frac{1}{2}}$ be an associated ``hard'' decision threshold for classifying $x$. The error of using $f$ to predict $y$ from $x$ can be written as $\modelerrorlabel(x;f) = \Indic{h(x,f)\neq y}.$

Now suppose we have an automated auditing system, $\riskassess: \mathcal{F} \times \mathbb{R}^d \to \{0,1\},$ that takes in a function, $f\in \mathcal{F},$ and an input instance, $x \in \mathbb{R}^d$, and returns a risk assessment for the prediction, $f(x)$. The auditing system returns $\riskassess(x;f)=0$ if the prediction made by the model is trustworthy, and it returns $\riskassess(x;f)=1$ if the case should be diverted to a domain expert for review, where we assume it can be correctly labeled by the reviewer. Discrepancies between $\modelerrorlabel(x;f)$ and $\riskassess(x;f)$ result in unnecessary costs to the system. When $\modelerrorlabel(x;f) = 0$ but $\riskassess(x;f) = 1$, the auditing system has diverted a case unnecessarily, creating extra work for the reviewers.  We refer to this as an ``audit cost''.  In contrast, when $\modelerrorlabel(x;f) = 1$ but $\riskassess(x;f) = 0$, the auditing system fails to flag an incorrect prediction, creating potential downstream consequences. We refer to this scenario as an ``audit risk.'' Audit risks and audit costs can be considered as errors of the auditing system itself.

Let us return to our idealized baseline model, $\baseline$. It may not be feasible to deploy the gold-standard model in all scenarios due to cost, resource availability, or other constraints. In this case, we may need to deploy a cheaper alternative, which we will refer to as the \emph{candidate model}, $f.$ Regardless of how well $f$ approximates $\baseline,$ both systems are liable to make mistakes on some test instances. Thus, the auditing system $\riskassess$ is used to reduce the risk of errors. The central problem of this work is how to design an auditing system, $\riskassess$, and a candidate function, $f$, such that the same set of samples is diverted for human review, whether we use $\baseline$ or $f$. In other words, we want $\riskassess(x;f) \approx \riskassess(x;\baseline)$, while minimizing the \emph{excess} audit costs and audit risks.

Comparing the behavior of an auditing system applied to the candidate model and the baseline model gives rise to 16 possible outcomes, corresponding to all combinations of $(\modelerrorlabel(x;f)$, $\riskassess(x;f))$ for the candidate and $(\modelerrorlabel(x;\baseline)$, $\riskassess(x;\baseline))$ for the baseline. Table~\ref{tab: combined audit outcomes} organizes these outcomes in a 4×4 matrix, where each row corresponds to one candidate outcome and each column corresponds to one baseline outcome. Each cell summarizes the difference in audit cost and audit risk for the candidate model relative to the baseline.

The most desirable outcomes are those in which the candidate meets the baseline, meaning that it reaches the same auditing decision as the baseline. These include cases where both systems correctly automate a prediction, both correctly divert it, or both incur the same unavoidable audit cost or audit risk. In Table \ref{tab: combined audit outcomes}, these outcomes are represented by the diagonal cells. We refer to these costs and risks as unavoidable because they arise from the behavior of the baseline model itself. For example, if a baseline prediction is diverted that was ultimately determined to be correct, or an erroneous baseline prediction is not diverted, the resulting audit risk and cost cannot be eliminated. Disagreements between the two systems instead introduce excess audit cost, excess audit risk, or a mixed outcome in which one quantity improves at the expense of the other. Some of these disagreements correspond to reduced audit cost or reduced audit risk, where the candidate performs better than the baseline on a particular instance. The objective of an auditing strategy is to minimize excess costs and risks while matching the baseline as closely as possible. In our experiments in Section~\ref{sec: Retailauditoutcomes} and Appendix~\ref{sec: Combinedoutcomes}, we observe that the candidate model typically achieves a lower audit cost than the baseline, but only at the expense of a higher audit risk.

\begin{table*}[t]
\centering
\renewcommand{\arraystretch}{1.4}
\setlength{\tabcolsep}{3pt}

\caption{Matrix of auditing outcomes comparing the audit system of a candidate, $f$, with the audit system of a baseline, $\baseline$. Rows correspond to the candidate outcome $(\modelerrorlabel(x;f)$, $\riskassess(x;f))$ and columns correspond to the baseline outcome $(\modelerrorlabel(x;\baseline)$, $\riskassess(x;\baseline))$. Each cell summarizes how the candidate changes audit risk and audit cost relative to the baseline. Green cells indicate agreement with the baseline, red cells indicate increased audit risk relative to baseline, light pink cells indicate reduced audit risk relative to  the baseline, yellow cells indicate increased audit cost relative to the baseline, and light yellow cells indicate reduced audit cost relative to baseline. Neutral outcomes in gray cells indicate that baseline and candidate models disagreed on their predicted labels but both correctly diverted or correctly automated the instances, and purple cells indicate mixed outcomes when baseline and candidate models disagreed on their predicted labels but both incorrectly diverted or incorrectly automated the instances.}
\label{tab: combined audit outcomes}

\begin{tabular}{c|>{\centering\arraybackslash}m{2.2cm}
                  |>{\centering\arraybackslash}m{3.2cm}
                  |>{\centering\arraybackslash}m{2.2cm}
                  |>{\centering\arraybackslash}m{3.2cm}}

&
\multicolumn{4}{c}{\textbf{Baseline model} $(\baseline)$ - $(\modelerrorlabel(x; \baseline)$, $\riskassess(x; \baseline))$}\\
\cline{2-5}

\makecell{\textbf{Candidate}\textbf{ model} $(f)$ - \\$(\modelerrorlabel(x; f)$, $\riskassess(x; f))$ }
&
$(0,0)$
&
$(0,1)$
&
$(1,0)$
&
$(1,1)$
\\
\hline

$(0,0)$
&
\cellcolor{cbaseline}
\makecell{\textbf{Outcome 0}\\Baseline}
&
\cellcolor{ccostdown}
\makecell{\textbf{Outcome 4}\\Cost $\downarrow$}
&
\cellcolor{criskdown}
\makecell{\textbf{Outcome 8}\\Risk $\downarrow$}
&
\cellcolor{cneutral}
\makecell{\textbf{Outcome 12}\\Neutral}
\\
\hline

$(0,1)$
&
\cellcolor{ccost}
\makecell{\textbf{Outcome 1}\\Cost $\uparrow$}
&
\cellcolor{cbaseline}
\makecell{\textbf{Outcome 5}\\Baseline}
&
\cellcolor{ctradeoff}
\makecell{\textbf{Outcome 9}\\Risk $\downarrow$\\Cost $\uparrow$}
&
\cellcolor{ccost}
\makecell{\textbf{Outcome 13}\\Cost $\uparrow$}
\\
\hline

$(1,0)$
&
\cellcolor{crisk}
\makecell{\textbf{Outcome 2}\\Risk $\uparrow$}
&
\cellcolor{ctradeoff}
\makecell{\textbf{Outcome 6}\\Risk $\uparrow$\\Cost $\downarrow$}
&
\cellcolor{cbaseline}
\makecell{\textbf{Outcome 10}\\Baseline}
&
\cellcolor{crisk}
\makecell{\textbf{Outcome 14}\\Risk $\uparrow$}
\\
\hline

$(1,1)$
&
\cellcolor{cneutral}
\makecell{\textbf{Outcome 3}\\Neutral}
&
\cellcolor{ccostdown}
\makecell{\textbf{Outcome 7}\\Cost $\downarrow$}
&
\cellcolor{criskdown}
\makecell{\textbf{Outcome 11}\\Risk $\downarrow$}
&
\cellcolor{cbaseline}
\makecell{\textbf{Outcome 15}\\Baseline}
\\
\hline

\end{tabular}
\end{table*}

\subsection{Our contributions}

We develop a framework for balancing the costs and risks of diverting to human experts by using an ensemble of candidate models within an auditing system. To do this, we leverage the framework of predictive multiplicity and the Rashomon set of models.  Prior work has largely focused on disagreement between individual models. We investigate how multiplicity propagates to downstream auditing systems and how finite ensembles can be used to improve reliability. Our baseline model $\baseline$ is the expected model over the Rashomon distribution, while a candidate model $f$ is an ensemble of finitely many models drawn from the Rashomon set. We can use this to assess how ensemble size, risks, and costs are related in auditing systems.  We call this choice of baseline a \emph{\largeset{} model} and the finite ensemble a \emph{\smallset{} model} as a metaphor for how the two ensembles are related. Our contributions are listed below.
\begin{itemize}
    \item We propose a new consistency measure, $(\beta,\sigma)$-consistency, that uses ensemble margin information and local perturbation stability in the embedding space of an input. This differs from prior measures of multiplicity that analyze a single candidate model. Sampling models allow us to study how well the \smallset{} set approximated the \largeset{} model by reducing variance induced by predictive multiplicity.
    \item Under stability and local Lipschitz assumptions, we provide probabilistic guarantees showing that if an instance has a high $(\beta,\sigma)$-consistency score with respect to the baseline model, then finite retail set ensembles will also assign high consistency scores to these instances with probability increasing exponentially in both the ensemble size $M$ and the number of perturbation samples $S$.
    \item We operationalize $(\beta,\sigma)$-consistency in an auditing application in which we threshold consistency scores to decide whether predictions should be diverted for expert review. 
    \item We empirically validate our results on fine-tuning applications with large language models. Our results quantify how increasing the number of models in the ensemble reduces predictive multiplicity and audit cost variability, while decreasing audit risk.
\end{itemize}

\subsection{Related work}

Our paper incorporates ideas from several different bodies of prior work.

\paragraph{Learning-to-reject and learning-to-defer.}
In the selective classification or learning-to-reject literature~\citep{herbei2006classification,bartlett2008classification,yuan2010classification, feng2022towards}, a classifier may abstain from making predictions on instances that are considered difficult or unreliable. Learning-to-defer extends this framework by explicitly incorporating the predictions or expertise of the downstream decision-maker. In many selective classification systems, the decision to reject a prediction is often implemented by thresholding confidence or uncertainty-based measures. Common choices include maximum softmax probability, which captures confidence in the top prediction, the margin between the top two classes, and entropy, which reflects overall uncertainty in the predictive distribution (see recent surveys by \citet{Zhang2023survey} and \citet{Hendrickx2024survey}). Applying a threshold to these measures determines if the instance is diverted. 

\paragraph{Rashomon sets and Predictive Multiplicity.} \label{sec: Rashomon sets and Predictive Multiplicity}
Modern machine learning models, particularly deep neural networks, operate in overparameterized regimes where many distinct solutions achieve nearly identical empirical performance. This Rashomon effect, first described by~\citet{breiman2001statistical}\footnote{Named after the movie \emph{Rash\={o}mon}, directed by Akira Kurosawa and based on the short story \emph{In a Grove} by Ry\={u}nosuke Akutagawa.} implies that empirical risk minimization yields a set of near-optimal models rather than a unique predictor, formalized as the \emph{Rashomon set}~\citep{fisher2019all,Semenova_2022,ciaperoni2024efficient,xin2022exploring}. Although these models are indistinguishable under aggregate metrics, they can produce different predictions or confidence scores on individual inputs, leading to predictive multiplicity.

This phenomenon has been widely studied in classification settings~\citep{marx2020predictive,watson2023predictive,black2022model,hsu2024dropout,zhu2024predictive}, including its dependence on data complexity~\citep{cavus2024investigating} and information-theoretic characterizations~\citep{hsu2022rashomon}. Predictive multiplicity arises even with identical architectures and data due to non-convex optimization and training stochasticity~\citep{gundersen2023sources}, with factors such as random initialization leading to substantially different learned functions despite similar accuracy~\citep{banerjee2024measuring,summers2021nondeterminism,jordan2023calibrated,bouthillier2021accounting,picard2021torch,henderson2019deep,melis2018state,reimers-gurevych-2017-reporting,dodge2020finetuning}. In high-stakes settings, such variability can lead to inconsistent decisions, undermining reliability and trust.

\paragraph{Consistency Analysis in Multiplicity.}
In the multiplicity literature, consistency across models has been used to characterize the stability of individual predictions. Consistency has been studied through model agreement~\citep{marx2020predictive, black2022consistent}, probabilistic prediction variance across the Rashomon set~\citep{hsu2022rashomon, watson2023predictive}, stochastic perturbations such as dropout~\citep{hsu2024dropout}, invariance of counterfactual explanations~\citep{black2022model,hamman, KR2023-78}, and stability under fine-tuning variability in large language models~\citep{hamman2024quantifying}. Some of the works provide probabilistic guarantees on a consistency measure, ensuring that one model’s decision will not be invalidated by another. A key insight from these works is that consistency depends not only on confidence but also on local stability. Predictions far from the decision boundary and robust to small input perturbations are more likely to remain invariant across models.

\paragraph{Local perturbation analysis.}
Local perturbation analysis has been widely used in the machine learning literature as a principled way to probe and quantify robustness and consistency of machine learning models, especially in over-parametrized settings. For instance, randomized smoothing yields classifiers with provable local robustness guarantees~\citep{cohen2019certified}, while combining smoothing with ensembles improves stability under perturbations~\citep{horvath2022boosting}. Similarly, in the multiplicity literature, stability of counterfactual explanations depends on local smoothness~\citep{black2022consistent}, and Lipschitz-based analyses provide probabilistic guarantees for prediction invariance under model perturbations~\citep{hamman}. Sampling local neighborhoods in representation space further enables efficient estimation of agreement across fine-tuned models~\citep{hamman2024quantifying}.

Despite their promise, existing consistency measures in the multiplicity literature have important limitations. First, many approaches rely on hard labels, such as model disagreement, discarding richer information contained in predictive confidence. Second, while these methods are effective at identifying consistent test instances, they can misclassify inconsistent instances as consistent, leading to high missed diversions. More fundamentally, consistency does not imply correctness: multiple models may agree on a prediction and still be jointly wrong. As a result, consistency-based approaches may fail to divert erroneous predictions, which is particularly problematic in high-stakes settings where missed errors carry significant risk.

\paragraph{Ensembling.}\label{sec:related:ensemble}

A widely adopted approach for mitigating predictive multiplicity is ensembling, in which predictions from multiple models are aggregated, typically by averaging their outputs. The core idea behind deep ensembling is that a group of diverse baseline predictors often outperforms a single predictor by reducing variance, mitigating overfitting, and capturing multiple modes of the hypothesis space. Lakshminarayanan et al. popularized deep ensembles as a scalable alternative to Bayesian neural networks, showing that ensembles of independently trained models provide strong empirical uncertainty estimates and improved robustness to distribution shift~\citep{lakshminarayanan2017simple}. Subsequent work has analyzed deep ensembles from a loss-landscape perspective, demonstrating that random initialization leads to convergence to distinct modes that are effectively exploited through the weighted averaging of predictions~\citep{fort2020deep}.
Empirically, deep ensembles differing only in random seed have been shown to perform competitively across vision and language tasks~\citep{lakshminarayanan2017simple,lee2015m}, while alternative ensemble construction methods such as bootstrapping may not offer additional benefits in deep learning settings~\citep{nixon2020bootstrapped}. Ensembles have also been studied as variance-reduction tools in randomized smoothing frameworks~\citep{horvath2022boosting}, further reinforcing their role in reducing predictive multiplicity. Ensembles of fine-tuned language models have also been shown to improve predictive performance and uncertainty estimation~\citep{wang2023loraensembleslargelanguage, arango2024ensemblingfinetunedlanguagemodels}.

Despite the benefits of ensembling, sampling a small number of models from the Rashomon set can introduce significant Monte Carlo error in the ensemble's estimate of the expected prediction. As a result, two ensembles of the same size, drawn independently from the same training procedure, can differ in their constituent models enough that ensemble prediction lies near a decision boundary for certain test points, making the predicted label unstable. While asymptotically, large ensemble predictions converge to posterior predictive decision and eliminate variability by law of large numbers, practical constraints (eg cost, latency, licensing) often limit ensembles to small sizes, where variance due to finite sampling remains non-negligible. Thus predictive inconsistency across ensembles is an inherent finite sample phenomenon, and motivates the need for instance-specific measures that can identify which instances cannot be trusted to produce stable predictions for fixed-size ensembles.

\section{Problem Setup}\label{sec: Problem Setup}

\noindent \textbf{Notation:} For a positive integer n, let $[n] = \{1,2,\ldots,n\}$. Random variables will be denoted in boldface, with realizations being non-bolded, so $\rtest$ is a random variable and $x$ is a realization. 

\subsection{Binary classification}\label{sec: bc DNN}

We consider the standard statistical learning setting for binary classification\footnote{We focus on binary classification to keep our notations simple. Our setting can be extended to multi-class classification by fixing any class and analyzing the predicted probability of that class.}. Let $\cX = \R^d$ and $\cY = \{-1,+1\}$ be a space of feature vectors and labels, respectively. Let $\dtrain = \{ (x_i, y_i) \in \cX \times \cY \colon i \in [\Ntrain] \}$ denote the training data. Here $x_{i}$ denotes the $i$-th realization of the feature vector $\mbf{x} \in \cX$ and $y_{i}$ the label corresponding to $x_{i}$. We assume the pairs in $\dtrain$ are generated by independent and identically distributed (i.i.d.) draws from an unknown distribution $\datadist$ on $\cX \times \cY$.

A training algorithm $\randomalgo$ is a randomized map from $\cX^{\Ntrain} \times \cY^{\Ntrain} \to \hypclass$, where $\hypclass = \{f_{\param}(x): \cX \to [0,1] \colon \param \in \varTheta \}$ is a set of models parameterized by $\param \in \varTheta$. We interpret the model $f_{\param} \in \hypclass$ as approximating a probability distribution $\hat{\datadist}(y | x)$ on $\cY$ where $\hat{\datadist}(y = 1 | x) = f_{\param}(x)$. The function $f_{\param}$ would correspond to the output of a softmax operation on logit values associated with $\cY$.  We define $h_{\param} = \Indic{f_{\param} \geq \frac{1}{2}}$ as the ``hard'' decision function associated with $f_{\theta}$.

We consider the setting where $\randomalgo$ selects parameters/weights for a fixed DNN architecture by minimizing a loss function, $\loss : \mc{Y} \times \mc{Y} \to \mathbb{R}$ (for example, cross entropy loss), resulting in the population risk
    \begin{align}
    L(f_{\param},\mbf{x},\mbf{y}) = \E_{\rtest,\rtestlabel \sim\datadist}[\loss(f_{\param}(\mbf{x}),\mbf{y})].
    \end{align}
The data-generating distribution $\datadist$ is unknown, so we replace the population risk with the empirical risk for training,
    \begin{align}
    \hat{L}(f_{\param},\dtrain) = \frac{1}{\Ntrain}\sum_{i=1}^{\Ntrain}(\loss(f_{\param}({x_{i}),y_{i}})).
    \end{align}
Since the training algorithm itself may be randomized, we can think of it as sampling a function $\rf_{\rparam} \in \hypclass$.
 
We assume that we are also given a test set $\dtest \in \cX^{\Ntest} \times \cY^{\Ntest}$ also sampled $\iid$ from $\datadist$. The test error for model $f_{\param}$ with a decision function $h_{\theta}$ is
\begin{align}\label{eq: empirical error}
    \perfmet(f_{\param}) = \frac{1}{\Ntest}\sum_{i=1}^{\Ntest}{[\Indic{h_{\param}({x_{i}}) \neq {y_{i}}})]}.
\end{align}
For notational simplicity, we henceforth suppress the parameter $\param$ and use $f$ to denote $f_{\theta}$.

\subsection{The empirical Rashomon set and predictive multiplicity}\label{sec: Empirical Rashomon sets}

We are interested in the \emph{Rashomon set} of models~\citep{breiman2001statistical,Semenova_2022}, which captures a class of competing models with similar aggregate predictive performance. Following~\citet{hamman2024quantifying}, we define the Rashomon set relative to a reference model $f_0$ with satisfactory predictive performance.

\begin{definition}[Empirical Rashomon set]\label{def: empirical rashomon set}
Given a hypothesis class $\hypclass$, a dataset $\dtest= \{{(x_i,y_i)},  i \in [\Ntest]\}$, a reference model $f_0\in\hypclass$, the empirical error  $\perfmet(f_{0})$, defined in Equation~\ref{eq: empirical error}, and an error tolerance $\rashomonparam\in(0,1)$, the $\rashomonparam$-empirical Rashomon set is
\begin{align}
\emprashomonset(\rashomonparam) = \{f\in\hypclass:
\perfmet(f)
\leq
\perfmet(f_0)+\rashomonparam\}.
\end{align}
\end{definition}

Once again, we suppress the dependence on $\rashomonparam$ and write $\emprashomonset$ in place of $\emprashomonset(\rashomonparam)$ to simplify our notations. The Rashomon set $\emprashomonset$ contains models whose empirical error is within $\rashomonparam$ of the reference model $f_0$. The choice of $\rashomonparam$ determines the range of models considered to be equally well-performing and is application dependent~\citep{marx2020predictive,watson2023predictive}.
Models in the Rashomon set may nevertheless produce different predictions on individual instances. In particular, an instance that is misclassified by one model may be correctly classified by another, while some instances may receive different predictions across many models in $\emprashomonset$. This phenomenon, known as \emph{predictive multiplicity}~\citep{marx2020predictive}, can be quantified in several ways. We review the multiplicity metrics considered in this work in Section~\ref{sec: Multiplicity metrics}.

\subsection{Ensembles using the Rashomon set}

We propose using ensembles of models sampled from the empirical Rashomon set as a proxy for the full Rashomon set. 
This involves running the training algorithm several times to generate the models. By using rejection sampling, we can assume that $\randomalgo$ is sampling from a distribution $\traindist$ on $\emprashomonset$. Conditioned on $\dtrain$, repeatedly running the algorithm $\randomalgo$ generates conditionally i.i.d.~samples $\rf_{1}, \rf_{2}, \ldots ,\rf_{\Nretail}$.

We define the Rashomon \smallset{} as this set of samples
     \begin{align}
    \rretailset =  \{ \rf_{\nretailind} :  k \in [\Nretail]\} \subseteq \mc{R}.
    \end{align}
 
As mentioned in Section \ref{sec:related:ensemble}, the basic approach to ensembling is to average the constituent model outputs. Given a realization $\retailset$ of the \smallset{} set, we define the ensemble of the \smallset{} set at $x$ as
    \begin{align}\label{eq: ensemble at x}
    \ensmretset(\retailset,x) = \frac{1}{\Nretail} \sum_{\nretailind=1}^{\Nretail} f_{\nretailind}(x).
    \end{align}
By modeling an ensemble of the \smallset{} set, $\ensmretset(\rretailset,x)$, as randomly sampled from a Rashomon set, we make explicit that ensemble predictions are themselves random variables. Even when all constituent models are near-optimal, different realizations of  $\ensmretset(\rretailset,x)$ may lead to different ensemble predictions at the same test point. The inconsistency in predictions arising from this randomness is the core difficulty addressed in this paper. In Appendix \ref{sec: Related measures}, we describe other measures that we can compute using a single \smallset{} of size $\Nretail$, and in Section \ref{sec: Comp consistency multiplicity}, we show how these measures relate to multiplicity. We also demonstrate that our proposed consistency measure, defined in the next section, performs better at preemptively capturing multiplicity than related measures.

\section{Proposed consistency measure}\label{sec: Proposed consistency measure}

To probe local behavior around a test instance
$x$, we introduce small random perturbations to the input and evaluate the ensemble’s predictions on these perturbed points.
For example, we may perturb the raw input $x \in \cX$ by adding noise when looking at deep ensembles. If we are studying ensembles produced by fine-tuning a transformer-based model, we could perturb the embedding, which serves as input to the fine-tuning layers. We can then use the perturbed values to test if the ensemble is consistent at a given input $x$.

\subsection{Measuring ensemble consistency}

For simplicity, we will prove our results for additive noise drawn from the uniform distribution on a sphere, although our results should hold for any isotropic subgaussian noise distribution~\citep{vershynin2018high}.  For a fixed $x$, define $\ltwohsphere(x,\sigma) = \{\localtest : \normtwo{x - \localtest} = \sigma\}$ to be the sphere of radius $\sigma$ centered at $x$.

Let $\rlocaltestmat = \{ \rlocaltest_j : j \in [\nogaussianpts] \}$ be drawn i.i.d.~from a uniform distribution $\Unif(\ltwohsphere(x,\sigma))$. We will measure how sensitive predictions are to small perturbations of $x$. 

\begin{definition}[Ensembles with perturbed inputs]\label{def: local ensemble}
For a fixed $x$, Rashomon \smallset{} set $\retailset$, and noise realizations $\localtest_j$, for each $j$ we define the ensemble value at $\localtest_j$ by
    \begin{align}
    \ensmretset(\retailset,\localtest_j) 
    =  \frac{1}{\Nretail}\sum_{k=1}^{\Nretail} f_{k}(\localtest_j).
    \end{align}
Averaging over the $\nogaussianpts$ noise realizations, we have
     \begin{align}
    \ensmretset(\retailset,\localtestmat) 
    = \frac{1}{\Nretail} \sum_{k=1}^{\Nretail}
    \frac{1}{\nogaussianpts}\sum_{j=1}^{\nogaussianpts}f_{k}(\localtest_{j}).
    \end{align}
\end{definition}

To understand the stability of an ensemble's predictions at $x$, we can analyze the ensemble from an independently sampled \smallset{} set. The goal is to develop a test based on $(x, \retailset,\localtestmat)$ that can identify instances $x$ that are inconsistent. In some applications, these instances will then be reviewed by a more complex system or a human auditor.

We would like to characterize test instances $x$, whose predictions should be stable with respect to any \smallset{} sets drawn from the Rashomon set. To that end, we propose a novel consistency measure called the $(\beta,\sigma)$-consistency. We first define consistent test points in Definition~\ref{Definition: beta conf} as those that are assigned high scores of consistency under the expected model in the full empirical Rashomon set $\emprashomonset$. We then show that if an individual test point is consistent with respect to the expected model, under some mild assumptions, the ensemble of a \smallset{} set $\retailset$, will also be consistent with high probability (exponential in the number of samples $\nogaussianpts$ from a local neighborhood around $x$ and the size $\Nretail$ of the \smallset{} set).

\begin{definition}[$(\beta,\sigma)$-consistency of $\emprashomonset$]\label{Definition: beta conf}
Given $\emprashomonset$ with distribution $\tau$, a random model $\rf$ $\iid \tau$, an instance $x \in \dtest$, the expected model $\E_{\rf}[\rf(x)]$ of the empirical Rashomon set $\emprashomonset$ at $x$, a random perturbation $\rlocaltest$ $\iid \Unif(\ltwohsphere(x,\sigma))$, and parameters $\beta \in [-\frac{1}{2},\frac{1}{2}]$ and $\sigma > 0$, let 
	\begin{align}
	 \consistency(\emprashomonset,x)
    =\abs*{\E_{\rf}\bracks*{\rf(x)} 
        - \frac{1}{2}} 
    - \E_{\rf,\rlocaltest}\bracks*{\abs*{\rf(x) - \rf(\rlocaltest)}},
    \end{align}
be the margin of the expected model of $\emprashomonset$ at $x$, penalized by the expected variability around $x$, where the expectation is taken jointly with respect to the perturbation $\rlocaltest$, and the model $\rf$. Then the expected model of $\emprashomonset$ is said to be $(\beta,\sigma)$-consistent at $x$ if, 
    \begin{align}\label{eq: beta consistency} 
    \consistency(\emprashomonset,x)  \ge \beta.
    \end{align}
\end{definition}

The consistency condition in Definition \ref{Definition: beta conf} characterizes test instances at which the margin of the empirical Rashomon set's expected model remains large after accounting for local variability. Our theoretical guarantees require two additional conditions, namely stability of model predictions within the Rashomon set and local Lipschitz continuity. We formalize these conditions below.

\begin{definition}[$\rashexpctparam$-stability of  $\emprashomonset$ at $x$]
\label{def: Rashomon stability}    
Given  a Rashomon set $\emprashomonset$, a test point $x \in \dtest$, and a small constant $\rashexpctparam > 0$, $\emprashomonset$ is said to be $\rashexpctparam$ stable at $x$ if, 
\begin{align}
    \abs{f(x)-\E_{\rf}[\rf(x)]} \le \rashexpctparam, \:\: \text{for all}\:\: f \in \emprashomonset.
\end{align}
\end{definition}

\begin{definition}[Local Lipschitz continuity]\label{Definition: eta Lipschitz}
   A function $f : \mc{A} \subset \mc{X} \to \mathbb{R}$ is said to be locally Lipschitz if for each $x \in \mc{A}$ there exists constant $\eta > 0$ and $\sigma > 0$ such that,
    \begin{align}
      \abs{f(x) - f(\localtest)} \leq  \eta\normtwo{x-\localtest}, \:\: \text{for all}\:\:\localtest \in \ltwohsphere(x,\sigma).
    \end{align}
\end{definition}

Global Lipschitz continuity requires that a model's output cannot change too rapidly anywhere in the input space. For modern machine learning models, particularly deep neural networks and large language models, this is both unrealistic and unnecessary. These models are known to exhibit highly non-smooth and complex behavior in regions of the input space far from the data distribution's support, even when they behave predictably near the data. Imposing a global Lipschitz constraint would therefore exclude most practical models or require overly conservative constants that are dominated by worst-case behaviors irrelevant to the test instances of interest. Rather than assuming smoothness everywhere, we require smoothness only within a neighborhood $\ltwohsphere(x,\sigma)$. 

\begin{definition}[$(\beta - \slackmeasure, \sigma)$-consistency of a \smallset{} set]\label{def: Consistency measure} Given a Rashomon \smallset{} set $\retailset = \{f_k : k \in [\Nretail]\} \subseteq \emprashomonset$ of size $\Nretail$, a set of perturbed samples $\localtestmat = \{\localtest_{j} : j \in [\nogaussianpts]\} \subseteq \ltwohsphere(x,\sigma)$, the ensemble prediction $\ensmretset(\retailset,x)$ at $x$ as in Equation~\eqref{eq: ensemble at x}, and a constant $\slackmeasure \ge 0$, let $ \consistency(\retailset,\localtestmat, x)$ be defined as,
\begin{align}
    \consistency(\retailset,\localtestmat, x) = \babs{\ensmretset(\retailset,x)-\frac{1}{2}} - \frac{1}{\nogaussianpts}\sum_{j=1}^{\nogaussianpts}\frac{1}{\Nretail}\sum_{j=1}^{\Nretail}\babs{f_k(x) - f_k(\localtest_j)}.
    \label{eq:consistency}
\end{align}
 Then $\retailset$ is said to be $(\beta - \slackmeasure, \sigma)$-consistent if  
 \begin{align}
    \consistency(\retailset,\localtestmat, x) \ge \beta - \slackmeasure
 \end{align}
\end{definition}

To prove probabilistic guarantees on our proposed consistency measure, we make some common assumptions.

\begin{restatable}[Rashomon consistency assumptions]{assumption}{assume}\label{assume:rashomon}
The set $\rretailset  = \{\rf_{k} : k \in [\Nretail]\}$ is an i.i.d. sample from a distribution $\tau$ on $\emprashomonset$. Every model $f$ in the Rashomon set $\emprashomonset$ is $\eta$-locally Lipschitz as in Definition \ref{Definition: eta Lipschitz}. For fixed parameters $\beta, \sigma$ and $\rashexpctparam$, $x$ belongs to the subset of test instances for which the Rashomon set is $(\beta,\sigma)$-consistent according to Definition \ref{Definition: beta conf}, and $\rashexpctparam$-stable according to Definition \ref{def: Rashomon stability}. We denote this set as $ \dtestconf(\beta,\sigma,\rashexpctparam) = \{ x \in \dtest : \consistency(\emprashomonset, x) \ge \beta \:\:\text{and} \:\:\emprashomonset \:\:\text{is} \:\: \rashexpctparam \:\:\text{stable at} \:\: x\}$  $\subseteq \dtest$.
\end{restatable}

The concentration results that follow are stated only for a fixed $x \in  \dtestconf(\beta,\sigma,\rashexpctparam)$. For notational simplicity, we suppress the dependence of $\dtestconf$ on $\beta, \sigma, \rashexpctparam$. Theorem \ref{Theorem: Main Theorem} provides a probabilistic guarantee that an ensemble of a randomly drawn \smallset{} set $\rretailset$, is $(\beta - \slackmeasure,\sigma)$-consistent at $x$ under Assumption \ref{assume:rashomon}. 

\begin{restatable}[Consistency guarantee]{theorem}{consistent}
\label{Theorem: Main Theorem}

  Let $\emprashomonset$, $\rretailset$, and $x$ satisfy Assumption \ref{assume:rashomon} for parameters  $\beta \in [-\frac{1}{2},\frac{1}{2}]$, and $\sigma, \rashlexpctparam > 0$. For a test instance $x \in \dtestconf$, let $\rlocaltestmat =  \{\rlocaltest_{j}:  j \in [\nogaussianpts]\}$ be a random set of perturbations of $x$, distributed $\iid \Unif(\ltwohsphere(x,\sigma))$, and  let $\ensmretset(\retailset,x)$ be the ensemble prediction at $x$ as in Equation \eqref{eq: ensemble at x}. Given the consistency measure $\consistency(\retailset,\localtestmat,x) = \babs{\ensmretset(\retailset,x)-\frac{1}{2}} - \frac{1}{\nogaussianpts}\sum_{j=1}^{\nogaussianpts}\frac{1}{\Nretail}\sum_{k=1}^{\Nretail}\babs{f_k(x) - f_k(\localtest_j)} $ from Definition \ref{def: Consistency measure}, for all $\slackmeasure \ge 0$,
\begin{align}
    \P_{\rretailset,\rlocaltestmat}\bracks*{\consistency(\rretailset,\rlocaltestmat,x) \ge \beta - \slackmeasure} 
    \ge 1 - 2\exp\parens*{
        \frac{-\Nretail\slackmeasure^2}{8\rashexpctparam^2}} - 2\exp\parens*{\frac{
        -\nogaussianpts\slackmeasure^{2}}{8\eta^2\sigma^{2}}} -  2\exp\parens*{\frac{-\Nretail\slackmeasure^2}{8\eta^2\sigma^{2}}}. 
\end{align}
\end{restatable}
\begin{proof}
    We refer the readers to Section \ref{sec: Proofs} for the proof.
\end{proof}

Theorem \ref{Theorem: Main Theorem} states that if $x$ is $(\beta,\sigma)$-consistent under the expected model in the Rashomon set $\emprashomonset$, then with probability increasing exponentially in both $\Nretail$ and $\nogaussianpts$, over the random draws of the retail set, $x$ will remain $(\beta-\slackmeasure,\sigma)$-consistent under the resulting retail sets. For $x \notin \dtestconf$, the theorem provides no consistency guarantee: the local variability of the ensemble predictions may be too large, the prediction may be too close to the decision boundary, or both. Failure to belong to $\dtestconf$ does not imply $x$ is inconsistent under the resulting retail sets, rather, the theorem is simply inconclusive for such points.

\subsection{Interpretation of the consistency measure}
Our goal is to determine, for a fixed instance $x$, whether the prediction produced by a finite ensemble sampled from a Rashomon set is consistent, or whether it is likely to change if a different ensemble were sampled. In Definition~\ref{def: Consistency measure}, we introduce a consistency measure $\consistency(\retailset,\localtestmat,x)$ that achieves this goal. The first term captures how far the ensemble prediction at $x$ lies from the decision boundary, i.e., the margin gap, rather than raw confidence. This distinction is important, as prior consistency measures often rely either on hard-label agreement~\citep{black2022consistent}, which discards useful soft information, or on a single model's confidence, which can be misleading in the presence of multiplicity. More recent approaches \citep{hamman2024quantifying, hamman} incorporate confidence penalized by variance for a specific class, but can assign low scores to consistent predictors (see Section~\ref{sec: hamman note}). In contrast, our margin-based formulation avoids this issue by evaluating distance from the decision boundary, ensuring that such predictors are correctly identified as consistent.

However, margin alone does not guarantee robustness. In a Rashomon set, models may agree at $x$ while having decision boundaries arbitrarily close to it, leading to instability under small perturbations. The second term in $\consistency(\retailset, \localtestmat, x)$ explicitly penalizes this by measuring the average variability of predictions of constituent models in the ensemble between $x$ and nearby perturbed points $\localtest$. This term acts as a local proxy for predictive multiplicity, capturing whether agreement at $x$ persists in its neighborhood. Crucially, unlike prior work that evaluates consistency with a single competing model, our approach leverages ensembles of models sampled from the Rashomon set to reduce variance.  Empirically, we show that combining this consistency measure with ensembling reduces predictive multiplicity, mitigates variability in audit decisions, and lowers the risk of incorrect predictions going unchecked.

Theorem~\ref{Theorem: Main Theorem} shows that under Assumption~\ref{assume:rashomon}, retail ensembles of size $\Nretail$ increasingly recover test instances in $\dtestconf(\beta,\sigma,\rashexpctparam)$ as $(\beta - \slackmeasure, \sigma)$-consistent with an increase in the number of perturbations $\nogaussianpts$ and the size $\Nretail$ of the retail set. The purpose of the theorem is to establish the concentration phenomenon of retail ensembles' consistency score rather than to prescribe the choice of $\slackmeasure$ or $\beta$. Since the proof relies on the distribution-free McDiarmid and Hoeffding inequalities, the resulting finite sample bound is conservative. Accordingly, we fix $\slackmeasure$ to a small positive value in our experiments and defer a discussion on the empirical selection of $\beta$ to Section~\ref{sec: Thresholdvsauditoutcomes}.

To empirically validate this concentration phenomenon, we repeatedly sample retail ensembles of size $\Nretail$ from the Rashomon set following the experimental protocol described in Section~\ref{sec: Scenario specification}. For each retail ensemble, we compute
$\frac{1}{|\dtestconf(\beta,\sigma,\rashexpctparam)|}
\sum_{x\in\dtestconf(\beta,\sigma,\rashexpctparam)}
\Indic{\consistency(\retailset,\localtestmat,x)\ge \beta-\slackmeasure},$
the fraction of test instances in $\dtestconf(\beta,\sigma,\rashexpctparam)$ identified as $(\beta - \slackmeasure,\sigma)$-consistent by the retail ensemble. Figure~\ref{fig:consistency_concentration} shows the empirical distribution of this quantity over $500$ randomly sampled retail ensembles for different values of $\Nretail$. Even for small-sized retail ensembles, the fraction is close to one, indicating that most retail ensembles recover nearly all $(\beta,\sigma)$-consistent test instances as $(\beta - \slackmeasure,\sigma)$-consistent. As $\Nretail$ increases, the distributions become more concentrated near one, demonstrating increasingly consistent recovery across random ensemble draws, in agreement with Theorem~\ref{Theorem: Main Theorem}.

\begin{figure*}[hbt]
\centering
\begin{subfigure}[t][][t]{0.48\textwidth}
    \centering
    \includegraphics[width=\textwidth]{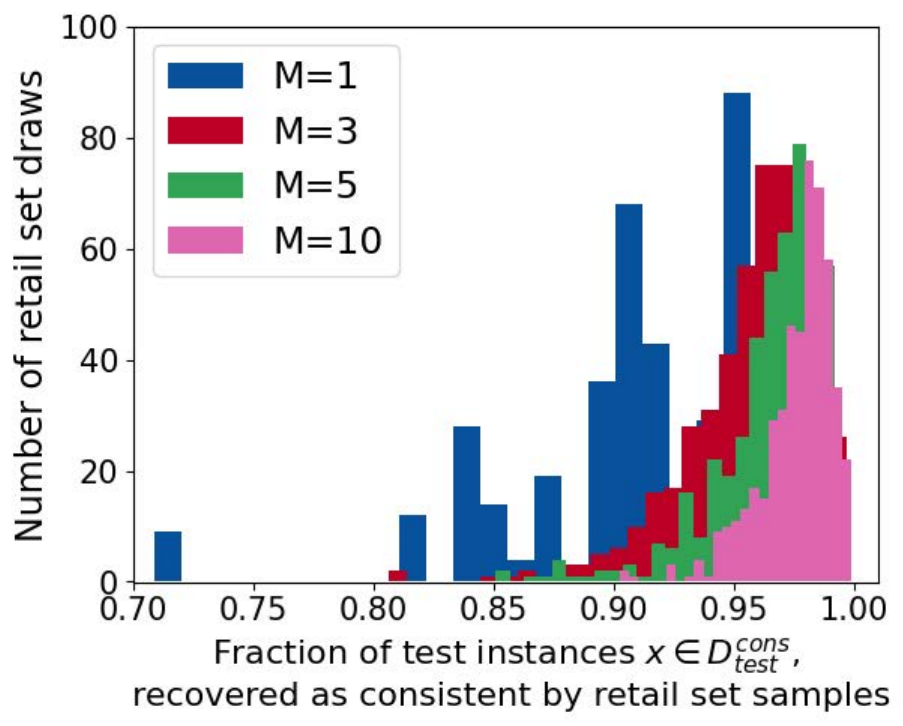}
    \caption{Binary classification on the \textbf{Bank} dataset using  ensembles of fine-tuned \textbf{BIG-SCIENCE T0} models.}
    \label{fig:consistency_concentration_bank}
\end{subfigure}
\hfill
\begin{subfigure}[t][][t]{0.48\textwidth}
    \centering
    \includegraphics[width=\textwidth]{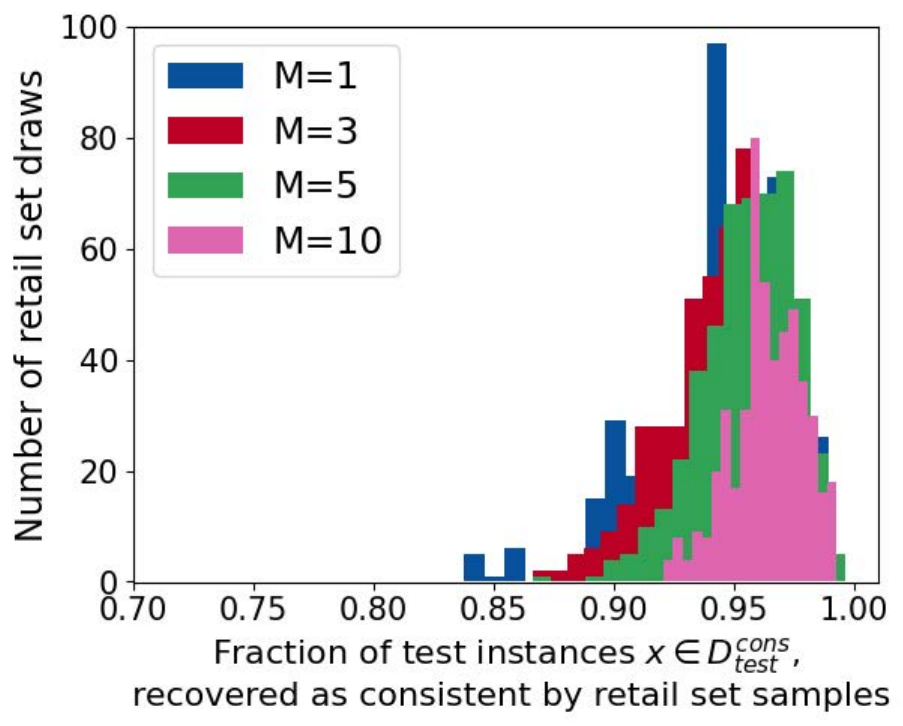}
    \caption{Paraphrase detection on the \textbf{MRPC} dataset using ensembles of fine-tuned \textbf{BERT} models.}
    \label{fig:consistency_concentration_mrpc}
\end{subfigure}
\caption{Empirical illustration of the concentration guarantee in Theorem~\ref{Theorem: Main Theorem}. For each of $500$ independently sampled retail ensembles of size $\Nretail$, we compute the fraction of test instances in $\dtestconf(\beta,\sigma,\rashexpctparam)$ that remain $(\beta - \slackmeasure,\sigma)$-consistent under the retail set,
$\frac{1}{|\dtestconf(\beta,\sigma,\rashexpctparam)|}
\sum_{x\in\dtestconf(\beta,\sigma,\rashexpctparam)}
\Indic{\consistency(\retailset,\localtestmat,x)\ge \beta - \slackmeasure}.$
As the \smallset{} ensemble size increases, the distributions become increasingly concentrated near $1$, indicating that most randomly sampled \smallset{} ensembles recover nearly all $(\beta,\sigma)$-consistent test instances as $(\beta - \slackmeasure,\sigma)$-consistent.}
\label{fig:consistency_concentration}
\end{figure*}

\section{Experimental Setup}\label{sec: Experiments}

The theoretical results in Section \ref{sec: Proposed consistency measure} establish that the proposed consistency measure identifies test instances whose predictions remain consistent across finite ensembles sampled from the Rashomon set. In this section, we empirically evaluate whether these guarantees translate into practical benefits for auditing and predictive multiplicity analysis. Our experiments are designed to answer three questions that directly follow from our theory.

\begin{itemize}
    \item \textbf{Can the proposed consistency measure be used as an effective auditing criterion}?
   Consistency in predictions alone does not imply correctness in predictions. A prediction may be stable across models and still be incorrect. Our objective is not to show that consistency perfectly identifies incorrect predictions, but rather to investigate the measure's usefulness in auditing.
   We therefore evaluate whether thresholding the proposed consistency score produces a favorable trade-off between audit risk (missed erroneous predictions) and audit cost (unnecessary diversions of correct predictions). We do not attempt to optimize the auditing policy itself or compare against every possible rejection strategy, since our goal is to evaluate the usefulness of the consistency measure rather than design a novel auditing framework.
   We find that increasing the consistency threshold can reduce audit risk, in some setups close to 0, while requiring a only moderate increase in diversions.

    \item \textbf{Do finite \smallset{} ensembles approximate the audit decision of the Rashomon set as ensemble size increases}?
    Theorem \ref{Theorem: Main Theorem} shows that consistency estimates obtained from finite ensembles should concentrate around those of the expected model of the Rashomon set. We therefore study how diversion decisions and audit outcomes change as the ensemble size $\Nretail$ increases. We observe that as the number of models in the retail ensemble increases, auditing outcomes become increasingly concentrated around the ensemble model of the full Rashomon set. This empirical behavior is consistent with the theoretical guarantees established in Section \ref{sec: Proposed consistency measure}.

    Since enumerating the Rashomon set is infeasible, practitioners can only deploy ensembles with a finite number of models due to constraints on computation, latency, or memory. Demonstrating that relatively small retail ensembles approximate the Rashomon set behavior is therefore essential for showing that the proposed method is practically deployable. We intentionally restrict our study to the standard ensemble formed by uniformly averaging model predictions, since this estimator is the one analyzed throughout our theoretical development. Alternative ensemble strategies, including learned weighted ensembles and mixture-of-experts architectures, assign non-uniform or input-dependent weights to constituent models, thereby changing the estimator and requiring separate theoretical analysis. Understanding how the proposed consistency measure interacts with these more sophisticated ensemble constructions is an interesting direction for future work.
    \item \textbf{Does the proposed consistency measure capture predictive multiplicity?} Accurate estimation of existing multiplicity metrics outlined in Section \ref{sec: Multiplicity metrics} requires access to the Rashomon set and is therefore difficult to estimate in practice. We investigate whether the proposed consistency measure, which jointly accounts for ensemble margin and local prediction stability, provides a more informative characterization of predictive multiplicity than prior measures that consider individual models or rely solely on confidence scores. 
    Our experiments show that our proposed consistency measure exhibits strong correlations with established multiplicity metrics. By combining ensemble margin and local perturbation stability, our measure also shows stronger correlations with multiplicity metrics than other related measures while remaining practical to estimate from finite retail ensembles.
\end{itemize}

The remainder of this section is organized as follows. In Section \ref{sec: Scenario specification}, we describe the experimental protocol and how we construct the Rashomon set and sample retail ensembles from it. In Section \ref{sec: Auditing with proposed measure}, we answer Questions 1 and 2, and in Section \ref{sec: Comp consistency multiplicity}, we answer Question 3.

\subsection{Scenario Specification}\label{sec: Scenario specification}

To avoid potential application-specific bias, we implement the proposed consistency measure in two disparate application domains, natural language processing (NLP) and classification of tabular data, and evaluate the behavior of the consistency measure with respect to ensembles drawn from the respective Rashomon sets. In the NLP domain, we consider two binary classification tasks from the GLUE benchmark. The first is sentiment classification, using the SST-2 dataset,~\citep{socher-etal-2013-recursive}. SST-2 consists of movie reviews labeled as positive or negative. We use approximately 67,000 instances for training and 872 instances for validation. The second task is paraphrase detection, using the MRPC dataset~\citep{dolan-brockett-2005-automatically}. MRPC contains sentence pairs labeled as semantically equivalent, with approximately 3,700 training instances and 408 validation instances. For both tasks, we fine-tune encoder-only transformer architectures, namely BERT~\citep{devlin2019bertpretrainingdeepbidirectional} and RoBERTa~\citep{zhuang-etal-2021-robustly}, using standard supervised fine-tuning.

In the second evaluation setting, we consider binary classification on tabular data using the Adult~\citep{adult_2} and Bank Marketing~\citep{bank_marketing_222} datasets. The Adult dataset predicts whether an individual's annual income exceeds $\$50,000$ based on demographic and employment-related attributes, while the Bank Marketing dataset predicts whether a client subscribes to a term deposit using demographic, financial, and campaign-related features. We adopt the T-Few prompt-based fine-tuning framework~\citep{liu2022fewshot}, which reformulates each row of a tabular dataset as natural language before it is processed by a large language model. Following this framework, each tabular instance is serialized by representing each feature along with its corresponding value, and then combined with a task-specific prompt. This serialization enables structured tabular prediction to be formulated as a natural language understanding problem. We use the BIG-SCIENCE T0 encoder-decoder language model~\citep{sanh2022multitask} together with IA$^3$ adapters~\citep{liu2022fewshot} for parameter-efficient fine-tuning. We adopt this experimental protocol because it uses the same T-Few framework as \citet{hamman2024quantifying}, enabling comparison with their analysis of predictive multiplicity under fine-tuning. 

The source of randomness depends on the fine-tuning strategy. For encoder-based models (BERT and RoBERTa) on NLP tasks, all runs start from the same pre-trained backbone, while the random seed determines the initialization of the task-specific classification head and other stochastic components of training, including data shuffling and dropout. For the BIG-SCIENCE T0 models on the tabular tasks, the pre-trained backbone remains frozen, and only the IA$^3$ adapter parameters are randomly initialized and optimized during fine-tuning.
 
\subsubsection{Empirical Rashomon set generation}\label{sec: Empirical Rashomon Setup}

As previously discussed, it is infeasible to enumerate the entire empirical Rashomon set for any of the scenarios we wish to evaluate. Thus, for each dataset-model pair, we approximate the empirical Rashomon set by fine-tuning a large pool of models from the same pre-trained initialization under different random seeds. For each dataset-model pair, we select the model with the lowest empirical error in the pool and choose $\rashomonparam = 0.03$ to define the Rashomon set relative to its empirical error following Definition~\ref{def: empirical rashomon set}. We then apply rejection sampling to retain $\Nlib = 100$ models that satisfy the Rashomon set criterion, yielding our empirical approximation of the Rashomon set
\footnote{The authors acknowledge the Office of Advanced Research Computing (OARC) at Rutgers, The State University of New Jersey, for providing access to the Amarel cluster and associated research computing resources that have contributed to the results reported here~\citep{OARC_Amarel}. The choice of $\Nlib$ as a result is influenced by the available GPU resources and per-job walltime limits on the cluster}.
The resulting set of $\Nlib$ retained models is referred to as the Rashomon \largeset{} and is denoted by $\retailset_{\Nlib}$.

\subsubsection{Selecting the radius \texorpdfstring{$\mbf{\sigma}$}{\sigma} of local perturbations}
For both the encoder-based models and the T-Few framework, we compute the proposed consistency measure in the model's representation space. Each input is first mapped to an embedding representation before being processed by the fine-tuning layers. We abuse notation and use $x \in \mc{X}$ to denote both the original input and its embedding representation. Thus, all perturbations described in this section are applied to $x$ in the embedding space. Following \citet{hamman2024quantifying}, we set the perturbation radius as
\begin{align}
    \sigma = \alpha \cdot \median_x(\normtwo{x - \NN(x)}),
\end{align}
where $\NN(x)$ denotes the nearest neighbor of $x$ in the training set and $\alpha$ is a scaling parameter. The median nearest neighbor distance provides a data-driven estimate of the size of the local neighborhood around a sample. Choosing $\sigma$ as a fraction of this distance ensures that perturbations remain within the local neighborhood of each test instance, allowing the proposed consistency measure to evaluate local prediction stability rather than probing behavior under unrealistically large perturbations. A radius of $\sigma = 0.01$ computed this way worked well across all settings. Similar neighborhood-based scaling is commonly used in machine learning algorithms, where hyperparameters are chosen relative to the intrinsic geometry of the data, such as in the literature of certified robustness techniques~\citep{cohen2019certified,salman_hadi,alfarra2022data}.

\subsubsection{Computing retail set consistency}
From the Rashomon \largeset{} $\retailset_{\Nlib}$, we sample $\Nretail$ models, to form a \smallset{} set $\retailset_{\Nretail}$, of size $\Nretail$, to reflect realistic deployment scenarios where only a limited ensemble size is feasible. 
We explore different values of $\Nretail$, i.e., $\Nretail = [1, 3, 5, 10, 30]$, to analyze the benefits of ensembling more models over individual models. We estimate consistency using Monte Carlo samples drawn uniformly from the surface of the $l_{2}$ sphere of radius $\sigma$ centered at $x$. For each consistency computation, all models in a sampled retail ensemble must be evaluated on the same perturbation set. 

Ideally, a new perturbation set would be generated for every Monte Carlo sample of the retail set, as well as for different sizes $\Nretail$ of the retail set. However, this would require rerunning model inference for every replication and every $\Nretail$, which is computationally prohibitive, since we observed GPU memory limits restrict inference to at most 100 perturbations per forward pass. Instead, we use  $\Nlib$ independently trained models in the warehouse to predict on a library of $\nogaussianptslib = 200$ randomly generated perturbations for each $x$. Each model is evaluated once for each perturbation in the perturbation library, and the resulting predictions are cached. For each Monte Carlo replication $q=1,\dots.,\nsampretail$, we sample without replacement the ensemble predictions corresponding to a retail set $\retailset_{\Nretail,q}$ of size $\Nretail$, and perturbation set $\localtestmat_{\nogaussianpts,q}$ of size $\nogaussianpts$ to compute the consistency score $\consistency(\retailset_{\Nretail,q},\localtestmat_{\nogaussianpts,q},x)$ for the $q-th$ draw. Repeating this procedure for $\nsampretail = 500$ replications yields an empirical distribution of the consistency estimator that captures the variability due to both retail ensemble sampling and perturbation sampling. For comparison, the Rashomon warehouse consistency score is computed once using the entire warehouse ensemble $\retailset_{\Nlib}$ together with the perturbation set $\localtestmat
_{\nogaussianptslib}$
, yielding $\consistency(\retailset_{\Nlib},x)$. Code to reproduce the experiments in this paper will be made available shortly, following final cleanup\footnote{\url{https://github.com/Sinjini77/Measuring-ensemble-consistency}}.

\subsection{Operationalizing the consistency measure through auditing}\label{sec: Auditing with proposed measure}

We return to the auditing framework introduced in Section~\ref{sec: Auditing and diversion}, where we compared the decisions made using an idealized baseline reference model and a practical candidate model. Table~\ref{tab: combined audit outcomes} enumerated all possible combinations of correct/incorrect predictions and diversion decisions of these models. The remaining question is how to decide when to divert an input. 

Our proposed approach decides when to divert by thresholding the consistency measure. The complete decision procedure is summarized in Algorithm~\ref{alg:consistent test point search}. Given a set of models $\retailset$, let $\ensmretset(\retailset,x)$ denote the ensemble prediction for input $x$, which produces an estimated label $h(x; \retailset)=\Indic{\bar{f}(\retailset,x)\ge \frac{1}{2}}$ and an error $\modelerrorlabel(x; \retailset)=\Indic{h(x; \retailset)\neq y}$. Specifically, we define the estimated model risk as $\riskassess(x; \retailset)=\Indic{\consistency(\retailset,\localtestmat,x) \ge \beta - \slackmeasure}$, where $\beta - \slackmeasure$ is the audit threshold for retail set ensembles. Inputs satisfying $\consistency(\retailset,\localtestmat,x)\ge \beta - \slackmeasure$ are diverted for human review, while the ensemble prediction is accepted for the remaining inputs.

In the subsequent sections, we show how varying $\beta$ induces a trade-off between audit cost and audit risk, and how this trade-off can guide the choice of $\beta$. We characterize this trade-off using the Rashomon warehouse ensemble, which serves as our baseline model and represents the reference auditing strategy against which finite sized ensembles are compared. Decreasing $\beta$ reduces the number of audited instances, thereby decreasing unnecessary human reviews of correctly classified samples, but may also increase the number of incorrect predictions that bypass auditing. Conversely, increasing $\beta$ results in more conservative auditing, reducing missed errors at the expense of auditing more correct predictions.

Finally, we evaluate whether the auditing decisions made using retail ensembles of fixed size $\Nretail$ approach those of the Rashomon warehouse ensemble as the ensemble size increases. Consistent with the objective outlined in Section~1.1, a desirable auditing strategy is one that minimizes excess audit cost and risk by reproducing the warehouse baseline's audit outcomes while requiring substantially fewer models. 

\begin{algorithm}[thb]
\caption{ Retail set ensembling consistency test}
\label{alg:consistent test point search}
\begin{algorithmic}
\renewcommand{\algorithmicrequire}{\textbf{Input:}}
\renewcommand{\algorithmicensure}{\textbf{Output:}}
\REQUIRE Test point $x$, Rashomon \smallset{} set $\retailset=\{f_{k} : k \in [\Nretail]\}$  of $\Nretail$  models, threshold $\beta$, slack parameter $\slackmeasure$
\ENSURE  Decision to divert $x$ to an expert
\STATE Sample $\nogaussianpts$ points $V = \{ \localtest_j :  j \in [\nogaussianpts]\} \iid \Unif(\ltwohsphere(x,\sigma))$.
\STATE Compute $ \consistency(\retailset, \localtestmat, x)$ as in Definition \ref{def: Consistency measure}.
\STATE Test if $ \consistency(\retailset, \localtestmat, x) \ge \beta - \slackmeasure$ 
\IF {Test is false}

\STATE Ensemble diverts 
\ELSE

\STATE Ensemble predicts
\ENDIF
 
\end{algorithmic}
\end{algorithm}
\subsubsection{Audit metrics}
We quantify the trade-off between risk and cost highlighted in the previous section using the auditing metrics introduced below, which measure both unnecessary audits for correct predictions and erroneous predictions that escape human review.

\begin{definition}[Costs and risks for auditing] \label{def:costsandrisks}
Let $\retailset$ be a retail set and $\ensmretset(\retailset,x)$ be the ensemble prediction of $\retailset$ for a fixed instance $x \in \dtest$, with $\Ntest$ as the size of $\dtest$. Let $h(x; \retailset)$ be the model's estimated label, $\modelerrorlabel(x; \retailset)= \Indic{h(x; \retailset) \ne y}$ the error label, and $\riskassess(x; \retailset) = \Indic{\consistency(\retailset,\localtestmat, x) \ge \beta - \slackmeasure}$ be the risk assessment.  The costs and risks of the auditing system in Algorithm \ref{alg:consistent test point search} can be measured in several ways.
    \begin{itemize}
    \item \textbf{Relative cost and risk.} The relative cost and risk compare the auditing system behavior in Algorithm \ref{alg:consistent test point search} in relation to the ensemble model predictions. Define the following two ratios:
    \begin{align}
    \label{def: Audit cost ratio}
    \auditcostratio(\retailset) 
        &= \frac{\sum_{x \in \dtest}\Indic{\modelerrorlabel(x; \retailset) = 0, \riskassess(x; \retailset) = 1}
            }{
            \sum_{x \in \dtest}\Indic{\modelerrorlabel(x; \retailset) = 0}}
    \\
    \label{def: Audit risk ratio}
    \auditriskratio(\retailset) 
        &= \frac{\sum_{x \in \dtest}\Indic{\modelerrorlabel(x; \retailset) = 1, \riskassess(x;\retailset) = 0}
            }{
            \sum_{x \in \dtest}\Indic{\modelerrorlabel(x; \retailset) = 1}}
    \end{align}
The audit cost ratio in \eqref{def: Audit cost ratio} is the fraction of points that the ensemble classified correctly but the auditing system sent for human intervention. The audit risk ratio in \eqref{def: Audit risk ratio} is the fraction of points that the ensemble classified incorrectly, but the auditing system did not send for human intervention. 
    \item \textbf{Absolute cost and risk.} We also measure the overall auditing burden and audit risk as fractions of the entire test set. Define the following two fractions:
    \begin{align}
    \label{def: Diversion fraction}
    \auditdivfrac(\retailset) 
        &= \frac{\sum_{x \in \dtest}\Indic{\riskassess(x; \retailset) = 1}}{\Ntest}
    \\
    \label{def: Audit risk fraction}
    \auditriskfrac(\retailset) 
        &= \frac{\sum_{x \in \dtest}\Indic{\modelerrorlabel(x; \retailset) = 1, \riskassess(x; \retailset) = 0}}{\Ntest}
    \end{align}
The diversion fraction \eqref{def: Diversion fraction} is the fraction of the test set sent for human intervention. The audit risk fraction \eqref{def: Audit risk fraction} is the fraction of the test set on which the auditing system fails to divert an erroneous prediction.
\end{itemize} 
\end{definition}
We use $\auditriskratio(\retailset_{\Nretail})$, $\auditcostratio(\retailset_{\Nretail})$, $\auditdivfrac(\retailset_{\Nretail})$, and $\auditriskfrac(\retailset_{\Nretail})$ to analyze audit behavior of a retail set $\retailset_{\Nretail}$, of size $\Nretail$. Similarly, for a Rashomon \largeset{} $\retailset_{\Nlib}$, with ensemble prediction $\ensmretset(\retailset_{\Nlib},x)$, estimated label $h(x; \retailset_{\Nlib})$, error label $\modelerrorlabel(x; \retailset_{\Nlib})= \Indic{h(x; \retailset_{\Nlib}) \ne y}$, and risk assessment $\riskassess(x; \retailset_{\Nlib}) =  \Indic{\consistency(\retailset_{\Nlib},x) \ge \beta}$, we use $\auditriskratio(\retailset_{\Nlib})$, $\auditcostratio(\retailset_{\Nlib})$, $\auditdivfrac(\retailset_{\Nlib})$, and $\auditriskfrac(\retailset_{\Nlib})$ to analyze audit behavior of the Rashomon \largeset{}.

\subsubsection{High consistency thresholds substantially reduce audit risk for Rashomon \largeset{}}\label{sec: Thresholdvsauditoutcomes}

In this section, we study how the auditing behavior of the Rashomon warehouse ensemble varies with the audit threshold $\beta$ in Equation~\eqref{eq: beta consistency}. The threshold determines the audit decision $\riskassess(x; \retailset_{\Nlib})=\Indic{\consistency(\retailset_{\Nlib}, x)\ge\beta}$ and therefore controls the trade-off between audit risk and audit cost in Definition~\ref{def:costsandrisks}. Throughout this section, we use paraphrase detection on the MRPC dataset with an ensemble of fine-tuned BERT models as a representative example. Results for additional datasets and architectures are provided in Appendix~\ref{sec: Additional beta tradeoff}.

Figure~\ref{fig: basecostratio} reports the \largeset{} audit cost ratio, $\auditcostratio(\retailset_{\Nlib})$, and audit risk ratio, $\auditriskratio(\retailset_{\Nlib})$ for thresholds $\beta=[-0.40,-0.30,-0.20,-0.10,0.0,0.10,0.20,0.25,0.30,0.35,0.40,0.45,0.50]$, using a coarser grid for smaller values of $\beta$ and finer grid for $\beta \ge 0.2$. Since our goal is to identify highly consistent test instances, the region of larger $\beta$ values is of greater practical interest, and the finer resolution allows us to more accurately characterize the corresponding trade-off between audit cost and audit risk. As $\beta$ increases, the auditing system becomes more conservative, diverting more test instances for human review. As a result, the audit risk ratio decreases while the audit cost ratio increases. Importantly, as we increase $\beta$, the audit risk ratio falls much more rapidly than the audit cost ratio rises, demonstrating the benefits of setting $\beta$ to higher values. This behavior is particularly desirable in high-risk applications, where reducing missed errors is more valuable than avoiding a moderate increase in human review.

 Figure~\ref{fig: basedivriskfrac} reports the \largeset{} diversion fraction and audit risk fraction as functions of $\beta$. Increasing $\beta$ substantially reduces the fraction of test instances that incur audit risk before causing a sharp increase in the diversion fraction. For example, increasing $\beta$ from $-0.5$ to $0.25$ reduces the audit risk fraction (represented by the solid red line) from approximately $0.174$ to almost $0.0$.
 
 Based on this, we set $\beta$ using a held-out test set to evaluate the trade-off between diversion fraction and audit risk fraction. We explicitly specify the maximum diversion fraction we are willing to incur and set an upper bound equal to the ensemble's prediction error plus an additional tolerance of $0.15$. Starting from $\beta=-0.5$, we increase $\beta$ in increments mentioned above and compute the resulting diversion fraction. We then select the largest value of $\beta$ for which the diversion fraction does not exceed the prescribed upper bound. This procedure allows us to tolerate only a modest increase in diversions beyond the baseline prediction error while maximizing the reduction in audit risk. Across all datasets and models considered in this work, this criterion consistently identifies operating points with audit risk fraction close to $0.05$ without incurring a substantial increase in diversion fraction.

\begin{figure*}[hbt]
\centering
\begin{subfigure}[t][][t]{0.45\textwidth}
    \includegraphics[width=\textwidth]{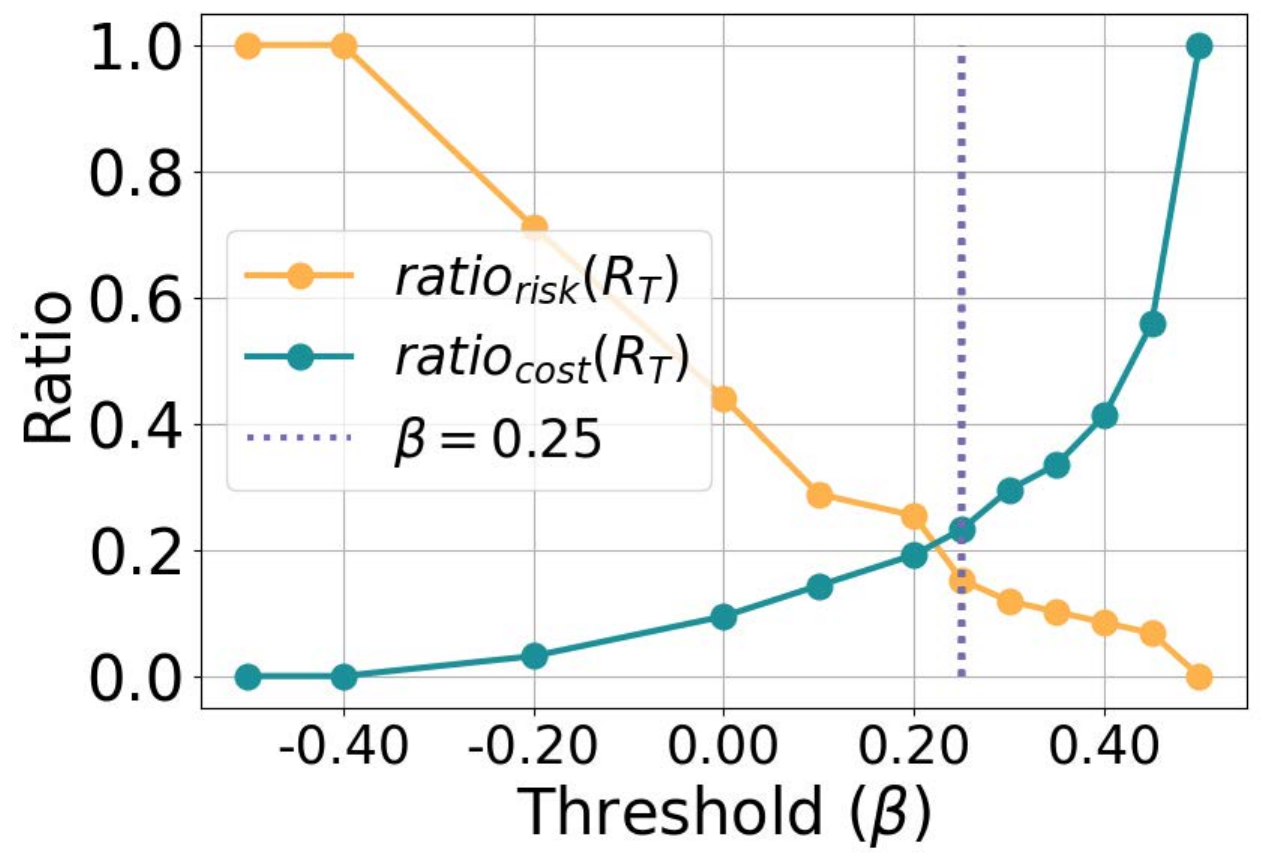}
    \caption{Comparing audit cost ratio ($\auditcostratio(\retailset_{\Nlib})$) and risk ratio ($\auditriskratio(\retailset_{\Nlib}))$. Increasing the threshold $\beta$ decreases risk at the expense of a higher cost.}
    \label{fig: basecostratio} 
\end{subfigure}
\hfill
\begin{subfigure}[t][][t]{0.45\textwidth}
    \includegraphics[width=\textwidth]{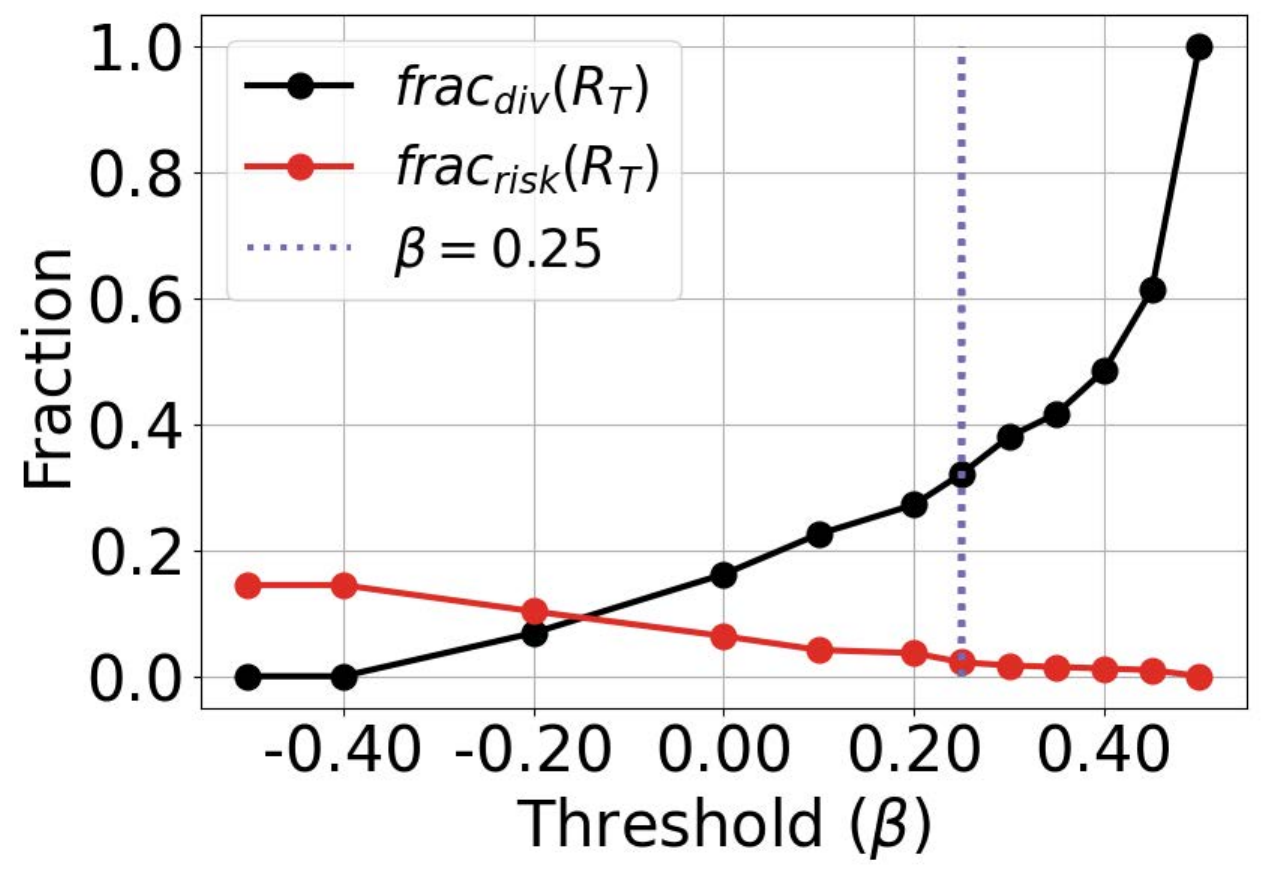}
\caption{Comparing diversion fraction ($\auditdivfrac(\retailset_{\Nlib})$) and audit risk fraction ($\auditriskfrac(\retailset_{\Nlib})$). Increasing the threshold $\beta$ causes the fraction of diverted points to increase and brings down the overall risk close to 0.}
    \label{fig: basedivriskfrac}
\end{subfigure}
\caption{Comparing the performance metrics in Definition \ref{def:costsandrisks} for the \largeset{} ensemble model as a function of the threshold $\beta$ used in the consistency check of Algorithm \ref{alg:consistent test point search}. Results are shown for paraphrase detection on the \textbf{MRPC} dataset using an ensemble of fine-tunings of the \textbf{BERT} model.}
\end{figure*}
        
\subsubsection{Retail set audit outcomes match \largeset{} audit outcomes as we increase $\Nretail$}\label{sec: Retailauditoutcomes}

Section~\ref{sec: Thresholdvsauditoutcomes} showed that thresholding consistency scores provides an effective auditing criterion for the Rashomon \largeset. In practice, however, consistency scores must be estimated from a finite \smallset{} ensemble. In this section, we investigate how closely auditing decisions based on \smallset{} ensembles approximate those of the \largeset{} as the ensemble size $\Nretail$ increases. We fix the audit threshold based on our discussion in Section~\ref{sec: Thresholdvsauditoutcomes}, and evaluate retail ensembles of size $\Nretail=[1,3,5,10,30]$. As described in Section~\ref{sec: Empirical Rashomon Setup}, we generate $\nsampretail=500$ independent retail ensembles $\retailset_{\Nretail,q}$ and compute the audit metrics $\auditcostratio(\retailset_{\Nretail,q})$, $\auditriskratio(\retailset_{\Nretail,q})$, $\auditdivfrac(\retailset_{\Nretail,q})$, and $\auditriskfrac(\retailset_{\Nretail,q})$  from Definition \ref{def:costsandrisks}. Figures~\ref{fig: praccostratio} and~\ref{fig: pracdivriskfrac} report the mean $\pm$ standard deviation over the $\nsampretail$ draws for paraphrase detection on the MRPC dataset using ensembles of independently fine-tuned BERT models. Results with additional datasets and architectures are provided in Appendix~\ref{sec:Addition retailauditoutcomes}.

The concentration result in Theorem~\ref{Theorem: Main Theorem} implies that retail consistency scores converge to the warehouse consistency scores as $\Nretail$ and $\nogaussianpts$ increase. Consequently, retail auditing decisions based on the consistency scores should increasingly align with those of the warehouse ensemble. Rather than using the potentially conservative theoretical value of $\slackmeasure$, we adopt a practical tolerance of $\slackmeasure=0.01$, yielding a retail auditing threshold of $\beta-\slackmeasure$.

Figure~\ref{fig: praccostratio} shows that small retail ensembles exhibit substantially higher audit risk ratios than the warehouse ensemble. For no ensembling, and ensemble size $\Nretail = 3$, the mean audit risk ratio (represented by the solid yellow line) is approximately $0.4$ and $0.25$, respectively. In contrast, the corresponding \largeset{} value (represented by the dashed yellow line) is much lower, around $0.15$. As $\Nretail$ increases, the audit risk ratio steadily decreases toward the warehouse value, while the audit cost ratio increases toward the warehouse audit cost ratio, indicating that larger retail ensembles more faithfully reproduce the warehouse auditing decisions.

A similar trend is observed in Figure~\ref{fig: pracdivriskfrac}. The audit risk fraction corresponding to no ensembling $(\Nretail = 1)$ represented by the solid red line,  is close to 0.07. In contrast, for the \largeset{} ensemble represented by the dashed red line, the value is close to $0$. The audit risk fraction decreases rapidly as $\Nretail$ increases, while the diversion fraction decreases to align with the warehouse diversion fraction. Most of the improvement occurs between $\Nretail=3$ and $\Nretail=10$, after which the curves largely stabilize.

Overall, these results provide empirical support for Theorem~\ref{Theorem: Main Theorem}. As the retail ensemble size increases, retail consistency scores and the resulting audit decisions converge toward those of the warehouse ensemble. Importantly, this behavior is achieved with relatively modest ensemble sizes, suggesting that practical retail ensembles can closely approximate the auditing performance of the Rashomon warehouse without requiring an excessive number of models. A more detailed analysis of the individual audit outcomes, including how retail and warehouse audit decisions differ across the outcome categories introduced in Table~\ref{tab: combined audit outcomes}, is provided in Appendix~\ref{sec: Combinedoutcomes}.

\begin{figure*}[hbt]
\centering
\begin{subfigure}[t][][t]{0.45\textwidth}
    \includegraphics[width=\textwidth]{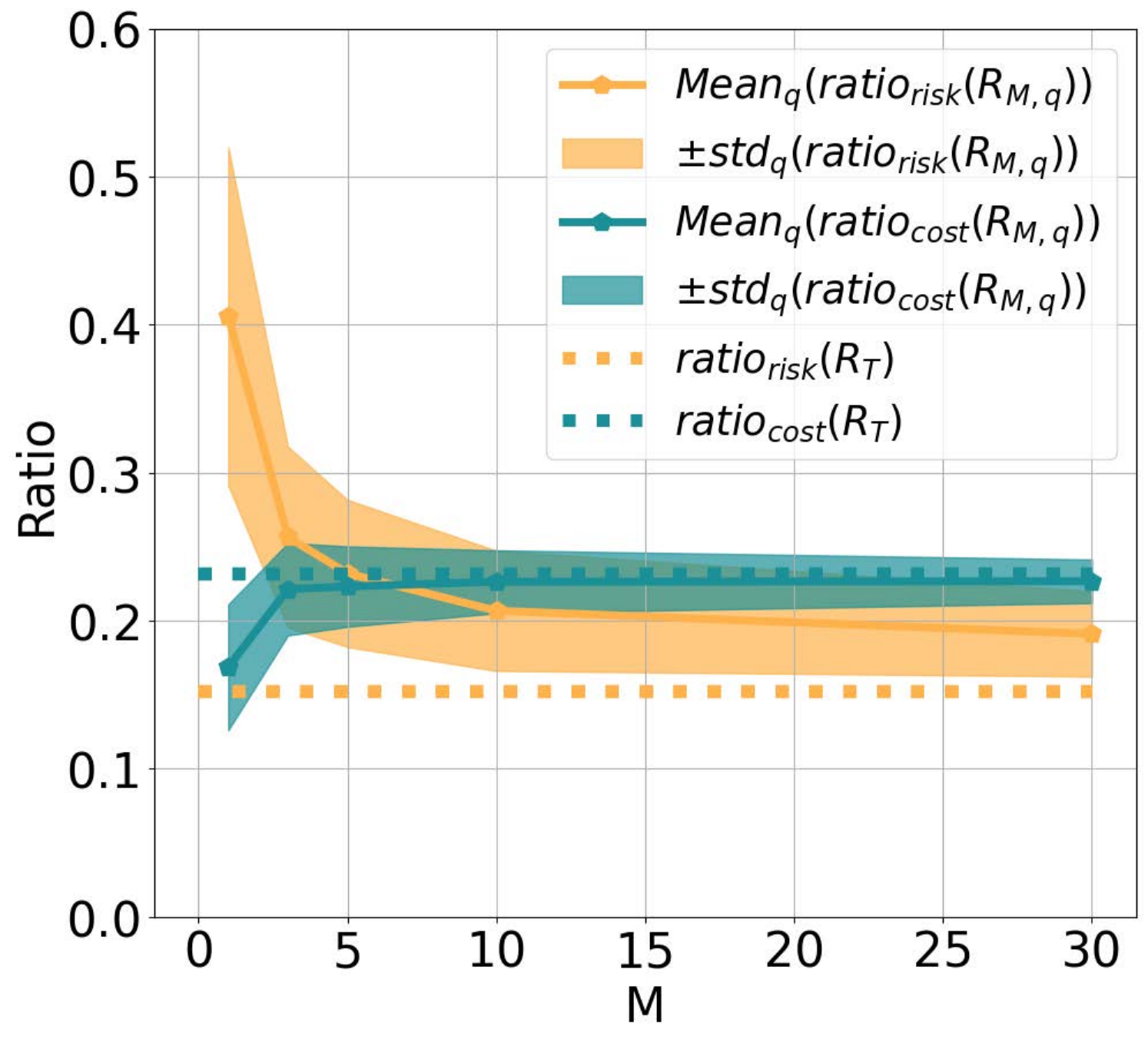}
    \caption{Comparing the audit cost ratio ($\auditcostratio$) and risk ratio ($\auditriskratio$) of retail sets for different ensemble sizes. Increasing the ensemble size has a small impact in matching \largeset{} once the ensemble is sufficiently large (around \Nretail = 10).}
    \label{fig: praccostratio} 
\end{subfigure}
\hfill
\begin{subfigure}[t][][t]{0.45\textwidth}
    \includegraphics[width=\textwidth]{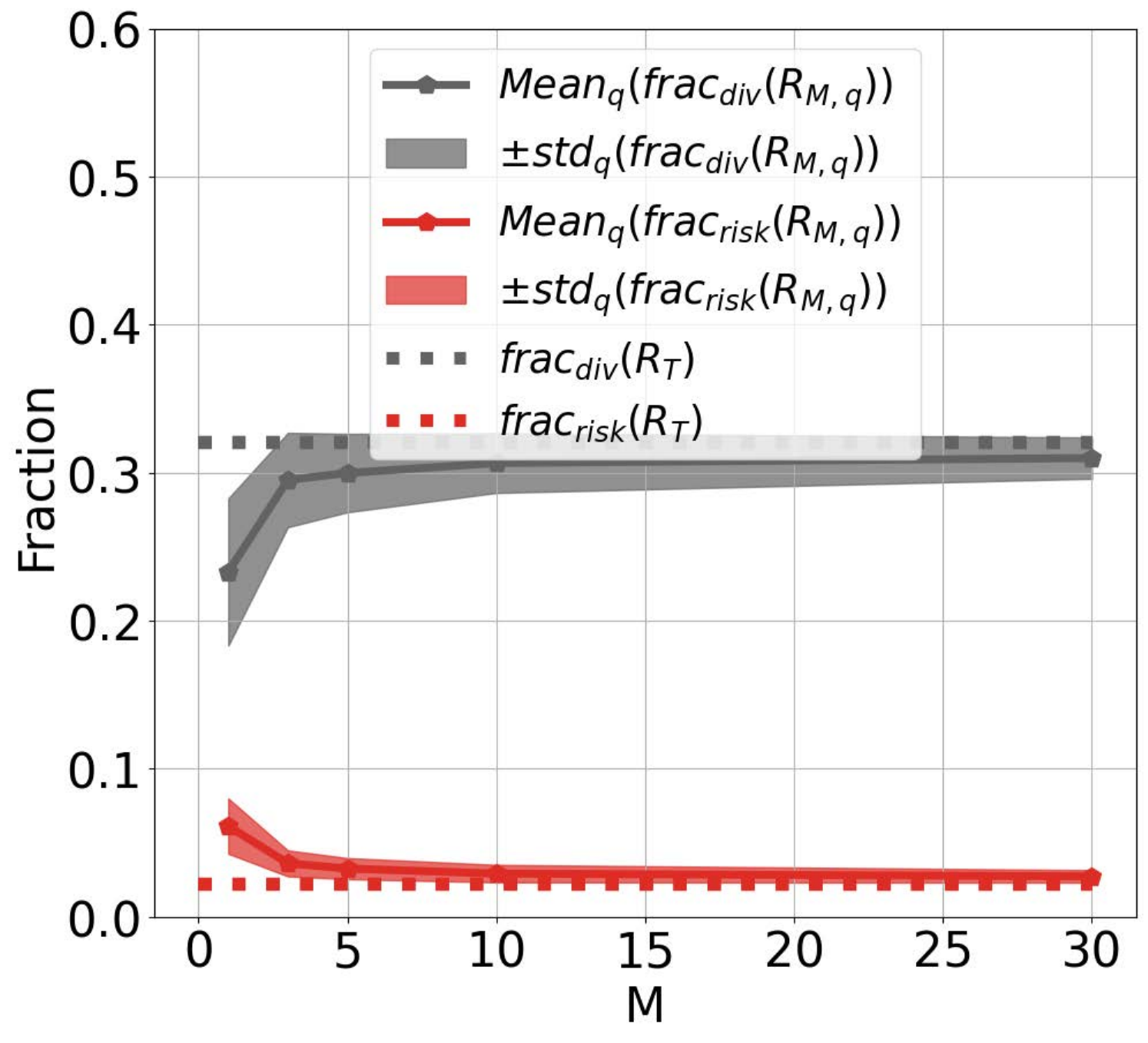}
\caption{Comparing the diversion fraction ($\auditdivfrac$) and audit risk fraction ($\auditriskfrac$) of retail sets for different ensemble sizes. Diversion and risk fractions nearly match those of the warehouse, even with relatively modest ensemble sizes.}
    \label{fig: pracdivriskfrac}
\end{subfigure}
\caption{Comparing the audit metrics in Definition \ref{def:costsandrisks} for  \smallset{} ensemble models as a function of the number of models $\Nretail$. For each size $\Nretail$, $500$ ensembles of size $\Nretail$ were generated by sampling with replacement from the \largeset{} set. Figures show the mean $\pm$ standard deviation (std) of the metrics in Definition \ref{def:costsandrisks} taken over the $500$ \smallset{} ensembles. Results are shown for paraphrase detection on the \textbf{MRPC} dataset using an ensemble of fine-tunings of the \textbf{BERT} model.}
\end{figure*}

\subsection{Comparison of proposed consistency measure against warehouse multiplicity metrics}\label{sec: Comp consistency multiplicity}

\begin{table*}[t]
\caption{Absolute Spearman correlation between consistency measures and multiplicity metrics when applied to the \eqref{tab:mrpc_corr} paraphrase detection task on the \textbf{MRPC} dataset, and \eqref{tab:sst2_corr} sentiment classification task on the \textbf{SST2} dataset using an ensemble of independently fine-tuned \textbf{BERT} models. Higher values indicate greater agreement between the consistency measures (defined in Appendix~\ref{sec: Related measures}) computed for small retail ensembles and the multiplicity metrics (defined in Appendix~\ref{sec: Multiplicity metrics}) computed from the entire Rashomon warehouse set. Rows highlighted in light blue correspond to the proposed consistency measures, $\consistency(x)$ and $\localmarginconsistency(x)$.} 
\label{tab:correlation_with_multiplicity_metrics}
\centering

\begin{subtable}[t]{0.572\textwidth}
\centering
\scriptsize
\setlength{\tabcolsep}{2.5pt}

\caption{\textbf{MRPC}}
\label{tab:mrpc_corr}

\resizebox{\linewidth}{!}{
\begin{tabular}{|p{1.6cm}|c|c|c|c|c|}
\hline
\multirow{4}{1.6cm}{\centering
\textbf{Consistency}\\[-0.5mm]
\textbf{measures}\\[-0.5mm]
\textbf{using \smallset}\\[-0.5mm]
\textbf{ensembles}}
&
\multicolumn{5}{c|}{}\\
&
\multicolumn{5}{c|}{\textbf{Multiplicity metrics}}\\
&
\multicolumn{5}{c|}{\textbf{using Rashomon \largeset{}}}\\
\cline{2-6}
&
$\pairwisedisag(x)$ &
$\disagens(x)$ &
$\localdisagens(x)$ &
$\predvar(x)$ &
$\predrange(x)$\\
\hline
\rowcolor{ourgrey}
\multicolumn{6}{|c|}{$\Nretail=1$}\\
\hline
\rowcolor{ourblue}
$\consistency(x)$ & $\mathbf{0.793}$ & $\mathbf{0.793}$ & $\mathbf{0.798}$ & $\mathbf{0.866}$ & $\mathbf{0.838}$\\
\hline
$\margin(x)$ & $0.786$ & $0.786$ & $0.762$ & $0.865$ & $0.836$\\
\hline
\rowcolor{ourblue}
$\localmarginconsistency(x)$ & $0.755$ & $0.755$ & $0.778$ & $0.841$ & $0.816$\\
\hline
$\localmargin(x)$ & $0.696$ & $0.696$ & $0.725$ & $0.773$ & $0.752$\\
\hline
$\localvar(x)$ & $0.712$ & $0.712$ & $0.721$ & $0.772$ & $0.751$\\
\hline
$\hammancons(x)$ & $0.679$ & $0.679$ & $0.668$ & $0.781$ & $0.760$\\
\hline
\rowcolor{ourgrey}
\multicolumn{6}{|c|}{$\Nretail=5$}\\
\hline
\rowcolor{ourblue}
$\consistency(x)$ & $\mathbf{0.885}$ & $\mathbf{0.885}$ & $\mathbf{0.891}$ & $\mathbf{0.957}$ & $\mathbf{0.929}$\\
\hline
$\margin(x)$ & $0.851$ & $0.851$ & $0.821$ & $0.927$ & $0.898$\\
\hline
\rowcolor{ourblue}
$\localmarginconsistency(x)$ & $0.848$ & $0.848$ & $0.888$ & $0.926$ & $0.897$\\
\hline
$\localmargin(x)$ & $0.814$ & $0.814$ & $0.864$ & $0.896$ & $0.868$\\
\hline
$\localvar(x)$ & $0.851$ & $0.851$ & $0.881$ & $0.928$ & $0.903$\\
\hline
$\hammancons(x)$ & $0.713$ & $0.713$ & $0.691$ & $0.823$ & $0.804$\\
\hline
\rowcolor{ourgrey}
\multicolumn{6}{|c|}{$\Nretail=10$}\\
\hline
\rowcolor{ourblue}
$\consistency(x)$ & $\mathbf{0.893}$ & $\mathbf{0.893}$ & $0.897$ & $\mathbf{0.964}$ & $\mathbf{0.937}$\\
\hline
$\margin(x)$ & $0.884$ & $0.884$ & $0.875$ & $0.957$ & $0.928$\\
\hline
\rowcolor{ourblue}
$\localmarginconsistency(x)$ & $0.883$ & $0.883$ & $\mathbf{0.899}$ & $0.953$ & $0.927$\\
\hline
$\localmargin(x)$ & $0.850$ & $0.850$ & $0.869$ & $0.926$ & $0.898$\\
\hline
$\localvar(x)$ & $0.887$ & $0.887$ & ${0.883}$ & $0.956$ & $0.932$\\
\hline
$\hammancons(x)$ & $0.660$ & $0.660$ & $0.636$ & $0.777$ & $0.762$\\
\hline
\end{tabular}
}
\end{subtable}
\begin{subtable}[t]{0.4\textwidth}
\centering
\scriptsize
\setlength{\tabcolsep}{2.5pt}

\caption{\textbf{SST2}}
\label{tab:sst2_corr}

\resizebox{\linewidth}{!}{%
\begin{tabular}{|c|c|c|c|c|}
\hline
\multicolumn{5}{|c|}{}\\
\multicolumn{5}{|c|}{\textbf{Multiplicity metrics}}\\
\multicolumn{5}{|c|}{\textbf{using Rashomon \largeset{}}}\\
\hline
$\pairwisedisag(x)$ &
$\disagens(x)$ &
$\localdisagens(x)$ &
$\predvar(x)$ &
$\predrange(x)$\\
\hline
\rowcolor{ourgrey}
\multicolumn{5}{|c|}{$\Nretail=1$}\\
\hline
\rowcolor{ourblue}
$\mathbf{0.686}$ & $\mathbf{0.686}$ & $0.709$ & $0.899$ & $\mathbf{0.884}$\\
\hline
$0.674$ & $0.674$ & $0.682$ & $0.889$ & $0.874$\\
\hline
\rowcolor{ourblue}
$0.682$ & $0.682$ & $\mathbf{0.711}$ & $\mathbf{0.900}$ & $\mathbf{0.884}$\\
\hline
$0.633$ & $0.633$ & $0.670$ & $0.827$ & $0.808$\\
\hline
$0.653$ & $0.653$ & $0.676$ & $0.830$ & $0.817$\\
\hline
$0.244$ & $0.244$ & $0.269$ & $0.430$ & $0.419$\\
\hline
\rowcolor{ourgrey}
\multicolumn{5}{|c|}{$\Nretail=5$}\\
\hline
\rowcolor{ourblue}
$\mathbf{0.753}$ & $\mathbf{0.753}$ & $0.768$ & $\mathbf{0.956}$ & $\mathbf{0.939}$\\
\hline
$0.748$ & $0.748$ & $0.747$ & $0.950$ & $0.933$\\
\hline
\rowcolor{ourblue}
$0.751$ & $0.751$ & $\mathbf{0.774}$ & $0.951$ & $0.934$\\
\hline
$0.740$ & $0.740$ & $0.767$ & $0.936$ & $0.919$\\
\hline
$0.744$ & $0.744$ & $0.765$ & $0.947$ & $0.931$\\
\hline
$0.202$ & $0.202$ & $0.224$ & $0.413$ & $0.403$\\
\hline
\rowcolor{ourgrey}
\multicolumn{5}{|c|}{$\Nretail=10$}\\
\hline
\rowcolor{ourblue}
$\mathbf{0.762}$ & $\mathbf{0.762}$ & $0.776$ & $\mathbf{0.971}$ & $\mathbf{0.956}$\\
\hline
$0.758$ & $0.758$ & $0.762$ & $0.970$ & $0.953$\\
\hline
\rowcolor{ourblue}
$0.758$ & $0.758$ & $\mathbf{0.780}$ & $0.964$ & $0.950$\\
\hline
$0.751$ & $0.751$ & $0.776$ & $0.957$ & $0.942$\\
\hline
$0.758$ & $0.758$ & $0.777$ & $0.961$ & $0.948$\\
\hline
$0.187$ & $0.187$ & $0.199$ & $0.390$ & $0.383$\\
\hline
\end{tabular}}
\end{subtable}
\end{table*}

In the previous sections, we showed that the proposed consistency measure can be used effectively as an auditing criterion. An equally important question is whether it can also pre-emptively capture predictive multiplicity. In this section, we evaluate the extent to which the proposed consistency score computed using a single \smallset{} ensemble of size $\Nretail$ correlates with multiplicity metrics computed using the Rashomon \largeset. A strong correlation would indicate that consistency computed from a single retail ensemble can serve as a practical proxy for predictive multiplicity in deployment settings where the full Rashomon \largeset{} is unavailable.


For a fixed test instance $x\in\dtest$, the multiplicity metrics in Appendix \ref{sec: Multiplicity metrics},  pairwise disagreement \citep{black2022model}, discrepancy \citep{marx2020predictive}, local discrepancy, prediction variance \citep{hamman2024quantifying,watson2023predictive}, and prediction range \citep{hamman2024quantifying,watson2023predictive} (Section~\ref{sec: Multiplicity metrics}), quantify different forms of disagreement among models in the Rashomon set. Since the Rashomon set is approximated by a \largeset{}\ $\retailset_{\Nlib}$ of size $\Nlib=100$, these multiplicity metrics are likewise estimated using this set. 


For notational simplicity, we suppress the explicit dependence of our proposed consistency measure $\consistency(x)$ and other related measures defined in Appendix \ref{sec: Related measures} on $\retailset_{\Nretail}$ and $\localtestmat_{\nogaussianpts}$. To assess whether the proposed consistency measure captures predictive multiplicity, we analyze the correlation of $\consistency(x)$ with these warehouse multiplicity metrics. We repeat this analysis for several related consistency measures that can also be computed from a single retail ensemble, to compare them against our proposed consistency measure.  In particular, we compare against the local ensemble consistency measure, $\localconsistency(x)$ (Definition~\ref{def: local consistency}), for which we also establish concentration guarantees in Appendix \ref{sec: Proof for local consistency}. While our proposed consistency measure evaluates the ensemble margin at the test instance and penalizes it by the average local prediction variability, $\localconsistency(x)$ instead computes the average ensemble prediction over the local neighborhood of $x$ before applying the same variability penalty term. We also compare against measures that isolate individual components of consistency, namely margin ($\margin(x)$, Definition~\ref{def: margin}), local margin ($\localmargin(x)$, Definition~\ref{def: Local margin}), and local variability ($\localvar(x)$, Definition~\ref{def: Local variability}). Finally, we compare with the consistency score $\hammancons(x)$ (Definition~\ref{def: Hamman stability}), a prior consistency measure proposed by \citet{hamman2024quantifying} that evaluates the confidence of a single model at the test instance and penalizes it by local prediction variability. As a result, it neither measures the margin relative to the decision boundary nor exploits ensembling over multiple models. 
For fair evaluation, we compare the measures for different ensemble sizes. By comparing their correlations with the warehouse multiplicity metrics, we evaluate whether our proposed consistency measure provides a more informative characterization of predictive multiplicity.

To quantify the relationship between consistency measures and predictive multiplicity, we use the Spearman rank correlation coefficient. Unlike linear correlation, Spearman correlation measures the strength of the monotonic relationship between two quantities based on their rankings, making it well suited for comparing consistency measures and multiplicity metrics that may not exhibit a linear relationship. For each consistency measure and multiplicity metric pair, we compute the Spearman correlation coefficient across the test set, defined as follows.
\begin{definition}[Spearman rank correlation coefficient]
\label{def: Spearman's corr}

Let $a=(a_1,\dots,a_N)$ and, $ b=(b_1,\dots,b_N)]$ be two sets of paired observations. Let $\rank(a_i)$ and $\rank(b_i)$ denote the ranks of $a_i$ and $b_i$, respectively, and define $d_i= \rank(a_i) - \rank(b_i)$.
Assuming there are no tied ranks, the Spearman  correlation coefficient between $a$ and $b$ is
\begin{align}
\rho(a,b)
=
1-\frac{6\sum_{i=1}^{N} d_i^2}
{N(N^2-1)}.
\end{align}

\end{definition}

Table \ref{tab:correlation_with_multiplicity_metrics} reports these correlation values for retail ensembles of size $\Nretail = [1, 5, 10]$. Rows highlighted in light blue correspond to the proposed consistency measures, $\consistency(x)$ and $\localmarginconsistency(x)$, while values in bold represent the highest value achieved for a fixed multiplicity metric and a fixed ensemble size.
Across all multiplicity metrics and for all ensemble sizes, our proposed consistency measure, $\consistency(x)$, and $\localmarginconsistency(x)$ achieve the highest correlation among all consistency measures. In particular, our measure exhibits strong correlations with pairwise disagreement, discrepancy, and local discrepancy, and particularly high correlations with prediction variance and prediction range. Test instances with high consistency scores correspond to a large ensemble margin and low variability of model predictions under local perturbations. A large margin indicates that the ensemble prediction lies far from the decision boundary, while low local variability indicates that the prediction is stable in the neighborhood of the test instance. As a result, these instances tend to exhibit low predictive multiplicity among the models in the \largeset, whereas instances with low consistency scores tend to exhibit substantially higher predictive multiplicity.

Another important observation is the benefit of computing consistency using a retail ensemble rather than a single model. Across nearly all consistency measures, the correlations with the warehouse multiplicity metrics increase as the retail ensemble size increases from $\Nretail=1$ to $\Nretail=5$ and $\Nretail=10$. This demonstrates that aggregating predictions from independently trained models provides a substantially more informative estimate of predictive multiplicity than a single model can. The gains are particularly pronounced when increasing the retail ensemble size from one to five models, with more modest improvements thereafter, suggesting that relatively small retail ensembles are sufficient to capture much of the multiplicity information present in the Rashomon warehouse. Additional experiments with other architectures and datasets are included in Appendix~\ref{sec: Comp consistency multiplicity app}.

We also note that measures based solely on margin information or solely on local variability perform consistently worse than measures that combine both. While margin captures distance relative to the decision boundary, it does not account for the stability of predictions under local perturbations. As a result, high-confidence predictions can still exhibit substantial multiplicity locally. Conversely, local variability measures only the stability of predictions under perturbations and does not account for how close the ensemble prediction lies to the decision boundary. As a result, predictions with low local variability may still have small margins and therefore be susceptible to predictive multiplicity. Combining both margin and local variability provides a more complete characterization of prediction consistency, which explains the stronger correlations achieved by the rows in blue.

\begin{figure*}[t]
\centering
\begin{minipage}[c]{1\textwidth}
\centering
\includegraphics[width=0.24\linewidth]{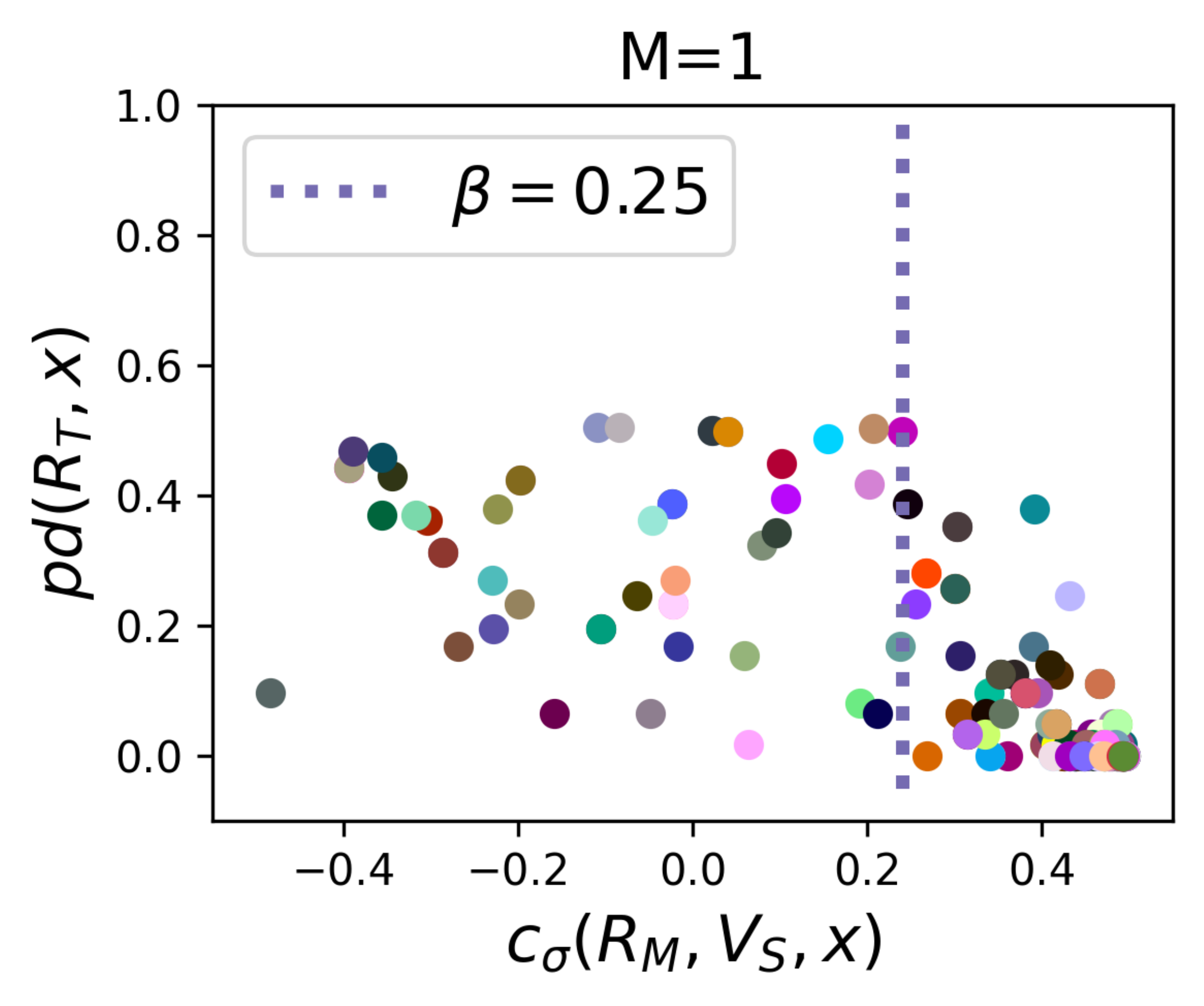}
\includegraphics[width=0.24\linewidth]{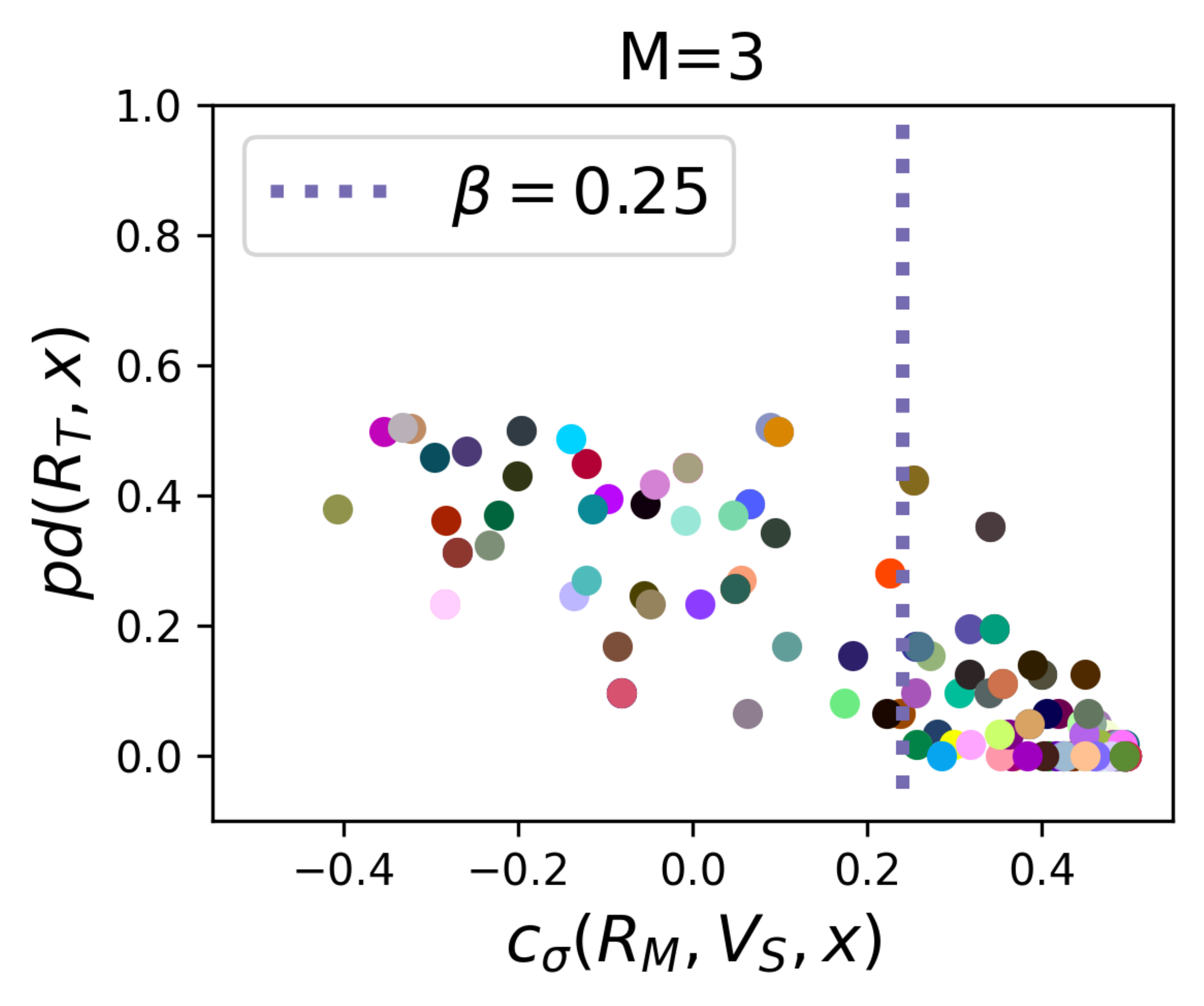}
\includegraphics[width=0.24\linewidth]{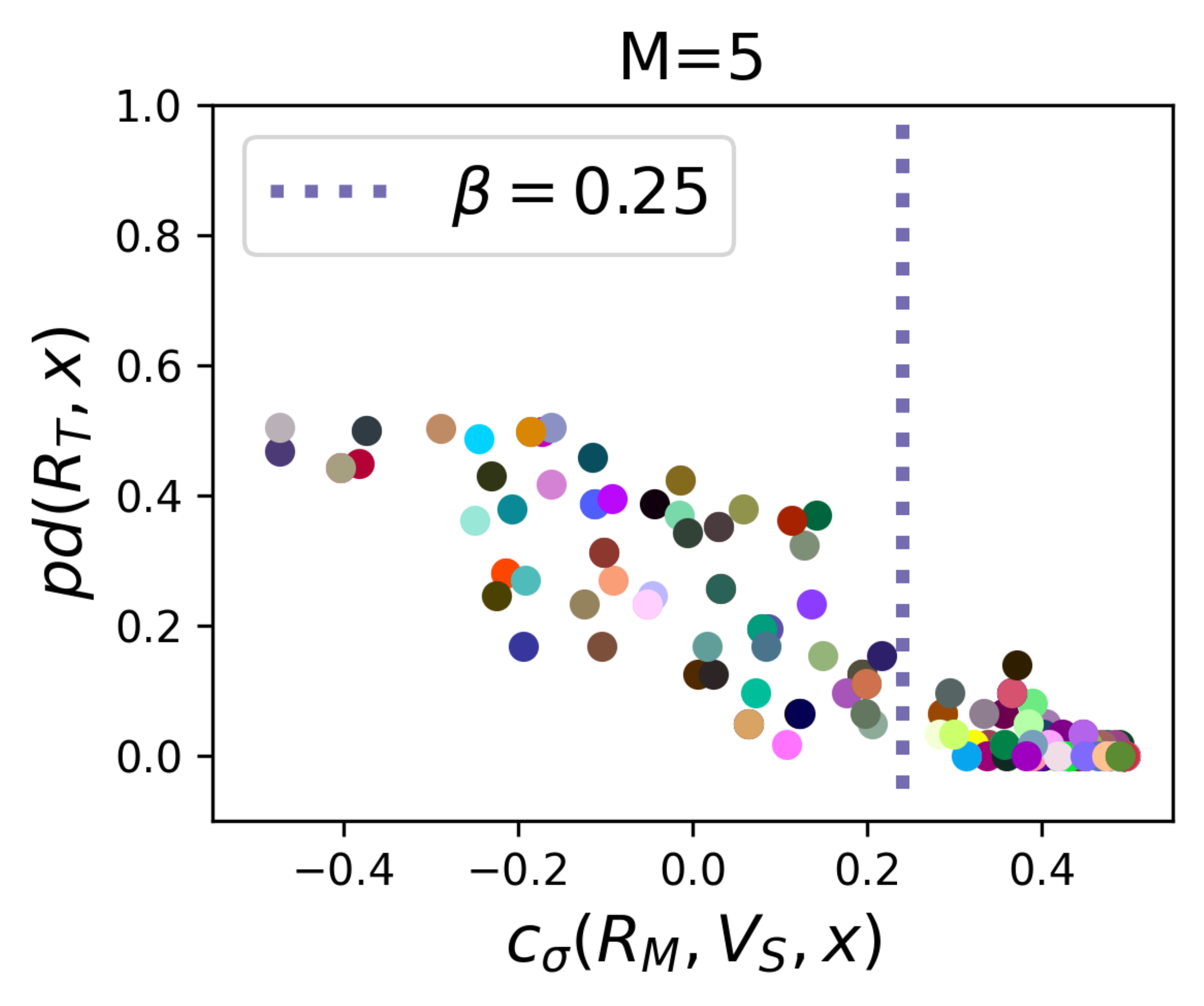}
\includegraphics[width=0.24\linewidth]{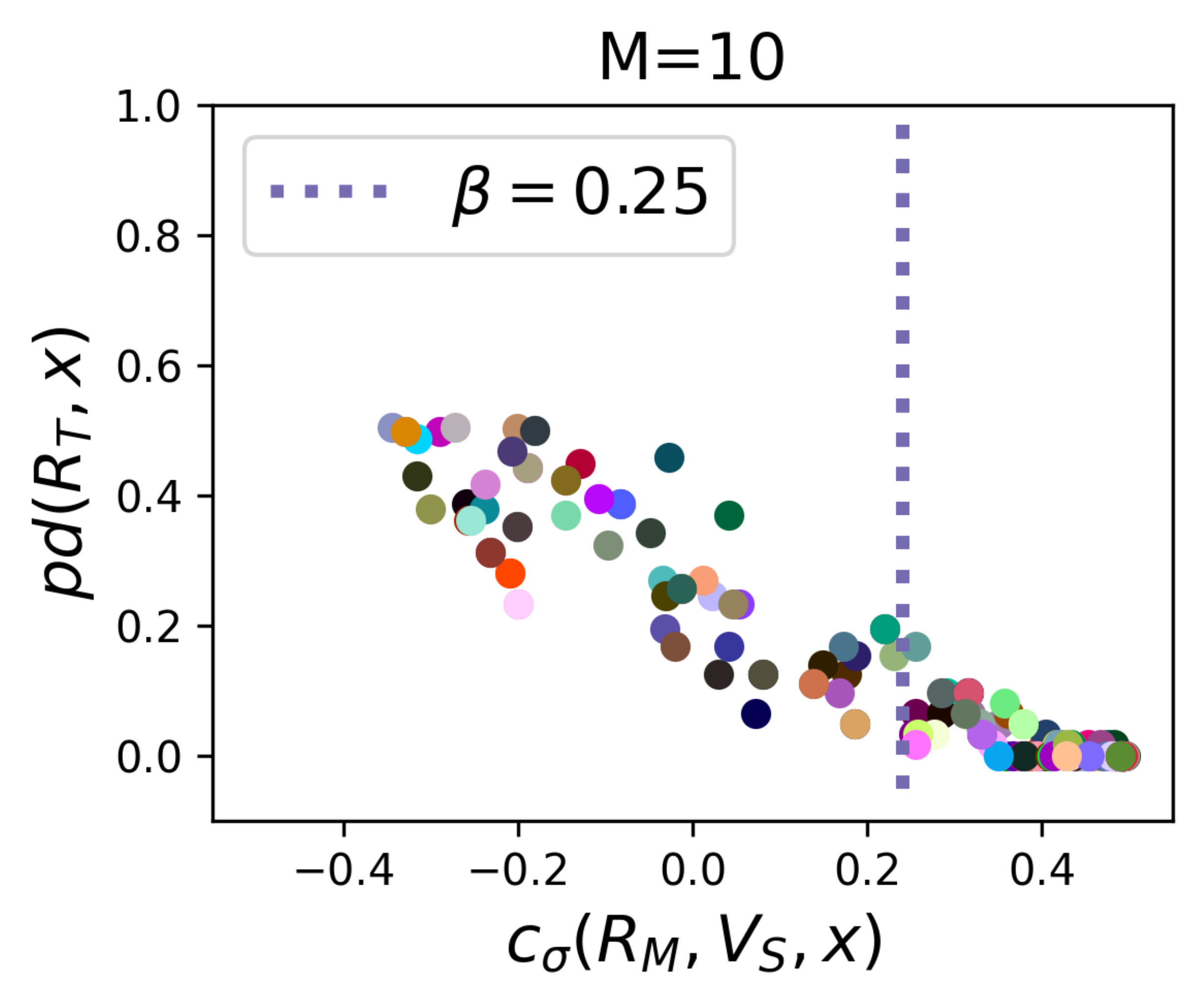}
\end{minipage}
\vspace{1ex}
\begin{minipage}[c]{1\textwidth}
\centering
\includegraphics[width=0.24\linewidth]{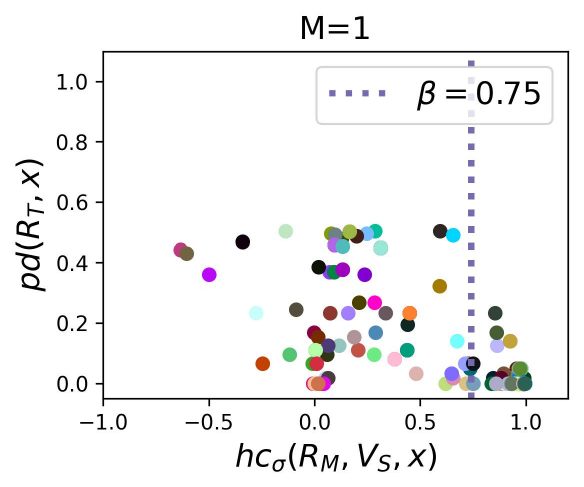}
\includegraphics[width=0.24\linewidth]{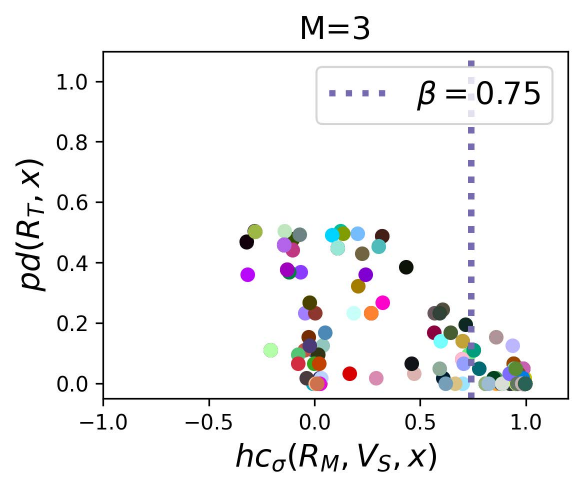}
\includegraphics[width=0.24\linewidth]{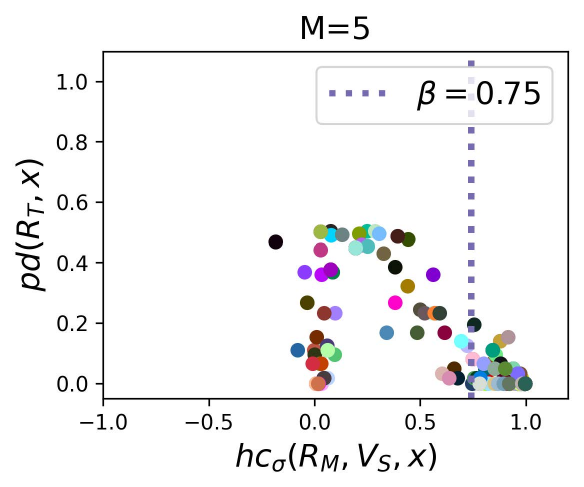}
\includegraphics[width=0.24\linewidth]{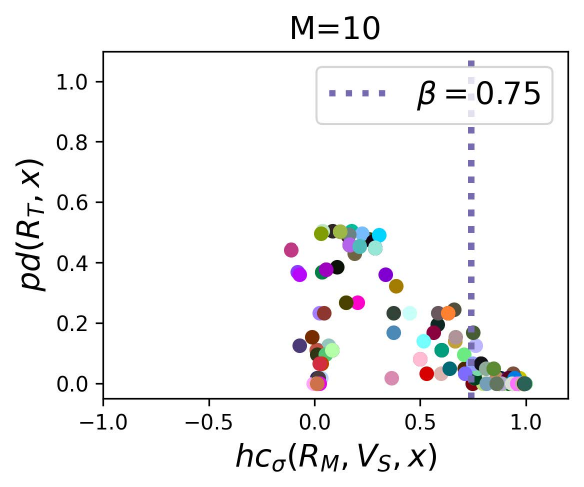}
\end{minipage}

\caption{
Scatter plots showing the relationship between consistency scores computed from a retail ensemble (x-axis) and the pairwise disagreement multiplicity metric~\citep{black2022model}, computed using the Rashomon warehouse (y-axis) for individual test instances. The top row corresponds to the proposed consistency measure $\consistency$ (Definition~\ref{def: Consistency measure}), while the bottom row corresponds to the consistency score $\hammancons$ of  \citet{hamman2024quantifying} (Definition~\ref{def: Hamman stability}). Columns show results for different retail ensemble sizes, and each colored dot represents one of 200 randomly sampled test instances. The proposed consistency measure exhibits a substantially stronger correlation with pairwise disagreement across all retail ensemble sizes than $\hammancons$. Results are shown for paraphrase detection on the \textbf{MRPC} dataset using ensembles of independently fine-tuned \textbf{BERT} models.}\label{fig:bert_mrpc_hammancomp}
\end{figure*}
\subsubsection{A note on \texorpdfstring{$\mbf{hc_{\sigma}}(x)$}{hc}}\label{sec: hamman note}

The score $\hammancons(x)$ in \citet{hamman2024quantifying} exhibits moderate correlations with the multiplicity metrics on MRPC, but its performance deteriorates substantially on SST2, where the correlations are considerably lower than those achieved by the proposed measures. However, the discrepancy between the high correlation of $\hammancons(x)$ with the multiplicity metrics in one experiment and the low correlation in the other is not surprising.  By definition, $\hammancons(x)$ measures the confidence assigned to a class and penalizes it by the average variability of that confidence under local perturbations. As a result, a predictor that consistently predicts the negative class with only moderate confidence can receive a low consistency score despite exhibiting highly stable predictions. An example of this behavior can be seen in Figure~\ref{fig:bert_mrpc_hammancomp}, where test instances with low multiplicity are assigned both high and low scores by $\hammancons(x)$. Several test instances with pairwise disagreement close to $0$ are assigned $\hammancons(x)$ scores near the extremes, close to $0$ or close to $1$. Thus, $\hammancons(x)$ incorrectly identifies many highly consistent test instances as inconsistent, because it depends on the confidence assigned to the predicted class rather than the distance from the decision boundary. We observe this behavior across all multiplicity metrics, the results of which are included in Section \ref{sec: MM vs cons scatter plot} of the Appendix.

In contrast, the proposed consistency measure assigns high scores to nearly all test instances exhibiting low pairwise disagreement in the Rashomon warehouse. Instances with pairwise disagreement close to $0$ are assigned the highest scores (above $0.4$). The low multiplicity of instances with high scores from the proposed consistency measure is by construction, because the score is high for instances with large ensemble margins and low local prediction variability, which implies that predictions across the Rashomon $\largeset$ are stable at these instances.

We also note that in~\citet{hamman2024quantifying}, $\hammancons(x)$  is evaluated with respect to the predicted class. As different models sampled from a Rashomon set may predict different classes for the same test instance, the quantity being evaluated changes whenever the predicted label changes. From our perspective, this complicates both the empirical interpretation and theoretical analysis of the measure, since the consistency score no longer corresponds to a fixed underlying quantity across different models. In contrast, our proposed consistency measure is defined with respect to an ensemble prediction for a fixed class, yielding a quantity whose concentration toward the corresponding Rashomon set value can be analyzed theoretically.

\section{Conclusion}
This work establishes a principled framework for analyzing predictive multiplicity and its impact on downstream auditing using finite ensembles drawn from the Rashomon set. We introduced a consistency measure that combines ensemble margin with local prediction variability, enabling the identification of predictions that remain stable across competing models. We further established theoretical guarantees showing that, under certain mild assumptions, consistency estimates computed from finite ensembles converge to those of the full Rashomon set as the ensemble size and the number of perturbation samples increase.

Our experiments demonstrate three key findings. First, finite ensembles provide a trade-off between audit cost and audit risk. Although a single model is considerably less expensive to deploy than approximating the full Rashomon set with larger ensembles, they exhibit much greater variability in auditing decisions. Additionally, individual models result in a higher volume of erroneous predictions escaping human review, but fewer unnecessary diversions, compared to an ensemble of the full Rashomon set. Even with a modest increase in ensemble size, audit decisions become increasingly stable and closely approximate those of the full Rashomon set. Second, contrary to prior work, our results show that using a single model to characterize predictive multiplicity is insufficient. Even relatively small ensembles consistently achieve substantially stronger correlations with established predictive multiplicity metrics, demonstrating that aggregating independently trained models provides a more faithful characterization of multiplicity. Finally, we show that how consistency is measured is equally important. A margin-based consistency measure penalized by local prediction variability consistently outperforms confidence-based approaches across most experimental settings.

Several directions remain for future work. Our theoretical analysis considers uniformly weighted ensembles, whereas modern prediction systems increasingly employ weighted ensembles and mixture-of-experts architectures, indicating that extending the concentration guarantees to these settings would be valuable. While this work focuses on probabilistic classification, extending the proposed consistency measure and auditing framework to regression and generative foundation models presents an important direction for future research. More broadly, we hope this work encourages the study of predictive multiplicity not only as a property of machine learning models but also of their effects on downstream decision systems that rely on them.

\bibliography{refs}
\bibliographystyle{tmlr}


\appendix

\section{Appendix}\label{sec : Appendix}

\subsection{Notation Summary}

We recap the notation used in the proofs in Table~\ref{Table: Notation summary }.

\begin{table}[htb]
\caption{Table of notation used in this paper}
\label{Table: Notation summary }
\centering
\renewcommand\arraystretch{1.2}
\begin{tabular}{cl}
\toprule
    \textbf{Symbol} & \textbf{Meaning} \\
\midrule
$f_{\param}$ & model from a set $\hypclass$ parameterized by $\varTheta$ \\
$h_{\param}$ & thresholded model output $\Indic{ f_{\param} \ge 1/2}$ \\
$\traindist$ & distribution on $\hypclass$ induced by training \\
$\datadist$ & data distribution on $\cX \times \cY$ \\

$\dtest$ & test dataset \\
$\Ntest$ & size of test data \\
$\emprashomonset(\rashomonparam)$ 
    & empirical Rashomon set with tolerance level $\rashomonparam$ \\
$\retailset = \{f_{k} : k \in [\Nretail]\}$ 
    & \smallset{} set with $\Nretail$ models sampled  $\iid \traindist$ \\ 
$\ensmretset(\retailset,x)$ & ensemble model formed from the \smallset{} set \\
$\ltwohsphere(x,\sigma)$ & Euclidean sphere of radius $\sigma$ centered at $x$ \\ 
$\localtestmat = \{\localtest_{j} : j \in [\nogaussianpts]\}$ & set of perturbed samples of size $\nogaussianpts$ sampled $\iid \Unif{\ltwohsphere(x,\sigma)}$ \\
$\ensmretset(\retailset,\localtestmat)$
    & average prediction of an ensemble of \smallset{} set on $\localtestmat$ \\
$\consistency(\retailset,\localtestmat, x)$
    & consistency measure in Equation \eqref{eq:consistency} \\
$\beta$ & consistency threshold \\ 
$\slackmeasure$ & deviation tolerance in the finite ensemble consistency guarantee\\
$\eta$ & Lipschitz constant of $f_{\param}$\\
$\rashexpctparam$ & stability tolerance of the Rashomon set\\
$\dtestconf(\beta,\sigma,\rashexpctparam)$ & Test instances that are $\beta,\sigma$ consistent and $\rashexpctparam$ stable.\\

 \bottomrule
\end{tabular}

\end{table}

\subsection{Measure concentration tools}

We restate McDiarmid's inequality~\citep{mcdiarmid1989method} in the form in which we use it to make the manuscript more self-contained.

\begin{lemma}[McDiarmid's inequality~\citep{mcdiarmid1989method}]
\label{lemma: mcdiarmid ineq}
Let $\mbf{z}_1,\dots,\mbf{z}_N$ be independent random variables taking values in a set $\mathcal{Z}$.
Let $\Phi:\mathcal{Z}^N \to \mathbb{R}$ be a function of $(\mbf{z}_1,\dots,\mbf{z}_N)$.
Assume $\Phi$ satisfies the bounded differences property: for each $i\in\{1,\dots,n\}$ there exists $c_i\ge0$ such that for all
$z_1,\dots,z_n,z'_i \in \mathcal{Z}$,
\begin{align}
    \big|\Phi(z_1,\dots,z_i,\dots,z_N) - \Phi(z_1,\dots,z'_i,\dots,z_N)\big|
    \le c_i .
    \label{eq:mcdiarmid:boundeddiff}
\end{align}
Then for every $t>0$,
\begin{align}
    \prob_{\mbf{z}}*{
        \abs*{ \Phi(\mbf{z}_1,\dots,\mbf{z}_N)
        - \E[\Phi(\mbf{z}_1,\dots,\mbf{z}_N)]}
        \ge t
        }
    \le
    2\exp\parens*{ -\frac{2t^2}{\sum_{i=1}^N c_i^2} }
\end{align}
\end{lemma}

\begin{lemma}[Hoeffding's Inequality]\label{lemma: Hoeffding's ineq}
Let ${\mbf{z}_{1}, \mbf{z}_{2}, \ldots, \mbf{z}_{N}}$ be independent random variables such that for each $i$, $z_i \in [a,b]$ almost surely. Then for any $\varepsilon > 0$, we have
	\begin{align}
	\prob_{\mbf{z}}*{
		\abs*{ \frac{1}{N}\sum_{i=1}^{N} z_i - \mathbb{E}[z_i] }
		\ge \varepsilon 
		}
	\leq 2 \exp\parens*{ - \frac{2N \varepsilon^2}{(b-a)^2}  }.
	\end{align}
\end{lemma}

 \begin{proof}
     We refer the readers to~\citet{ledoux2001concentration} for proof of this lemma.
 \end{proof}

\subsection{Probabilistic guarantee on our proposed consistency measure}\label{sec: Proofs}

Before proving our main results, we state and prove some lemmas that we will leverage. All our proofs make some common assumptions which we restate below.

\assume*

\begin{lemma}\label{lemma: Concentration ensemble pred at x}
Let $\emprashomonset$, $\rretailset$, and $x$ satisfy Assumption \ref{assume:rashomon} for parameter $\rashexpctparam > 0$. Let $\ensmretset(\retailset,x)$ be the ensemble model defined in Equation \eqref{eq: ensemble at x}.
Define
	\begin{align}
	\gapmargin(\retailset,x) = \abs*{\abs*{\ensmretset(\retailset, x)-\frac{1}{2}}  - \abs*{\E_{\rretailset}\bracks*{\ensmretset(\rretailset,x)} - \frac{1}{2}}}.
    \label{eq:margingap}
	\end{align}
Then for all $\slackgap > 0$,
	\begin{align}
	\prob_{\rretailset}*{
		{\gapmargin(\rretailset, x)} \ge \slackgap } 
	\le 2\exp\parens*{ - \frac{\Nretail\slackgap^2}{2\rashexpctparam^2} }.
	\end{align}
\end{lemma}

\begin{proof}
Fix $\rashexpctparam > 0$ and $x \in \dtestconf$.  
Since $\rretailset$ is sampled i.i.d. according to $\traindist$, we will use McDiarmid's inequality (Lemma \ref{lemma: mcdiarmid ineq}) on the ensemble prediction $\ensmretset(\retailset,x)$. To do so, we must verify the bounded differences property in \eqref{eq:mcdiarmid:boundeddiff}. Let $\retailset = \{f_{1},\dots,f_{k},\dots, f_{\Nretail}\} \subseteq \emprashomonset$ and $\retailset'= \retailset \setminus \{f_k\} \cup \{f'_k\}$, where $f_k' \in \emprashomonset$. Using triangle inequality and the fact that $\emprashomonset$ is $\rashexpctparam$-stable at $x$ (Definition \ref{def: Rashomon stability}):
    \begin{align}
    \abs*{ \ensmretset(\retailset,x) - \ensmretset(\retailset',x) }
    &= \abs*{\frac{1}{\Nretail}\sum_{k=1}^{\Nretail}f_{k}(x) 
        - \frac{1}{\Nretail}\sum_{k=1}^{\Nretail}f'_{k}(x)} 
        \notag \\
    &= \frac{1}{\Nretail}\abs*{f_k{(x)} - f'_k{(x)}} 
        \notag \\
    &\le \frac{1}{\Nretail}\abs*{f_k{(x)} - \E_{\rf}[\rf{(x)}]} 
        + \frac{1}{\Nretail}\abs*{f'_k{(x)} - \E_{\rf}[\rf{(x)}]}
        \notag \\
    &\le \frac{2 \rashexpctparam}{\Nretail}.
\end{align}
Now applying Lemma \ref{lemma: mcdiarmid ineq} with $c_k = \frac{2 \rashexpctparam}{\Nretail}$, we have
    \begin{align}
      \prob_{\rretailset}*{
            \abs*{ \ensmretset(\rretailset,x)
                 - \expect_{\rretailset}*{\ensmretset(\rretailset,x)} }
          \ge \slackgap }
      \le 
      2\exp\parens*{ \frac{-\Nretail\slackgap^2}{2\rashexpctparam^2} }.
    \end{align}
Let $\gapmargin(\retailset,x )$ be defined as in \eqref{eq:margingap}. From the reverse triangle inequality,
    \begin{align}
    {\gapmargin(\retailset,x)} 
        \le \abs*{\ensmretset(\retailset,x) 
            - \expect_{\rretailset}*{\ensmretset(\rretailset,x)}}.
    \end{align}
Therefore the event $\{{\gapmargin(\retailset, x)} \ge \slackgap\}$ implies  $\{\abs*{\ensmretset(\retailset,x) - \E_{\rretailset}\big[\ensmretset(\rretailset,x)\big] } \ge \slackgap\}$, so
    \begin{align}
    \prob_{\rretailset}*{{\gapmargin(\rretailset,x)} \ge \slackgap} 
    \le 2\exp\parens*{
        \frac{-\Nretail\slackgap^2}{2\rashexpctparam^2} }.
    \end{align}
\end{proof}

Our main result involves showing measure concentration over perturbations and ensemble models. Since both are generated randomly, we are interested in the random variables corresponding to the differences
    \begin{align}
    \rtestdiff_{j,k}(x) = \abs*{\rf_{k}(x) - \rf_{k}(\rlocaltest_{j})} \qquad j \in [\nogaussianpts], k \in [\Nretail],
    \label{eq:rtestdiff}
    \end{align}
where $\rretailset = \{\rf_k : k \in [\Nretail]\}$ are drawn i.i.d.~from $\emprashomonset$ and $\rlocaltestmat = \{ \rlocaltest_j : j \in [\nogaussianpts] \}$ are draw  uniformly from $\ltwohsphere(x,\sigma)$. Define 
    \begin{align}
    \allavg(x) &= \frac{1}{\Nretail} \sum_{k=1}^{\Nretail}
            \frac{1}{\nogaussianpts} \sum_{j=1}^{\nogaussianpts} \rtestdiff_{j,k}(x).
    \label{eq:conc:allavg} \\
    \expallavg(x) &= \expect_{\rretailset,\rlocaltestmat}*{ \allavg(x) }.
    \label{eq:conc:expallavg}
    \end{align}
 We will analyze the concentration of these quantities against their expectations.   
\begin{lemma}\label{lemma:difference:conditional}
Let $\emprashomonset$, $\rretailset$, and $x$ satisfy Assumption \ref{assume:rashomon} for parameters $\sigma > 0$. Let $\rlocaltestmat =  \{\rlocaltest_{j}: j \in [\nogaussianpts]\}$ be a random set of perturbations of $x$, distributed $\iid \Unif(\ltwohsphere(x,\sigma))$. 
Define
    \begin{align}
    \condavg(\retailset,x)
        = \expect_{\rlocaltestmat}*{ 
            \frac{1}{\Nretail} \sum_{k=1}^{\Nretail}
            \frac{1}{\nogaussianpts} \sum_{j=1}^{\nogaussianpts}
                \abs*{f_{k}(x) - f_{k}(\rlocaltest_{j})} 
             }.
    \end{align}
Then for all $\slackparam > 0$
    \begin{align}
    \label{eq:ball:condbound}
    \prob_{\rretailset}*{
        \abs*{
        \condavg(\rretailset,x) 
            - \expallavg(x)
            }
        \ge
        \slackparam
    }
    \le
    2 \exp\parens*{\frac{-2\Nretail\slackparam^2}{\eta^2\sigma^{2}} }.
    \end{align}
\end{lemma}

\begin{proof}
Since each element of $\emprashomonset$ is locally Lipschitz in the neighborhood $\ltwohsphere(x,\sigma)$ with constant $\eta$, for every perturbation $\localtest_{j}$,
    \begin{align}
    \abs*{f_{k}(x) - f_{k}(\localtest_{j})} \le \eta\normtwo{x - \localtest_{j}} \le \eta \sigma.
    \end{align}
This in turn means, 
    \begin{align}
    \E_{\localtestmat}\bracks*{\frac{1}{\nogaussianpts} \sum_{j=1}^{\nogaussianpts} \abs*{f_{k}(x) - f_{k}(\rlocaltest_{j})}} 
    \le \eta \sigma.
    \end{align}
Thus, $ \E_{\localtestmat}\bracks*{\frac{1}{\nogaussianpts} \sum_{j=1}^{\nogaussianpts} \abs*{\rf_{k}(x) - \rf_{k}(\rlocaltest_{j})}} \in [0, \eta\sigma]$ is an i.i.d. random variable since $\{\rf_{k} : k \in [\Nretail]\}$ are i.i.d..  Applying Hoeffding's inequality (Lemma \ref{lemma: Hoeffding's ineq}) over this randomness, we have 
    \begin{align}
    \prob_{\rretailset}*{
        \abs*{ \condavg(\rretailset,x) - \expallavg(x) }
        \ge \slackparam }
    &\le 2 \exp\parens*{ \frac{-2\Nretail\slackparam^2}{\eta^2\sigma^{2}} }.
    \end{align}
as desired.
\end{proof}

\begin{lemma}\label{Lemma: Subgaussian conc random}
Let $\emprashomonset$, $\rretailset$, and $x$ satisfy Assumption \ref{assume:rashomon} for parameters  $\sigma > 0$. Let $\rlocaltestmat =  \{\rlocaltest_{j}: j \in [\nogaussianpts]\}$ be a random set of perturbations of $x$, distributed $\iid \Unif(\ltwohsphere(x,\sigma))$. 
Then for any $\slackgap > 0$, 
    \begin{align}
    \prob_{\rretailset,\rlocaltestmat}*{
        \abs*{\allavg(x) - \expallavg(x)}
    \ge \slackgap }
    \le 2\exp\parens*{\frac{ -\nogaussianpts\slackgap^{2}}{2\eta^2\sigma^{2}} } 
        +  2\exp\parens*{\frac{-\Nretail\slackgap^2}{2\eta^2\sigma^{2}} }
    \end{align}
\end{lemma}

\begin{proof}
Let $\rgapvar(\rretailset,\rlocaltestmat,x) = \abs*{\allavg(x) - \expallavg(x)}$ be the random deviation of the empirical variability term from its population counterpart. We decompose this gap as follows:
    \begin{align}
    \rgapvar(\rretailset,\rlocaltestmat,x)
    \le \underbrace{\abs*{ \allavg(x) - \condavg(\rretailset,x) }}_{\gapcondr(\rretailset,\rlocaltestmat,x)}
        + \underbrace{ \abs*{ \condavg(\rretailset,x) - \expallavg(x) } }_{\gaprandr(\rretailset,x)}.
    \end{align}
Lemma \ref{lemma:difference:conditional} shows that for any $\slackparam > 0$
    \begin{align}
    \prob_{\rretailset}*{ \gaprandr(\rretailset,x) \ge \slackparam } \le 2 \exp\parens*{ \frac{-2\Nretail\slackparam^2}{\eta^2\sigma^{2}} }.
    \end{align}
We now argue that $\gapcondr(\rretailset,\rlocaltestmat,x)$ is also bounded with high probability. Let $R = \{f_k :  k \in [\Nretail]\}$ be a realization of $\rretailset$ and let
    \begin{align}
    \mbf{z}_j = \frac{1}{\Nretail} \sum_{k=1}^{\Nretail} \abs*{ f_k(x) - f_k(\rlocaltest_j) }.
    \end{align}
The random variables $\mbf{z}_j$ are i.i.d. since $\{\rlocaltest_j :  j \in [\nogaussianpts]\}$ are i.i.d.. Since $\abs*{ f_k(x) - f_k(\rlocaltest_j)} \le \eta \sigma$, applying Hoeffding's inequality to $\{\mbf{z}_j :  j \in [\nogaussianpts]\}$ yields
    \begin{align}\label{eq: variability cond bound}
    \prob_{\rlocaltestmat}*{ 
        \abs*{ \frac{1}{\nogaussianpts} \sum_{j=1}^{\nogaussianpts } \mbf{z}_j 
            -{\condavg(\retailset,x)} } \ge \slackparam }
    &\le 2\exp\parens*{ \frac{-2S\slackparam^2}{\eta^2\sigma^{2}} }
    \end{align}
This is an upper bound on $\gapcondr(\retailset,\rlocaltestmat,x)$.
Since Equation \eqref{eq: variability cond bound} holds uniformly for all $\rretailset = \retailset$, applying the law of total probability yields, 
\begin{align}
    \P_{\rretailset,\rlocaltestmat}\bracks*{\abs*{\gapcondr(\rretailset, \rlocaltestmat, x)} \ge \slackparam}
    = \sum_{\retailset}\P_{\rlocaltestmat \lvert \rretailset=\retailset}\bracks*{\gapcondr(\retailset,\rlocaltestmat,x) \ge {\slackparam} \lvert \rretailset=\retailset } \P\bracks*{\rretailset=\retailset}
      \le    2\exp\parens*{ \frac{-2S\slackparam^2}{\eta^2\sigma^{2}} }, \label{eq: variability uniform bound}
\end{align}
Then setting $\slackparam = \slackgap/2$ we have 
    \begin{align}
    \prob_{\rretailset,\rlocaltestmat}*{ 
        \gapvar(\rretailset,\rlocaltestmat,x) \ge \slackgap }
    &\le 
    \prob_{\rretailset,\rlocaltestmat}*{
        \gapcondr(\rretailset,\rlocaltestmat,x) \ge \slackgap/2, 
        \gaprandr(\rretailset,x) \ge \slackgap/2 } 
     \\ \notag
    &\le 
    \prob_{\rretailset,\rlocaltestmat}*{
        \gapcondr(\rretailset,\rlocaltestmat, x) \ge \slackgap/2 }
    + \prob_{\rretailset,\rlocaltestmat}*{ 
        \gaprandr(\rretailset, x) \ge \slackgap/2 }
    \\ \notag
    &\le 2\exp\parens*{ \frac{-S\slackgap^2}{2\eta^2\sigma^{2}} } +
        2\exp\parens*{ \frac{-\Nretail\slackgap^2}{2 \eta^2\sigma^{2}} }.
    \end{align}
 Here we used the fact that for $\gapvar(\rretailset,\rlocaltestmat,x) \ge \slackgap$ to hold, at least one of the terms $\gapcondr(\rretailset,\rlocaltestmat,x)$ and $\gaprandr(\rretailset ,x)$ should be larger than $\slackgap/2$.
 \end{proof}

\subsubsection{Proof of main results}\label{sec: Proof of main results}
Theorem \ref{Theorem: Main Theorem} bounds the probability that the consistency score computed using a finite retail ensemble falls significantly below the corresponding consistency score of the Rashomon set.  Let 
\begin{align}
    \gap(\retailset,\localtestmat, x) = \abs*{\consistency(\retailset,\localtestmat,x) - \consistency(\emprashomonset,x)},
\end{align}
where $\consistency(\retailset,\localtestmat,x) =  \abs*{\ensmretset(\retailset, x)-\frac{1}{2}} - \frac{1}{\nogaussianpts}\sum_{j=1}^{\nogaussianpts}\frac{1}{\Nretail}\sum_{k=1}^{\Nretail}\abs*{f_{k}(x) - f_{k}(\localtest_{j})}$ from Definition \ref{def: Consistency measure} and  $\consistency(\emprashomonset,x) = \abs*{\E_{\rf}\bracks*{\rf(x)} - \frac{1}{2}} - \E_{\rf,\rlocaltest}\bracks*{\abs*{\rf(x) - \rf(\localtest)}}$. Since the models in $\emprashomonset$ are $\iid \traindist$ and the perturbations in $\localtestmat$ are $\iid$ $\Unif(\ltwohsphere(x, \sigma))$, equivalently, $\consistency(\emprashomonset,x) = \abs*{\E_{\rretailset}\bracks*{\ensmretset(\rretailset,x)} - \frac{1}{2}} - \E_{\rretailset,\rlocaltestmat}\bracks*{\frac{1}{\nogaussianpts}\sum_{j=1}^{\nogaussianpts}\frac{1}{\Nretail}\sum_{k=1}^{\Nretail}\abs*{\rf_{k}(x) - \rf_{k}(\localtest_{j})}}$.  Since $x \in \dtestconf$, by Definition \ref{Definition: beta conf} of $(\beta,\sigma)$ consistency of $x$, $\consistency(\emprashomonset,x) \ge \beta$. Therefore, if we can bound the deviation ${\gap(\retailset, \localtestmat,x)}$ as ${\gap(\retailset, \localtestmat, x)} \le \slackmeasure$, we can bound $\consistency(\retailset,\localtestmat,x)$ as,
    \begin{align}\label{eq: consbeta}
        \consistency(\retailset,\localtestmat,x) \ge \consistency(\emprashomonset,x) - \slackmeasure \ge \beta - \slackmeasure.
    \end{align}
Thus, to show $\consistency(\retailset,\localtestmat,x) \ge \beta - \slackmeasure$, it suffices to upper bound the term ${\gap(\retailset,\localtestmat, x)}$ which we prove next.

\consistent*
\begin{proof}
 Since ${\gap(\rretailset,\rlocaltestmat,x)} = \abs*{\consistency(\rretailset,\rlocaltestmat,x) - \consistency(\emprashomonset,x)}$, we have
 \begin{align}\label{eq: gmeasure}
     &{\gap(\retailset,\localtestmat, x)} = \notag\\
     &\abs*{\abs*{\ensmretset(\retailset, x)-\frac{1}{2}} - \frac{1}{\nogaussianpts}\sum_{j=1}^{\nogaussianpts}\frac{1}{\Nretail}\sum_{k=1}^{\Nretail}\abs*{f_{k}(x) - f_{k}(\rlocaltest_{j})} - \abs*{\E_{\rretailset}\bracks*{\ensmretset(\rretailset,x)} - \frac{1}{2}} + \E_{\rretailset,\rlocaltestmat}\bracks*{\frac{1}{\nogaussianpts}\sum_{j=1}^{\nogaussianpts}\frac{1}{\Nretail}\sum_{k=1}^{\Nretail}\abs*{\rf_{k}(x) - \rf_{k}(\rlocaltest_{j})}}}.
 \end{align}

After rearranging the terms as in Equation \eqref{eq: gmeasure}, we represent them with two functions for notation convenience. The function
 \begin{align}
     \gapmargin(\retailset,x) = \abs*{\abs*{\ensmretset(\retailset, x)-\frac{1}{2}}  - \abs*{\E_{\rretailset}\bracks*{\ensmretset(\rretailset,x)} - \frac{1}{2}}},
 \end{align} 
 represents the difference between the margin of a single ensemble and the margin of the expected ensemble. The other term
 \begin{align}
     \gapvar(\retailset,\localtestmat, x) =  \abs*{\frac{1}{\nogaussianpts}\sum_{j=1}^{\nogaussianpts}\frac{1}{\Nretail}\sum_{k=1}^{\Nretail}\abs*{f_{k}(x) - f_{k}(\localtest_{j})} - \E_{\rretailset,\rlocaltestmat}\bracks*{\frac{1}{\nogaussianpts}\sum_{j=1}^{\nogaussianpts}\frac{1}{\Nretail}\sum_{k=1}^{\Nretail}\abs*{\rf_{k}(x) - \rf_{k}(\rlocaltest_{j})}}},
 \end{align}
 captures how predictions vary when we perturb the input $x$. The function compares the average local stability of an ensemble to its population counterpart.
  From the triangle inequality we have,
  \begin{align}
     \gap(\retailset,\localtestmat,x) \le  \gapmargin(\retailset,x) +  \gapvar(\retailset,\localtestmat,x).
  \end{align}

From Lemma \ref{lemma: Concentration ensemble pred at x}, $ \P_{\rretailset}\bracks*{{\gapmargin(\rretailset, x)} \ge {\slackgap}} \le 2\exp\Big(
        \frac{-\Nretail\slackgap^2}{2\rashexpctparam^{2}}\Big)$, and from Lemma \ref{Lemma: Subgaussian conc random}, $\P_{\rretailset,\rlocaltestmat}\bracks*{{\gapvar(\rretailset,\rlocaltestmat, x)} \ge {\slackgap}} \le 2\exp\parens*{\frac{
        -\nogaussianpts\slackgap^{2}}{2\eta^2\sigma^{2}}\bigg) +  2\exp\Big(
        \frac{-\Nretail\slackgap^2}{2\eta^2\sigma^{2}} }$.
For $\gap(\rretailset,\rlocaltestmat,x)$ to be greater than $\slackmeasure$, at least one of the terms $\gapmargin(\rretailset, x)$ or $\gapvar(\rretailset,\rlocaltestmat,x)$ should be greater than $\frac
{\slackmeasure}{2}$. Fixing $\slackgap = \frac{\slackmeasure}{2}$, we get
\begin{align}
    \P_{\rretailset,\rlocaltestmat}\bracks*{{\gap(\rretailset,\rlocaltestmat,x)} \ge \slackmeasure} \le \P_{\rretailset}\bracks*{{\gapmargin(\rretailset,x)} \ge \frac{\slackmeasure}{2}} + \P_{\rretailset,\rlocaltestmat}\bracks*{{\gapvar(\rretailset,\rlocaltestmat,x)} \ge \frac{\slackmeasure}{2}}.
\end{align}
Thus,
\begin{align}
    \P_{\rretailset,\rlocaltestmat}\bracks*{{\gap(\rretailset,\rlocaltestmat,x)} \ge \slackmeasure} \le 2\exp\Big(
        \frac{-\Nretail\slackmeasure^2}{8\rashexpctparam^{2}}\Big) + 2\exp\parens*{\frac{
        -\nogaussianpts\slackmeasure^{2}}{8\eta^2\sigma^{2}}\bigg) +  2\exp\Big(
        \frac{-\Nretail\slackmeasure^2}{8\eta^2\sigma^{2}}}.
\end{align}
It follows,
\begin{align}
      \P_{\rretailset,\rlocaltestmat}\bracks*{{\gap(\rretailset,\rlocaltestmat,x)} \le \slackmeasure} \ge 1 - 2\exp\Big(
        \frac{-\Nretail\slackmeasure^2}{8\rashexpctparam^{2}}\Big) - 2\exp\parens*{\frac{
        -\nogaussianpts\slackmeasure^{2}}{8\eta^2\sigma^{2}}\bigg) -  2\exp\Big(
        \frac{-\Nretail\slackmeasure^2}{8\eta^2\sigma^{2}}}.
\end{align}
When combined with Equation $\eqref{eq: consbeta}$ yields,
\begin{align}
    \P_{\rretailset,\rlocaltestmat}\bracks*{\consistency(\rretailset,\rlocaltestmat, x) \ge \beta - \slackmeasure} \ge 1 - 2\exp\Big( \frac{-2\Nretail\slackmeasure^2}{8\rashexpctparam^{2}}\Big) - 2\exp\parens*{\frac{
        -\nogaussianpts\slackmeasure^{2}}{8\eta^2\sigma^{2}}\bigg) -  2\exp\Big(
        \frac{-\Nretail\slackmeasure^2}{8\eta^2\sigma^{2}}}.
\end{align}

\end{proof}

\subsection{Local consistency}
We also provide a consistency guarantee using a related measure that assesses the average local margin of an ensemble of models and penalizes it with the same variability term as before. We call this measure the \emph{local consistency} (Definition~\ref{def: local consistency}). Under this new measure we define consistent test points with respect to the empirical Rashomon set $\emprashomonset$ in Definition~\ref{Definition: local beta conf}, and show that the ensemble of a \smallset{} set $\retailset$ will also remain consistent at these test points with high probability (exponential in the number of samples $\nogaussianpts$ from a local neighborhood around $x$ and the size $\Nretail$ of the \smallset{} set).

\begin{definition}[$(\beta,\sigma)$-local consistency of $\emprashomonset$]
\label{Definition: local beta conf}
     Given $\emprashomonset$ with distribution $\tau$,  a random model $\rf$ $\iid \tau$, an instance $x \in \dtest$ , a random perturbation $\rlocaltest$ $\iid \Unif(\ltwohsphere(x,\sigma))$, the local expected prediction $\E_{\rf,\rlocaltest}[\rf(\rlocaltest)]$ at x, and parameters $\beta \in [-\frac{1}{2},\frac{1}{2}]$ and $\sigma > 0$, let $  \localconsistency(\emprashomonset,x)
    =\abs*{\E_{\rf,\rlocaltest}\bracks*{\rf(\rlocaltest)} 
        - \frac{1}{2}} 
    - \E_{\rf,\rlocaltest}\bracks*{\abs*{\rf(x) - \rf(\rlocaltest)}} $, the local expectation of the ensemble margin of $\emprashomonset$ at $x$, penalized by the expected ensemble variability around $x$ . Then given  $x \in \cX$, the ensemble of  $\emprashomonset$ is said to be $(\beta,\sigma)$-locally consistent at $x$ if
    \begin{align}\label{eq: local beta confidence}
       \localconsistency(\emprashomonset,x) \ge \beta.
    \end{align}
\end{definition}

\begin{definition}[$(\beta-\slackmeasure,\sigma)$-local consistency of a retail set]\label{def: local consistency} Given a Rashomon \smallset{} set $\retailset = \{f_k :  k \in [\Nretail]\} \subseteq \emprashomonset$ of size $\Nretail$, a set of perturbed samples $\localtestmat = \{\localtest_{j}: j \in [\nogaussianpts]\} \subseteq \ltwohsphere(x,\sigma)$, the ensemble prediction $\ensmretset(\retailset,\localtestmat)$ in Definition \ref{def: local ensemble}, and a constant $\slackmeasure >0$, let  $ \localmarginconsistency(\retailset,\localtestmat,x) $ be defined as
\begin{align}
    \localmarginconsistency(\retailset,\localtestmat, x) = \babs{\ensmretset(\retailset,\localtestmat)-\frac{1}{2}} - \frac{1}{\nogaussianpts}\sum_{j=1}^{\nogaussianpts}\frac{1}{\Nretail}\sum_{j=1}^{\Nretail}\babs{f_k(x) - f_k(\localtest_j)}.
\end{align} 
for a fixed test instance $x$. Then, $\retailset$ is said to be $(\beta-\slackmeasure,\sigma)$-locally consistent if,
\begin{align}
    \localmarginconsistency(\retailset,\localtestmat,x) \ge \beta-\slackmeasure
\end{align}
\end{definition}

Similar to the conditions needed to establish the theoretical guarantees for our proposed consistency measure, $\consistency(\retailset,\localtestmat,x)$, we require local stability of model predictions in the Rashomon set $\emprashomonset$ along with Lipshchitz continuity as in Definition \ref{Definition: eta Lipschitz}.

\begin{definition}[$\rashlexpctparam$ local stability of  $\emprashomonset$ at $x$]
\label{def: local Rashomon stability}    
Given  a Rashomon set $\emprashomonset$, a test point $x \in \dtest$, a perturbed sample of $x$, $\localtest \in \ltwohsphere(x,\sigma)$, and a small constant $\rashlexpctparam > 0$, $\emprashomonset$ is said to be locally $\rashlexpctparam$ stable if, 
\begin{align}
    \abs{f(\localtest)-\E_{\rf}[\rf(\localtest)]} \le \rashlexpctparam, \:\: \text{for all}\:\: f \in \emprashomonset, \localtest \in \ltwohsphere(x,\sigma).
\end{align}
\end{definition}

We list the assumptions required to establish a probabilistic guarantee on $\localmarginconsistency$ below.
\begin{assumption}[Rashomon local consistency assumptions] \label{assume:rashomon local}
The set $\rretailset  = \{\rf_{k} : k \in [\Nretail]\}$ is an i.i.d. sample from a distribution $\tau$ on $\emprashomonset$. Every model $f$ in the Rashomon set $\emprashomonset$ is $\eta$-locally Lipschitz as defined in \ref{Definition: eta Lipschitz}. For fixed parameters $\beta, \sigma$ and $\rashlexpctparam$, $x$ belongs to the subset of test instances which are $(\beta,\sigma)$-locally consistent according to Definition \ref{Definition: local beta conf}, and $\rashlexpctparam$ stable according to Definition \ref{def: local Rashomon stability}. We denote this set as $  \dtestlocalconf(\beta,\sigma,\rashlexpctparam) = \{ x \in \dtest : \localconsistency(\emprashomonset, x) \ge \beta \:\:\text{and} \:\:\emprashomonset \:\:\text{is} \:\: \rashlexpctparam \:\:\text{locally stable at} \:\: x\}$.
\end{assumption}

The concentration results that follow are stated only for a fixed $x \in  \dtestlocalconf$. For notational simplicity, we suppress the dependence $\dtestlocalconf$ on $\beta, \sigma, \rashlexpctparam$.

For a fixed $x \in \dtestlocalconf$, Theorem \ref{Theorem: local consistency guarantee} provides the following probabilistic guarantee that the local consistency measure on the ensemble of the \smallset{} set $\localconsistency(\retailset,\localtestmat,x)$ will not deviate too much from the local consistency measure on the empirical Rashomon set $\localconsistency(\emprashomonset,x)$.

\begin{restatable}[Local consistency guarantee]{theorem}{localconsistent}
\label{Theorem: local consistency guarantee}
 Let $\emprashomonset$, $\rretailset$, and $x$ satisfy Assumption \ref{assume:rashomon local} for parameters  $\beta \in [-\frac{1}{2},\frac{1}{2}]$, and $\sigma, \rashlexpctparam > 0$. For a test instance $x \in \dtestlocalconf$, let $\rlocaltestmat =  \{\rlocaltest_{j} : j \in [\nogaussianpts]\}$ be a random set of perturbations of $x$, distributed $\iid \Unif(\ltwohsphere(x,\sigma))$, and  let $\ensmretset(\retailset,\localtestmat)$ be the ensemble prediction as in Definition \eqref{def: local ensemble}. Given the local consistency measure $\localmarginconsistency(\retailset,\localtestmat,x) = \babs{\ensmretset(\retailset,\localtestmat)-\frac{1}{2}} - \frac{1}{\nogaussianpts}\sum_{j=1}^{\nogaussianpts}\frac{1}{\Nretail}\sum_{k=1}^{\Nretail}\babs{f_k(x) - f_k(\localtest_j)} $ from Definition \ref{def: local consistency}, for all $\slackmeasure \ge 0$,
\begin{align}
    \P_{\rretailset,\rlocaltestmat}\bracks*{\localconsistency(\rretailset,\rlocaltestmat,x) \ge \beta - \slackmeasure} 
    \ge 1 - 2\exp\parens*{ \frac{-\Nretail\slackmeasure^2}{32\rashlexpctparam^2} } - 2\exp\parens*{ \frac{-S\slackmeasure^2}{32\eta^2\sigma^{2}} } - 2\exp\parens*{\frac{
        -\nogaussianpts\slackmeasure^{2}}{8\eta^2\sigma^{2}}\bigg) -  2\exp\Big(
        \frac{-\Nretail\slackmeasure^2}{8\eta^2\sigma^{2}}}. 
\end{align}
\end{restatable}
\begin{proof}
    We refer the readers to Section~\ref{sec: Theorem local consistency} for the proof.
\end{proof}
\subsubsection{Probabilistic guarantee on local ensemble consistency}\label{sec: Proof for local consistency}
Our main result in this section is a concentration bound for the average of perturbed predictions from an ensemble model. Before proving the probabilistic guarantee on the local consistency measure $\localconsistency(\retailset,\localtestmat,x)$, we first prove the following lemmas that we will leverage for our proof.
\begin{lemma}\label{lemma: Concentration of local expected ensemble margin}
Let $\emprashomonset$, $\rretailset$, and $x$ satisfy Assumption \ref{assume:rashomon local} for parameter $\rashlexpctparam > 0$. Let $\rlocaltestmat =  \{\rlocaltest_{j} : j \in [\nogaussianpts]\}$ be a random set of perturbations of $x$, distributed $\iid \Unif(\ltwohsphere(x,\sigma))$, and $\ensmretset(\retailset,\localtestmat)$ be the ensemble model defined in \ref{def: local ensemble}. 
Define
    \begin{align}\label{eq: gcond}
    \lcondavg(\retailset,x)
        &= \abs*{\expect_{\rlocaltestmat}*{\ensmretset(\retailset,\rlocaltestmat)} - \frac{1}{2}},
    \\
    \label{eq: expgaplcmargin}
    \explallavg(x) &= \abs*{\expect_{\rretailset,\rlocaltestmat}*{\ensmretset(\rretailset,\rlocaltestmat)}-\frac{1}{2}}
\end{align}
Then
    \begin{align}
    \label{eq:ball: local condbound}
    \prob_{\rretailset}*{
        \abs*{
        \lcondavg(\rretailset,x)
            -  \explallavg(x)
            }
        \ge
        \slackparam
    }
    \le
    2 \exp\parens*{\frac{-\Nretail\slackparam^2}{2\rashlexpctparam^2}}.
    \end{align}
\end{lemma}
\begin{proof}
Fix $\rashlexpctparam > 0$ and $x \in \dtestlocalconf$. Since $\rretailset = \{\rf_{k} : k \in [\Nretail]\}$ is sampled i.i.d. according to $\traindist$, we will use McDiarmid's inequality (Lemma \ref{lemma: mcdiarmid ineq}) on the expected local prediction of the ensemble, $\E_{\rlocaltestmat}[\ensmretset(\retailset,\rlocaltestmat)]$. To do so, we must verify the bounded differences property in Equation \eqref{eq:mcdiarmid:boundeddiff}. Let $\retailset = \{f_{1},\dots,f_{k},\dots, f_{\Nretail}\} \subseteq \emprashomonset$ and $\retailset'= \retailset \setminus \{f_k\} \cup \{f'_k\}$, where $f_k' \in \emprashomonset$. Using the triangle inequality and the fact that $\emprashomonset$ is $\rashlexpctparam$-stable at $x$ (Definition \ref{def: local Rashomon stability})

\begin{align}
    \abs*{
        \E_{\rlocaltestmat}[{\ensmretset(\retailset,\rlocaltestmat)}] 
        - {\E_{\rlocaltestmat}}[\ensmretset(\retailset',\rlocaltestmat) 
        }
    &=\abs*{
        \frac{1}{\Nretail}\sum_{k=1}^{\Nretail} 
            \E_{\rlocaltestmat}\bracks*{
                \frac{1}{\nogaussianpts} \sum_{j=1}^{\nogaussianpts}
                    f_{k}(\rlocaltest_{j})} 
        - \frac{1}{\Nretail}\sum_{k=1}^{\Nretail}   
            \E_{\rlocaltestmat}\bracks*{
                \frac{1}{\nogaussianpts}\sum_{j=1}^{\nogaussianpts}
                    f'_{k}(\rlocaltest_{j})}}
        \label{eq: linearity expec}\\
    &= \frac{1}{\Nretail}\abs*{
            \E_{\rlocaltest}\bracks*{f_{k}(\rlocaltest)} 
            - \E_{\rlocaltest}\bracks*{f'_{k}(\rlocaltest)}
            } 
         \label{eq: mcdiarmid diff1} \\
    &\leq \frac{1}{\Nretail}\abs*{
            \E_{\rlocaltest}\bracks*{f_{k}(\rlocaltest)}
            - \E_{\rf,\rlocaltest}\bracks*{\rf(\rlocaltest)}
            } 
        + \frac{1}{\Nretail} \abs*{
            \E_{\rlocaltest}\bracks*{f'_{k}(\rlocaltest)}
            - \E_{\rf,\rlocaltest}\bracks*{\rf(\rlocaltest)}
            }
        \notag \\
    &\leq \frac{\rashlexpctparam}{\Nretail} 
        + \frac{\rashlexpctparam}{\Nretail} 
        \notag \\
    &= \frac{2\rashlexpctparam}{\Nretail}.
        \label{eq: mcdiarmid diff2}
\end{align}
Now applying Lemma \ref{lemma: mcdiarmid ineq} with $c_k = \frac{2 \rashexpctparam}{\Nretail}$, we have
 \begin{align}
    \prob_{\rretailset}*{
        \abs*{
    \expect_{\rlocaltestmat}{\ensmretset(\rretailset,\rlocaltestmat)}
            -  \expect_{\rretailset,\rlocaltestmat}*{\ensmretset(\rretailset,\rlocaltestmat)}
            }
        \ge
        \slackparam
    }
    \le
    2 \exp\parens*{\frac{-\Nretail\slackparam^2}{2\rashlexpctparam^2}}.
    \end{align}
    Let $\lcondavg$, and $\explallavg$ be defined as in Equation \eqref{eq: gcond},and \eqref{eq: expgaplcmargin}. From the reverse triangle inequality,
    \begin{align}
        \abs*{\lcondavg(\retailset,x) -  \explallavg(x)}  \le \abs*{
    \expect_{\rlocaltestmat}{\ensmretset(\retailset,\rlocaltestmat)}
            -  \expect_{\rretailset,\rlocaltestmat}*{\ensmretset(\rretailset,\rlocaltestmat)}
            }.
    \end{align}
    Therefore, the event $\{\abs*{\lcondavg(\retailset,x) -  \explallavg(x)} \ge \slackparam\}$ implies $\{ \abs*{
    \expect_{\rlocaltestmat}{\ensmretset(\retailset,\rlocaltestmat)}
            -  \expect_{\rretailset,\rlocaltestmat}*{\ensmretset(\rretailset,\rlocaltestmat)}
            }
        \ge
        \slackparam\}$, so
    \begin{align}
    \prob_{\rretailset}*{
        \abs*{
        \lcondavg(\rretailset,x)
            -  \explallavg(x)
            }
        \ge
        \slackparam
    }
    \le
    2 \exp\parens*{\frac{-\Nretail\slackparam^2}{2\rashlexpctparam^2}}.
    \end{align}
\end{proof}

\begin{lemma}\label{Lemma: Subgaussian local conc random}
Let $\emprashomonset$, $\rretailset$, and $x$ satisfy Assumption \ref{assume:rashomon} for parameters $\sigma, \rashlexpctparam > 0$. Let $\rlocaltestmat =  \{\rlocaltest_{j}: j \in [\nogaussianpts]\}$ be a random set of perturbations of $x$, distributed $\iid \Unif(\ltwohsphere(x,\sigma))$, and $\ensmretset(\retailset,\localtestmat)$ be the ensemble model defined in \ref{def: local ensemble}.
Define
\begin{align}\label{eq: gaplmargin}
   \lallavg(x) = \abs*{\ensmretset(\rretailset,\rlocaltestmat) - \frac{1}{2}}.
\end{align}
Then for any $\slackgap > 0$, 
    \begin{align}
    \prob_{\rretailset,\rlocaltestmat}*{
        \abs*{  \lallavg(x) - \explallavg(x)}
    \ge \slackgap }
    \le 2\exp\parens*{ \frac{-S\slackgap^2}{8\eta^2\sigma^{2}} } +
        2\exp\parens*{ \frac{-\Nretail\slackgap^2}{8\rashlexpctparam^2} }.
    \end{align}
\end{lemma}

\begin{proof}
Our goal is to show concentration for the local margin term $ \lallavg(x)$ to $\explallavg(x)$. The variability comes from two distinct sources of randomness,
\begin{itemize}
    \item The randomness in the perturbation samples, $\rlocaltestmat = \{\rlocaltest_{j}: j \in [\nogaussianpts]\}$ when $\retailset$ is fixed.
    \item The randomness in the retail set $\rretailset = \{\rf_{k}: k \in [\Nretail]\}$.
\end{itemize}
Let $ \gaplmargin(\rretailset,\rlocaltestmat,x) =  \abs*{\lallavg(x) - \explallavg(x)}$. We can thus decompose the gap as follows
\begin{align}
    {\gaplmargin(\rretailset,\rlocaltestmat,x)} \le \underbrace{\abs*{ \lallavg(x) - \lcondavg(\rretailset,x)}}_{\gaplcondr(\rretailset,\rlocaltestmat,x)} + \underbrace{\abs*{ \lcondavg(\rretailset,x) - \explallavg(x)}}_{\gaplrandr(\rretailset,x)}.
\end{align}
Lemma \ref{lemma: Concentration of local expected ensemble margin} shows that for any $\slackparam$,
\begin{align}
    \prob_{\rretailset}*{
        {\gaplrandr(\rretailset,x)}
        \ge
        \slackparam
    }
    \le
    2 \exp\parens*{\frac{-\Nretail\slackparam^2}{2\rashlexpctparam^2}}.
    \end{align}
We now argue that $\gaplcondr(\rretailset,\rlocaltestmat,x) = \abs*{ \gaplmargin(\rretailset,\rlocaltestmat,x) - \lcondavg(\rretailset,x)}$ is also bounded with high probability.
 Let $R = \{f_k : k \in [\Nretail]\}$ be a realization of $\rretailset$ and let $\ensmretset(\retailset,\rlocaltestmat,x)$ be deterministic in $\retailset$. Since $\rlocaltestmat = \{\rlocaltest_{j}: j \in [\nogaussianpts]\}$ is sampled $\iid \Unif(\ltwohsphere(x,\sigma))$ we will use McDiarmid's inequality (Lemma \ref{lemma: mcdiarmid ineq}) on the ensemble prediction $\ensmretset(\retailset,\rlocaltestmat)$. To do so we must verify the bounded differences property in \eqref{eq:mcdiarmid:boundeddiff}. Let $\localtestmat = \{\localtest_{1},\dots,\localtest_{j},\dots, \localtest_{\nogaussianpts}\} \subseteq \ltwohsphere(x,\sigma)$ and $\localtestmat'= \localtestmat \setminus \{\localtest_j\} \cup \{
 \localtest'_j\}$, where $\localtest_j' \in  \ltwohsphere(x,\sigma)$. Using triangle inequality and the fact every element of $\emprashomonset$ is locally Lipschitz with constant $\eta$ in the neighborhood $\ltwohsphere(x,\sigma)$, for every perturbation $\localtest_{j}$, 
\begin{align}
    \abs*{ \ensmretset(\retailset,\localtestmat) - \ensmretset(\retailset,\localtestmat') }
    &= \abs*{\frac{1}{\nogaussianpts}\sum_{j=1}^{\nogaussianpts}\ensmretset(\retailset,\localtest_j) 
        -\frac{1}{\nogaussianpts}\sum_{j=1}^{\nogaussianpts}\ensmretset(\retailset,\localtest'_j) } 
        \notag \\
    &= \frac{1}{\nogaussianpts}\abs*{\ensmretset{(\retailset,\localtest_j)} - \ensmretset{(\retailset,\localtest'_j)}} 
        \notag \\
    &\le \frac{1}{\nogaussianpts}\frac{1}{\Nretail}\sum_{k=1}^{\Nretail}\abs*{f_k{(\localtest)} - f_k{(x)}} 
        +\frac{1}{\nogaussianpts}\frac{1}{\Nretail}\sum_{k=1}^{\Nretail}\abs*{f_k{(x)} - f_k{(\localtest')}}
        \notag \\
    &\le \frac{2 \eta\sigma}{\nogaussianpts}.
\end{align}
Now applying Lemma \ref{lemma: mcdiarmid ineq} with $c_k = \frac{2 \eta\sigma}{\nogaussianpts}$, we have
    \begin{align}
      \prob_{\rlocaltestmat}*{
            \abs*{ \ensmretset(\retailset,\rlocaltestmat)
                 - \expect_{\rlocaltestmat}*{\ensmretset(\retailset,\rlocaltestmat)} }
          \ge \slackparam }
      \le 
      2\exp\parens*{ \frac{-\nogaussianpts\slackparam^2}{2\eta^2\sigma^2} }.
    \end{align}
From the triangle inequality we have,
\begin{align}
    \gaplcondr(\retailset,\rlocaltestmat) \le \abs*{ \ensmretset(\retailset,\rlocaltestmat)
                 - \expect_{\rlocaltestmat}*{\ensmretset(\retailset,\rlocaltestmat,x)} }.
\end{align}
Therefore the event  $\{\abs*{ \ensmretset(\retailset,\rlocaltestmat,x)
                 - \expect_{\rlocaltestmat}*{\ensmretset(\retailset,\rlocaltestmat,x)} } \ge \slackparam\}$ implies $\{\gaplcondr(\retailset,\rlocaltestmat,x)\ge \slackparam\}$, so
    \begin{align}\label{eq: local margin cond bound}
    \prob_{\rlocaltestmat}*{\abs*{\gaplcondr(\retailset,\rlocaltestmat,x)} \ge \slackparam} 
    \le   2\exp\parens*{ \frac{-\nogaussianpts\slackparam^2}{2\eta^2\sigma^2} }.
    \end{align}
    Since Equation \eqref{eq: local margin cond bound} holds uniformly for all $\rretailset = \retailset$, applying the law of total probability yeilds, 
\begin{align}
    \prob_{\rretailset,\rlocaltestmat}*{\abs*{\gaplcondr(\rretailset,\rlocaltestmat, x)} \ge \slackparam}
    = \sum_{\retailset}\prob_{\rlocaltestmat \lvert \rretailset=\retailset}*{\gaplcondr(\retailset,\rlocaltestmat,x) \ge {\slackparam} \lvert \rretailset=\retailset } \prob*{\rretailset=\retailset}
      \le     2\exp\parens*{ \frac{-\nogaussianpts\slackparam^2}{2\eta^2\sigma^2} }, \label{eq:local margin uniform bound}
\end{align}
Then setting $\slackparam = \slackgap/2$ we have 
    \begin{align}
    \prob_{\rretailset,\rlocaltestmat}*{ 
        \gaplmargin(\rretailset,\rlocaltestmat,x) \ge \slackgap }
    &\le 
    \prob_{\rretailset,\rlocaltestmat}*{
        \gaplcondr(\rretailset,\rlocaltestmat,x) \ge \slackgap/2, 
        \gaplrandr(\rretailset,\rlocaltestmat,x) \ge \slackgap/2 } 
     \\ \notag
    &\le 
    \prob_{\rretailset,\rlocaltestmat}*{
        \gaplcondr(\rretailset,\rlocaltestmat,x) \ge \slackgap/2 }
    + \prob_{\rretailset,\rlocaltestmat}*{ 
        \gaplrandr(\rretailset,x) \ge \slackgap/2 }
    \\ \notag
    &\le 2\exp\parens*{ \frac{-S\slackgap^2}{8\eta^2\sigma^{2}} } +
        2\exp\parens*{ \frac{-\Nretail\slackgap^2}{8\rashlexpctparam^2} }.
    \end{align}
 Here we used the fact that for $\gaplmargin(\rretailset,\rlocaltestmat,x) \ge \slackgap$ to hold, at least one of the terms $\gaplcondr(\rretailset,\rlocaltestmat, x)$ and $\gaplrandr(\rretailset, x)$ should be larger than $\slackgap/2$.
 \end{proof}

\subsubsection{Proof of Theorem \ref{Theorem: local consistency guarantee}}\label{sec: Theorem local consistency}

Theorem~\ref{Theorem: local consistency guarantee} bounds the probability that the local consistency score computed using a finite retail ensemble falls significantly below the corresponding consistency score of the full Rashomon set.  Let $\gapl(\retailset,\localtestmat, x) = \abs*{\localmarginconsistency(\retailset,\localtestmat,x) - \localmarginconsistency(\emprashomonset,x)}$, where $\localmarginconsistency(\retailset,\localtestmat,x) =  \abs*{\ensmretset(\retailset,\localtestmat, x)-\frac{1}{2}} - \frac{1}{\nogaussianpts}\sum_{j=1}^{\nogaussianpts}\frac{1}{\Nretail}\sum_{k=1}^{\Nretail}\abs*{f_{k}(x) - f_{k}(\rlocaltest_{j})}$ from Definition~\ref{def: local consistency} and $\localmarginconsistency(\emprashomonset,x) = \abs*{\E_{\rretailset,\rlocaltestmat}\bracks*{\ensmretset(\rretailset,\rlocaltestmat)} - \frac{1}{2}} - \E_{\rretailset,\rlocaltestmat}\bracks*{\frac{1}{\nogaussianpts}\sum_{j=1}^{\nogaussianpts}\frac{1}{\Nretail}\sum_{k=1}^{\Nretail}\abs*{\rf_{k}(x) - \rf_{k}(\localtest_{j})}}$.  Since $x \in \dtestlocalconf$, by Definition \ref{Definition: local beta conf} of $(\beta,\sigma)$-local consistency of $x$, $\localmarginconsistency(\emprashomonset,x) \ge \beta$. Therefore, if we can bound the deviation ${\gapl(\retailset, \localtestmat,x)}$ as ${\gapl(\retailset, \localtestmat, x)} \le \slackmeasure$, we can bound $\localmarginconsistency(\retailset,\localtestmat,x)$ as,
    \begin{align}\label{eq: local consbeta}
        \localmarginconsistency(\retailset,\localtestmat,x) \ge \localmarginconsistency(\emprashomonset,x) - \slackmeasure \ge \beta - \slackmeasure.
    \end{align}
Thus to show $\consistency(\retailset,\localtestmat,x) \ge \beta - \slackmeasure$, it suffices to upper bound the term ${\gap(\retailset,\localtestmat)}$ which we prove next.

\localconsistent*
\begin{proof}
 Given ${\gapl(\rretailset,\rlocaltestmat,x)} = \abs*{\localmarginconsistency(\rretailset,\rlocaltestmat,x) - \localmarginconsistency(\emprashomonset,x)}$, we have
 \begin{align}\label{eq: glmeasure}
     &{\gapl(\retailset,\localtestmat)} = \notag\\
     &\abs*{\abs*{\ensmretset(\retailset,\localtestmat)-\frac{1}{2}} - \frac{1}{\nogaussianpts}\sum_{j=1}^{\nogaussianpts}\frac{1}{\Nretail}\sum_{k=1}^{\Nretail}\abs*{f_{k}(x) - f_{k}(\rlocaltest_{j})} - \abs*{\E_{\rretailset,\rlocaltestmat}\bracks*{\ensmretset(\rretailset,\rlocaltestmat)} - \frac{1}{2}} + \E_{\rretailset,\rlocaltestmat}\bracks*{\frac{1}{\nogaussianpts}\sum_{j=1}^{\nogaussianpts}\frac{1}{\Nretail}\sum_{k=1}^{\Nretail}\abs*{\rf_{k}(x) - \rf_{k}(\rlocaltest_{j})}}}
 \end{align}

After rearranging the terms, as in Equation \eqref{eq: glmeasure}, we represent them with two functions for notation convenience. The function
 \begin{align}
     \gaplmargin(\retailset,\localtestmat,x) = \abs*{\abs*{\ensmretset(\retailset,\localtestmat)-\frac{1}{2}}  - \abs*{\E_{\rretailset,\rlocaltestmat}\bracks*{\ensmretset(\rretailset,\rlocaltestmat)} - \frac{1}{2}}},
 \end{align} 
 represents the difference between the average margin over perturbed samples of a single ensemble and the expected margin of the expected ensemble of models. The remaining terms
 \begin{align}
     \gaplvar(\retailset,\localtestmat, x) =  \abs*{\frac{1}{\nogaussianpts}\sum_{j=1}^{\nogaussianpts}\frac{1}{\Nretail}\sum_{k=1}^{\Nretail}\abs*{f_{k}(x) - f_{k}(\localtest_{j})} - \E_{\rretailset,\rlocaltestmat}\bracks*{\frac{1}{\nogaussianpts}\sum_{j=1}^{\nogaussianpts}\frac{1}{\Nretail}\sum_{k=1}^{\Nretail}\abs*{\rf_{k}(x) - \rf_{k}(\rlocaltest_{j})}}},
 \end{align}
 capture how predictions vary when we perturb the input $x$. 
  From the triangle inequality,
  \begin{align}
     \gapl(\retailset,\localtestmat,x) \le  \gaplmargin(\retailset,\localtestmat,x) +  \gaplvar(\retailset,\localtestmat,x).
  \end{align}

From Lemma \ref{Lemma: Subgaussian local conc random}, $ \P_{\rretailset,\rlocaltestmat}\bracks*{{\gaplmargin(\rretailset,\rlocaltestmat,x)} \ge {\slackgap}} \le 2\exp\parens*{ \frac{-S\slackgap^2}{8\eta^2\sigma^{2}} } +
        2\exp\parens*{ \frac{-\Nretail\slackgap^2}{8\rashlexpctparam^2} }$, and from Lemma \ref{Lemma: Subgaussian conc random}, $\P_{\rretailset,\rlocaltestmat}\bracks*{{\gaplvar(\rretailset,\rlocaltestmat, x)} \ge {\slackgap}} \le 2\exp\parens*{\frac{
        -\nogaussianpts\slackgap^{2}}{2\eta^2\sigma^{2}}\bigg) +  2\exp\Big(
        \frac{-\Nretail\slackparam^2}{2\eta^2\sigma^{2}} }$.
For $\gapl(\rretailset,\rlocaltestmat,x)$ to be greater than $\slackmeasure$, at least one of the terms $\gaplmargin(\rretailset,\rlocaltestmat,x)$ and $\gaplvar(\rretailset,\rlocaltestmat,x)$ should be greater than $\frac
{\slackmeasure}{2}$. Fixing $\slackgap = \frac{\slackmeasure}{2}$, we get
\begin{align}
    \P_{\rretailset,\rlocaltestmat}\bracks*{{\gapl(\rretailset,\rlocaltestmat,x)} \ge \slackmeasure} \le \P_{\rretailset}\bracks*{{\gaplmargin(\rretailset, \rlocaltestmat, x)} \ge \frac{\slackmeasure}{2}} + \P_{\rretailset,\rlocaltestmat}\bracks*{{\gaplvar(\rretailset,\rlocaltestmat,x)} \ge \frac{\slackmeasure}{2}}.
\end{align}
Thus,
\begin{align}
    \P_{\rretailset,\rlocaltestmat}\bracks*{{\gapl(\rretailset,\rlocaltestmat,x)} \ge \slackmeasure} \le 
        2\exp\parens*{ \frac{-\Nretail\slackmeasure^2}{32\rashlexpctparam^2} } + 2\exp\parens*{ \frac{-S\slackmeasure^2}{32\eta^2\sigma^{2}} } + 2\exp\parens*{\frac{
        -\nogaussianpts\slackmeasure^{2}}{8\eta^2\sigma^{2}}\bigg) +  2\exp\Big(
        \frac{-\Nretail\slackmeasure^2}{8\eta^2\sigma^{2}}}.
\end{align}
It follows that
\begin{align}
      \P_{\rretailset,\rlocaltestmat}\bracks*{{\gapl(\rretailset,\rlocaltestmat,x)} \le \slackmeasure} \ge 1 - 2\exp\parens*{ \frac{-\Nretail\slackmeasure^2}{32\rashlexpctparam^2} } - 2\exp\parens*{ \frac{-S\slackmeasure^2}{32\eta^2\sigma^{2}} } - 2\exp\parens*{\frac{
        -\nogaussianpts\slackmeasure^{2}}{8\eta^2\sigma^{2}}\bigg) -  2\exp\Big(
        \frac{-\Nretail\slackmeasure^2}{8\eta^2\sigma^{2}}}.
\end{align}
When combined with Equation $\eqref{eq: local consbeta}$ yields,
\begin{align}
    \P_{\rretailset,\rlocaltestmat}\bracks*{\localmarginconsistency(\rretailset,\rlocaltestmat, x) \ge \beta - \slackmeasure} \ge 1 - 2\exp\parens*{ \frac{-\Nretail\slackmeasure^2}{32\rashlexpctparam^2} } - 2\exp\parens*{ \frac{-S\slackmeasure^2}{32\eta^2\sigma^{2}} } - 2\exp\parens*{\frac{
        -\nogaussianpts\slackmeasure^{2}}{8\eta^2\sigma^{2}}\bigg) -  2\exp\Big(
        \frac{-\Nretail\slackmeasure^2}{8\eta^2\sigma^{2}}}.
\end{align}

\end{proof}

\subsection{Multiplicity metrics}\label{sec: Multiplicity metrics}
As mentioned in the Introduction, predictive multiplicity means that models may disagree on an individual test instance $x$. When these models are in the Rashomon set, this disagreement may indicate that $x$ is fundamentally ``difficult'' because even well-performing models disagree on it. Several measures have been proposed to quantify multiplicity.

\begin{definition}[Pairwise Disagreement~\citep{black2022model}]\label{def: pairwise disagreement}
The pairwise disagreement of the empirical Rashomon set $\emprashomonset$, at a test instance $x$, is the average disagreement between pairs of models $f_{k},f_{l} \in \emprashomonset$ and is defined as,
\begin{align}
      \pairwisedisag(x) = \frac{1}{\abs{\emprashomonset}(\abs{\emprashomonset}-1)}\sum_{h_k,h_l \in \emprashomonset, h_k \ne _l}\Indic{h_{k}(x) \ne h_{l}(x)},
\end{align}  
where $h(x)= \Indic{f(x) \geq \frac{1}{2}}$ is the hard decision of $f$.
\end{definition}
\begin{definition}[Discrepancy~\citep{marx2020predictive}]\label{def: Discrepancy}The discrepancy of $\emprashomonset$, at a test instance x, is the fraction of conflicting predictions of models in $\emprashomonset$ w.r.t. a reference model and is defined as,
\begin{align}
     \disagens(x) =  \frac{1}{\abs{\emprashomonset}}\sum_{h_k \in  \emprashomonset}\Indic{h_k(x) \ne h_0(x)},
\end{align}
where we choose the reference model as $h_{0}(x) = \Indic{\frac{1}{\abs{\emprashomonset}}\sum_{f_k \in  \emprashomonset}f_k(x) \ge \frac{1}{2}}$.
    
\end{definition}
We also look at multiplicity in the neighborhood of $x$ using a new metric called the local discrepancy.
\begin{definition}[Local Discrepancy]\label{def: Local Discrepancy} The local discrepancy of $\emprashomonset$, at a test instance x, is the average fraction of conflicting predictions between models in $\emprashomonset$ over samples $\localtest_1,\dots.,\localtest_{\nogaussianpts} \in \ltwohsphere(x,\sigma)$,and the prediction of a reference model at x, and is defined as,
\begin{align}
     \localdisagens(x) = \frac{1}{\abs{\emprashomonset}}\sum_{h_k \in  \emprashomonset}\frac{1}{{\nogaussianpts}}\sum_{j=1}^{\nogaussianpts}\Indic{h_k(\localtest_j) \ne h_{0}(x)},
\end{align}
where we choose the reference model as $h_{0}(x) = \Indic{\frac{1}{\abs{\emprashomonset}}\sum_{f_k \in  \emprashomonset}f_k(x) \ge \frac{1}{2}}$
\end{definition}

We also consider two multiplicity metrics: prediction variance and prediction range, which focus on predicted probabilities rather than the `hard' predicted label. 

 \begin{definition}[Prediction Variance~\citep{hamman2024quantifying,watson2023predictive}]\label{def: Prediction Variance}
    The prediction variance of $\emprashomonset$ is the variance of model outputs for a given test instance $x$, and is defined as
    \begin{align}
        \predvar(x) = \frac{1}{\abs{\emprashomonset}}\sum_{f_k \in  \emprashomonset}(f_{k}(x) - \frac{1}{\abs{\emprashomonset}}\sum_{f_k \in  \emprashomonset}f_{k}(x))^{2}.
    \end{align}
\end{definition}

\begin{definition}[Prediction Range \citep{hamman2024quantifying,watson2023predictive}]\label{def: Prediction Range} The prediction range of $\emprashomonset$ is the range of predicted probabilities of models for a given test instance $x$, and is defined as,
\begin{align}
    \predrange(x) =  \max_{f_k \in  \emprashomonset}f_{k}(x) -  \min_{f_k \in  \emprashomonset}f_{k}(x).
\end{align}
\end{definition}

\subsection{Related measures}\label{sec: Related measures}
 
We define a few quantities that are related to our proposed consistency measure $\consistency(\retailset,x)$ in Definition \ref{def: Consistency measure}, and local consistency $\localmarginconsistency(\retailset,\localtestmat, x)$ in Definition \ref{def: local consistency}.

\begin{definition}[Margin]\label{def: margin} Given a Rashomon \smallset{} set $\retailset = \{f_k :  k \in [\Nretail]\}$, the ensemble model $\ensmretset(\retailset,x)$ of size $\Nretail$, we define the margin of  $\ensmretset(\retailset,x)$ at a fixed instance $x$ as, 
\begin{align}
    \margin(\retailset,x) = \babs{\ensmretset(\retailset,x) &- \frac{1}{2}}
\end{align}
    
\end{definition}

\begin{definition}[Local margin]\label{def: Local margin} Given a Rashomon \smallset{} set $\retailset = \{f_k: k \in [\Nretail]\}$  of size $\Nretail$, a set $\localtestmat = \{\localtest_{j} : j \in [\nogaussianpts]\} \subseteq \ltwohsphere(x,\sigma)$ of $\nogaussianpts$ perturbed samples of $x$, and the local ensemble $\bar{f}(\retailset, \localtestmat)$, we define the local margin of $\bar{f}(\retailset, \localtestmat)$ at a fixed instance $x$ as, 
\begin{align}
    \localmargin(\retailset,\localtestmat,x) = \babs{\ensmretset(\retailset,\localtestmat) &- \frac{1}{2}}
\end{align}    
\end{definition}


\begin{definition}[Local variability]\label{def: Local variability} Given a Rashomon \smallset{} set $\retailset = \{f_{k}: k \in [\Nretail]\}$ of size $\Nretail$, and a set $\localtestmat = \{\localtest_{j} : j \in [\nogaussianpts]\} \subseteq \ltwohsphere(x,\sigma)$ of $\nogaussianpts$ perturbed samples of $x$, we define the local variability at a fixed instance $x$ as, 
\begin{align}
    \localvar(\retailset,x) =  \frac{1}{\nogaussianpts}\sum_{j=1}^{\nogaussianpts}\frac{1}{\Nretail}\sum_{j=1}^{\Nretail}\babs{f_k(x) - f_k(\localtest_j)} 
\end{align}
\end{definition}

\begin{definition}[HDMLD Local Stability \cite{hamman2024quantifying}]\label{def: Hamman stability}
 Given a model from the Rashomon set $f_{k} \in \emprashomonset$, and a set $\localtestmat = \{\localtest_{j} : j \in [\nogaussianpts]\} \subseteq \ltwohsphere(x,\sigma)$ of $\nogaussianpts$ perturbed samples of $x$, we define the HDMLD local stability of the model at a fixed instance $x$ as
\begin{align}
    \hammancons(f_k,x)
    =
    \frac{1}{\nogaussianpts}
    \sum_{j=1}^{\nogaussianpts}
    \left(
        f_k(\localtest_j)
        -
        \babs{
            f_k(x)-f_k(\localtest_j)
        }
    \right).
\end{align}    
\end{definition}

\subsection{Additional Experiments}\label{sec: Additional Experiments}

\subsubsection{Additional plots showing high consistency thresholds substantially reduce audit risk for Rashomon \largeset{}}\label{sec: Additional beta tradeoff}

Section~\ref{sec: Thresholdvsauditoutcomes} presented the auditing results for the \largeset{} ensemble as a function of the threshold $\beta$ for paraphrase detection on the MRPC dataset using ensembles of independently fine-tuned BERT models. Figures~\ref{fig:costrisk} and Figure~\ref{fig:divrisk} report the audit cost ratio, audit risk ratio, diversion fraction, and audit risk fraction curves from Section~\ref{sec: Thresholdvsauditoutcomes} for the remaining datasets and model architectures. Across all settings, we observe the same qualitative behavior as in Section \ref{sec: Thresholdvsauditoutcomes}. Increasing the audit threshold $\beta$ makes the auditing system more conservative, substantially reducing audit risk before causing a sharp increase in audit cost and diversions. The vertical dashed line in each plot marks the value of threshold $\beta$ for different datasets chosen according to the discussion in Section \ref{sec: Thresholdvsauditoutcomes}.
\begin{figure*}[t]
\centering
\begin{subfigure}{\textwidth}
\centering
\begin{minipage}{0.24\textwidth}
    \centering
    \includegraphics[width=\linewidth]{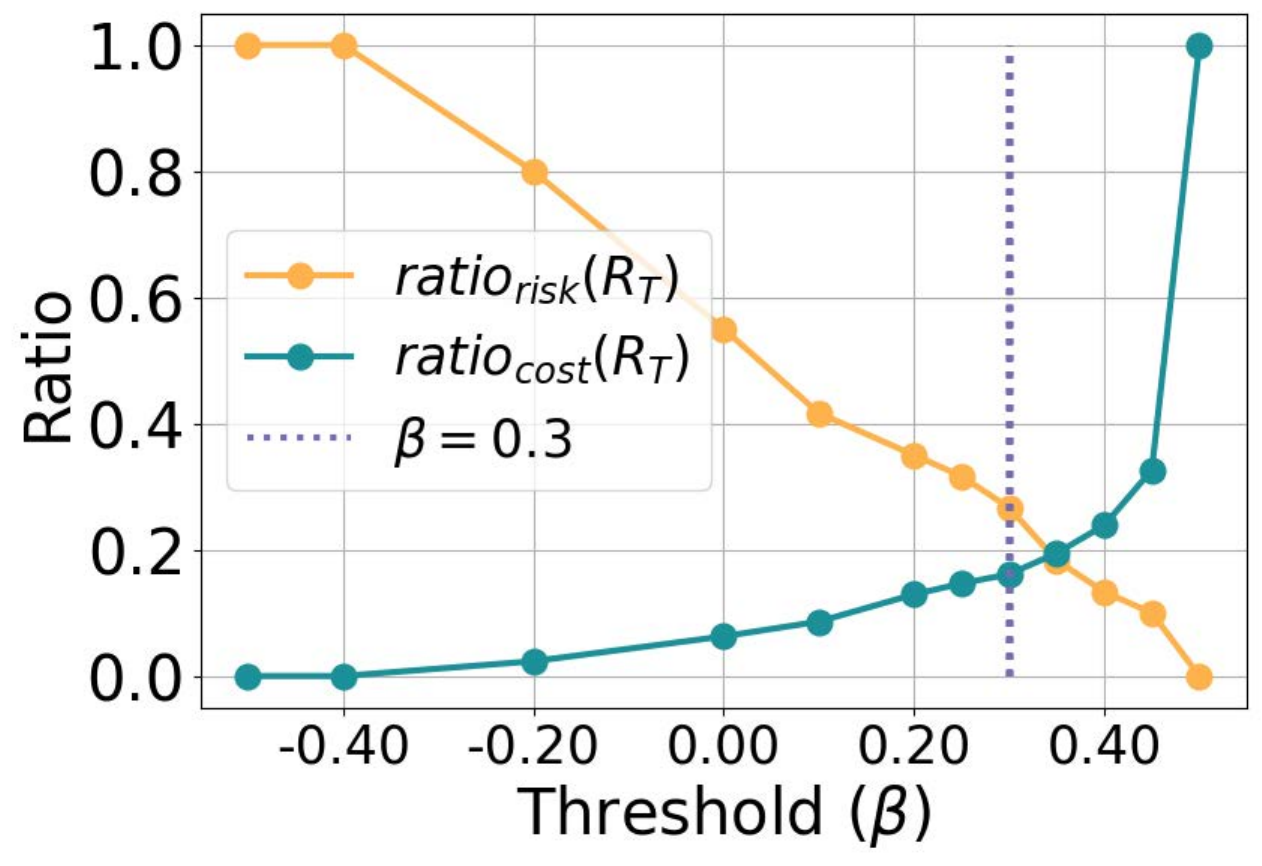}
\end{minipage}
\hfill
\begin{minipage}{0.24\textwidth}
    \centering
    \includegraphics[width=\linewidth]{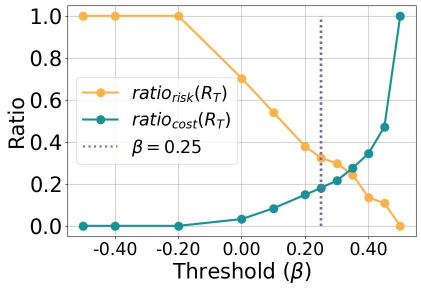}
\end{minipage}
\begin{minipage}{0.24\textwidth}
    \centering
    \includegraphics[width=\linewidth]{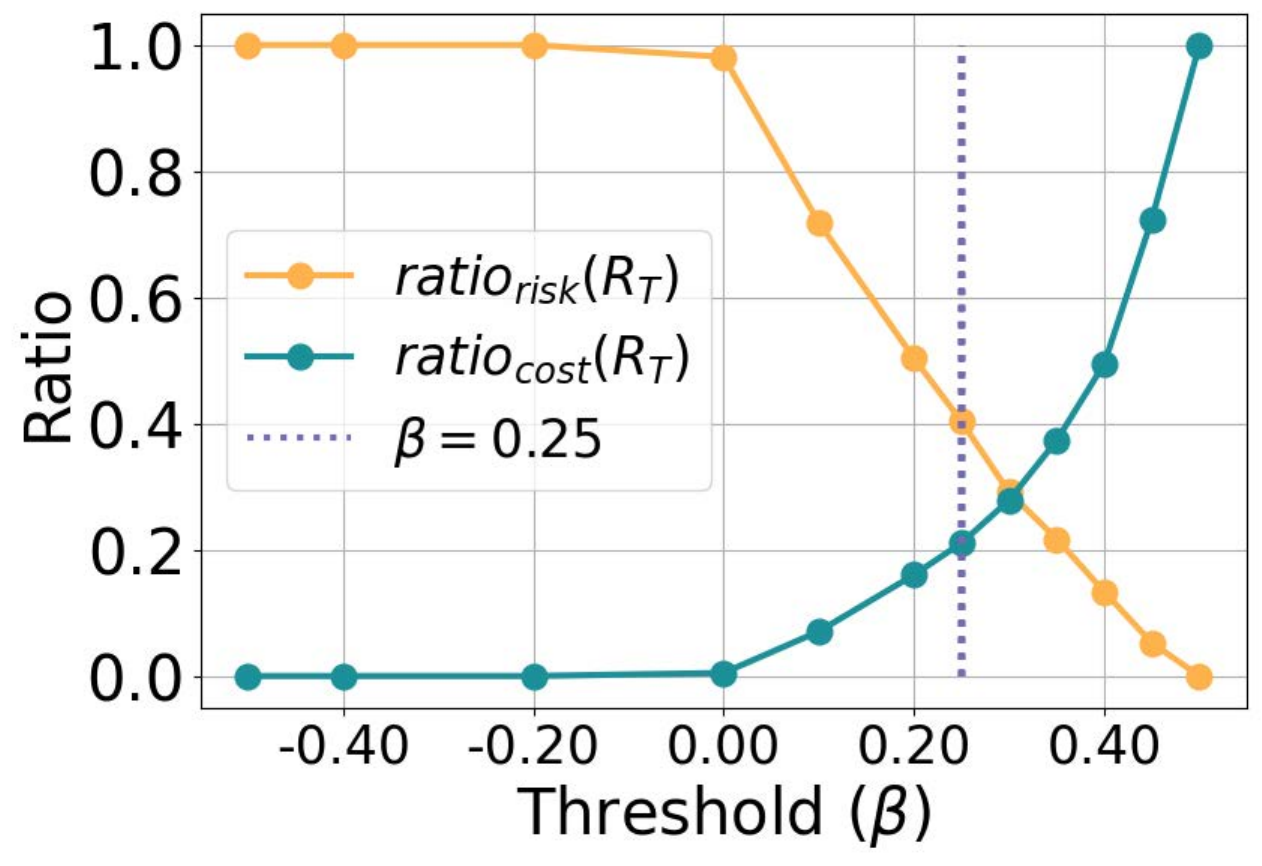}
\end{minipage}
\begin{minipage}{0.24\textwidth}
    \centering
    \includegraphics[width=\linewidth]{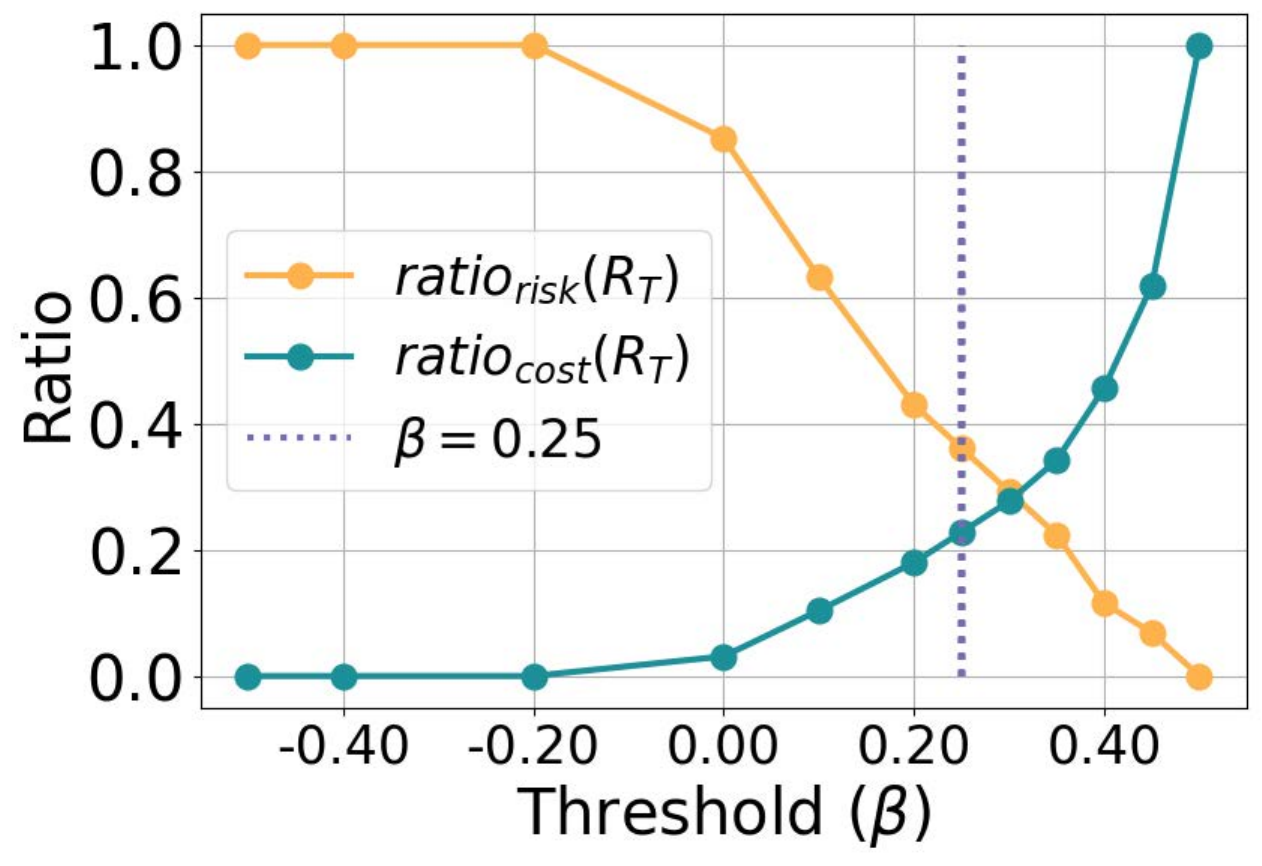}
\end{minipage}
\caption{Audit cost ratio ($\auditcostratio(\retailset_{\Nlib},x)$) versus audit risk ratio
($\auditriskratio(\retailset_{\Nlib},x))$. Results are shown (Left to Right) for: \textbf{BERT} on \textbf{SST2}, \textbf{RoBERTa} on \textbf{MRPC}, \textbf{BIG-SCIENCE T0} on \textbf{Bank}, \textbf{BIG-SCIENCE T0} on \textbf{Adult}. Increasing $\beta$ decreases audit risk at the expense of higher audit cost.}
\label{fig:costrisk}
\end{subfigure}

\vspace{1em}
\begin{subfigure}{\textwidth}
\centering
\begin{minipage}{0.24\textwidth}
    \centering
    \includegraphics[width=\linewidth]{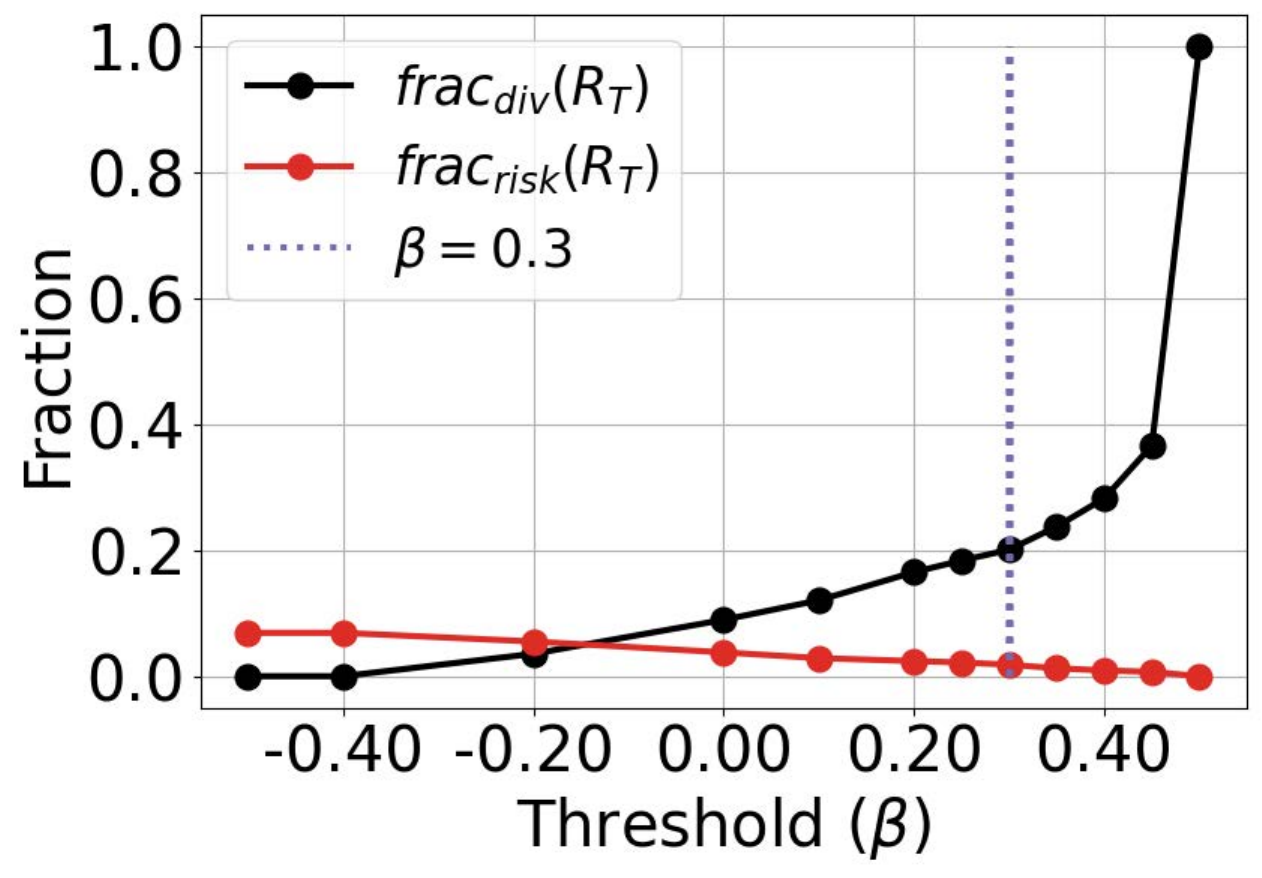}
\end{minipage}
\hfill
\begin{minipage}{0.24\textwidth}
    \centering
    \includegraphics[width=\linewidth]{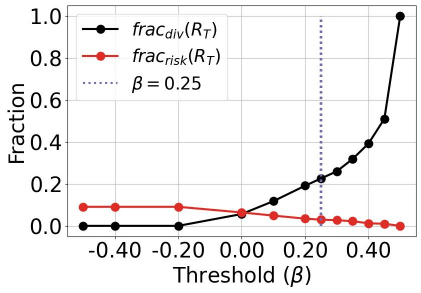}
\end{minipage}
\begin{minipage}{0.24\textwidth}
    \centering
    \includegraphics[width=\linewidth]{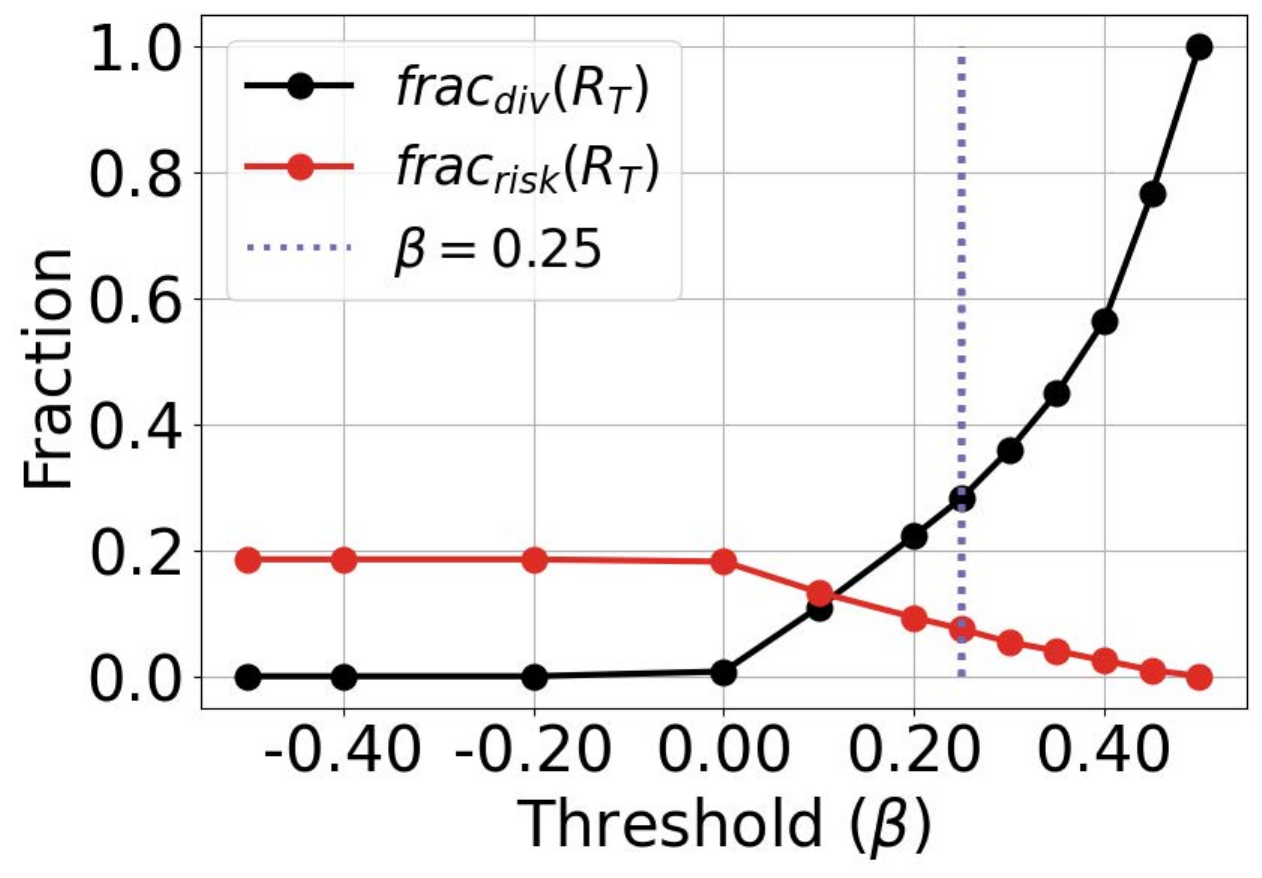}
\end{minipage}
\begin{minipage}{0.24\textwidth}
    \centering
    \includegraphics[width=\linewidth]{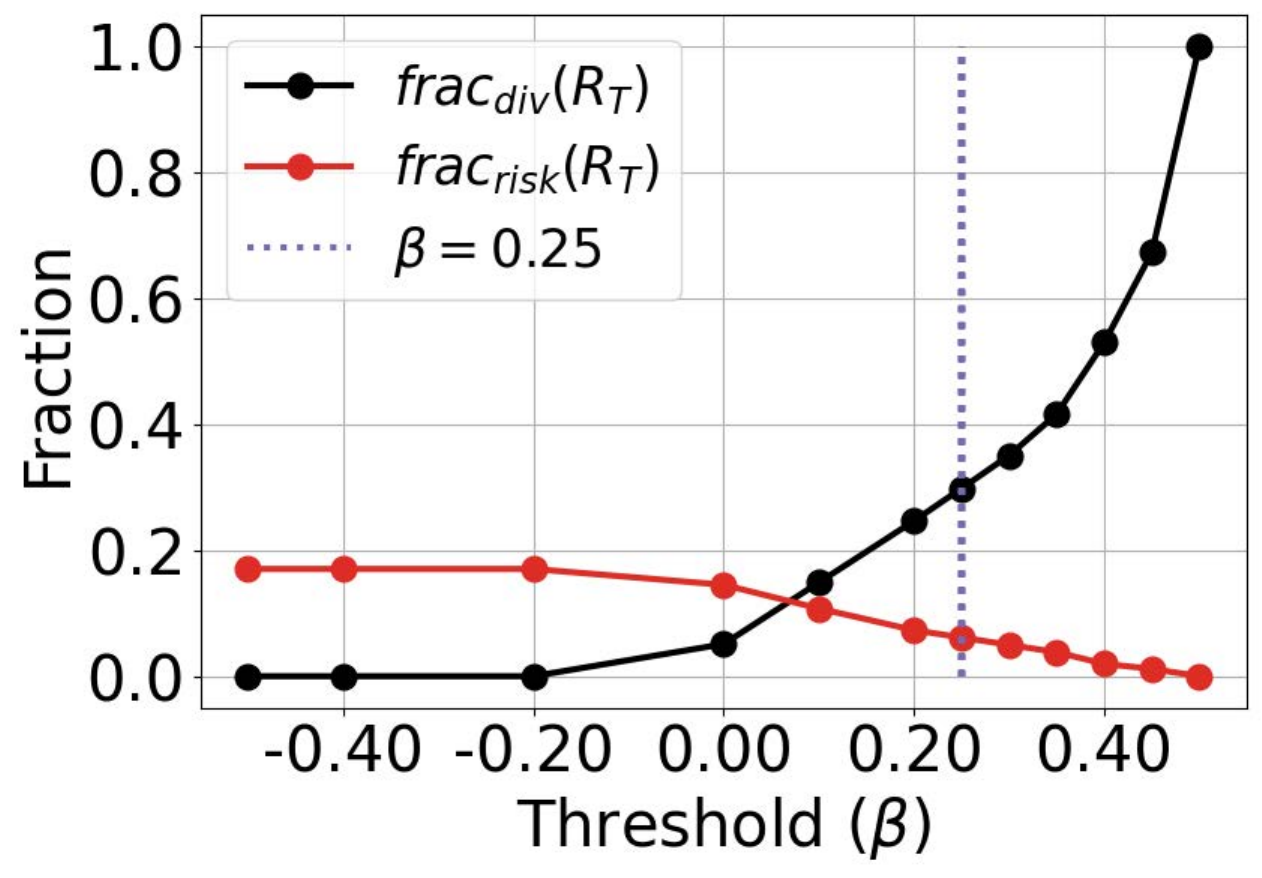}
\end{minipage}
\caption{Diversion fraction ($\auditdivfrac(\retailset_{\Nlib},x)$) versus audit risk fraction
($\auditriskfrac(\retailset_{\Nlib},x))$.
Results are shown (Left to Right) for: \textbf{BERT} on \textbf{SST2}, \textbf{RoBERTa} on \textbf{MRPC}, \textbf{BIG-SCIENCE T0} on \textbf{Bank}, \textbf{BIG-SCIENCE T0} on \textbf{Adult}.
Increasing $\beta$ increases the diversion fraction while reducing overall audit risk.}
\label{fig:divrisk}
\end{subfigure}

\caption{Comparing the audit performance metrics in Definition~\ref{def:costsandrisks}
for the \largeset{}\ ensemble model as a function of the threshold $\beta$ used in the
consistency check of Algorithm~\ref{alg:consistent test point search}. The threshold $\beta$ is chosen based on our discussion in Section~\ref{sec: Thresholdvsauditoutcomes}.}
\label{fig:combined_metrics}
\end{figure*}

\begin{figure*}[t]
\centering

\begin{subfigure}{\textwidth}
\centering
\begin{minipage}{0.24\textwidth}
    \centering
    \includegraphics[width=\linewidth]{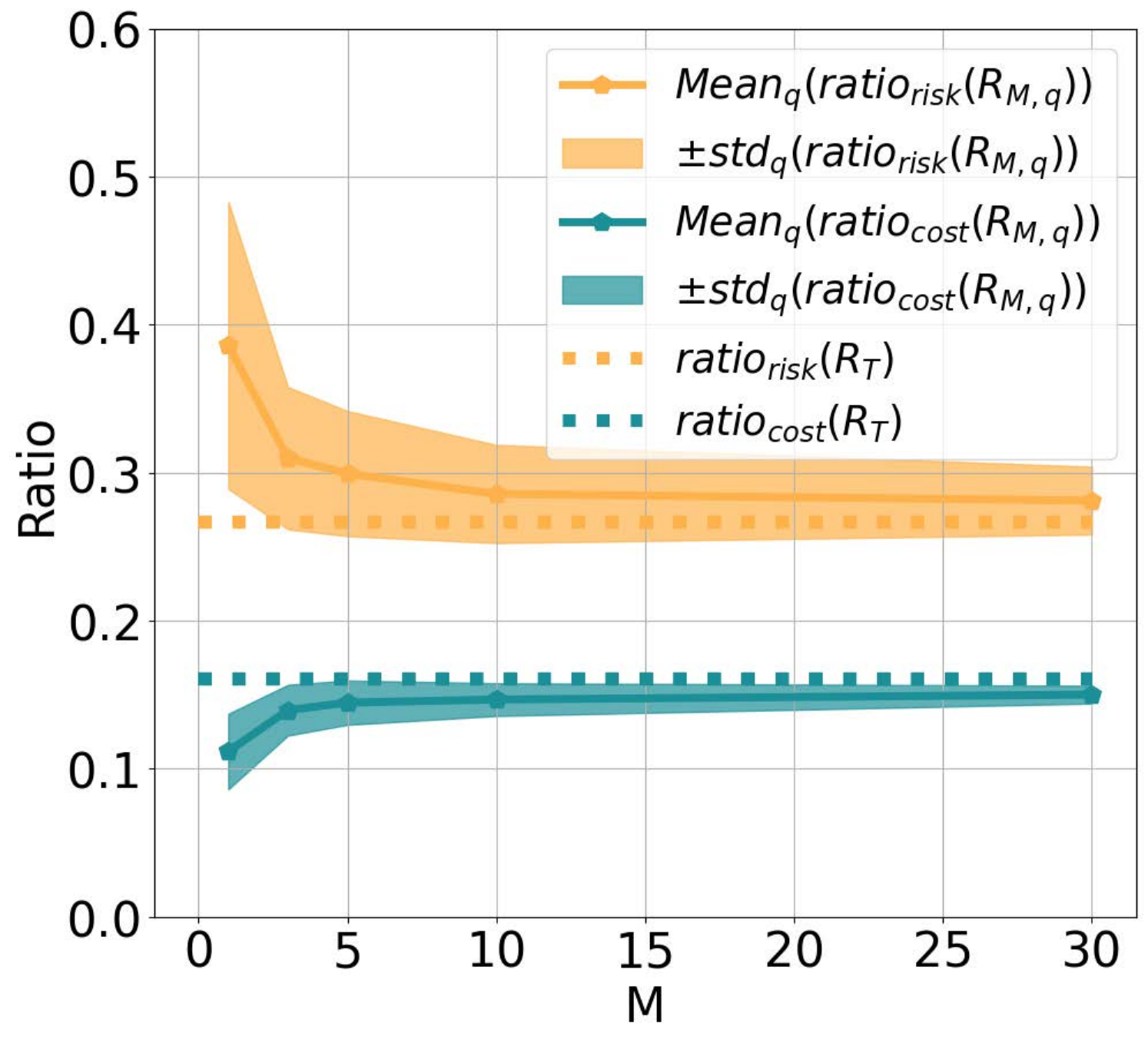}
\end{minipage}
\hfill
\begin{minipage}{0.24\textwidth}
    \centering
    \includegraphics[width=\linewidth]{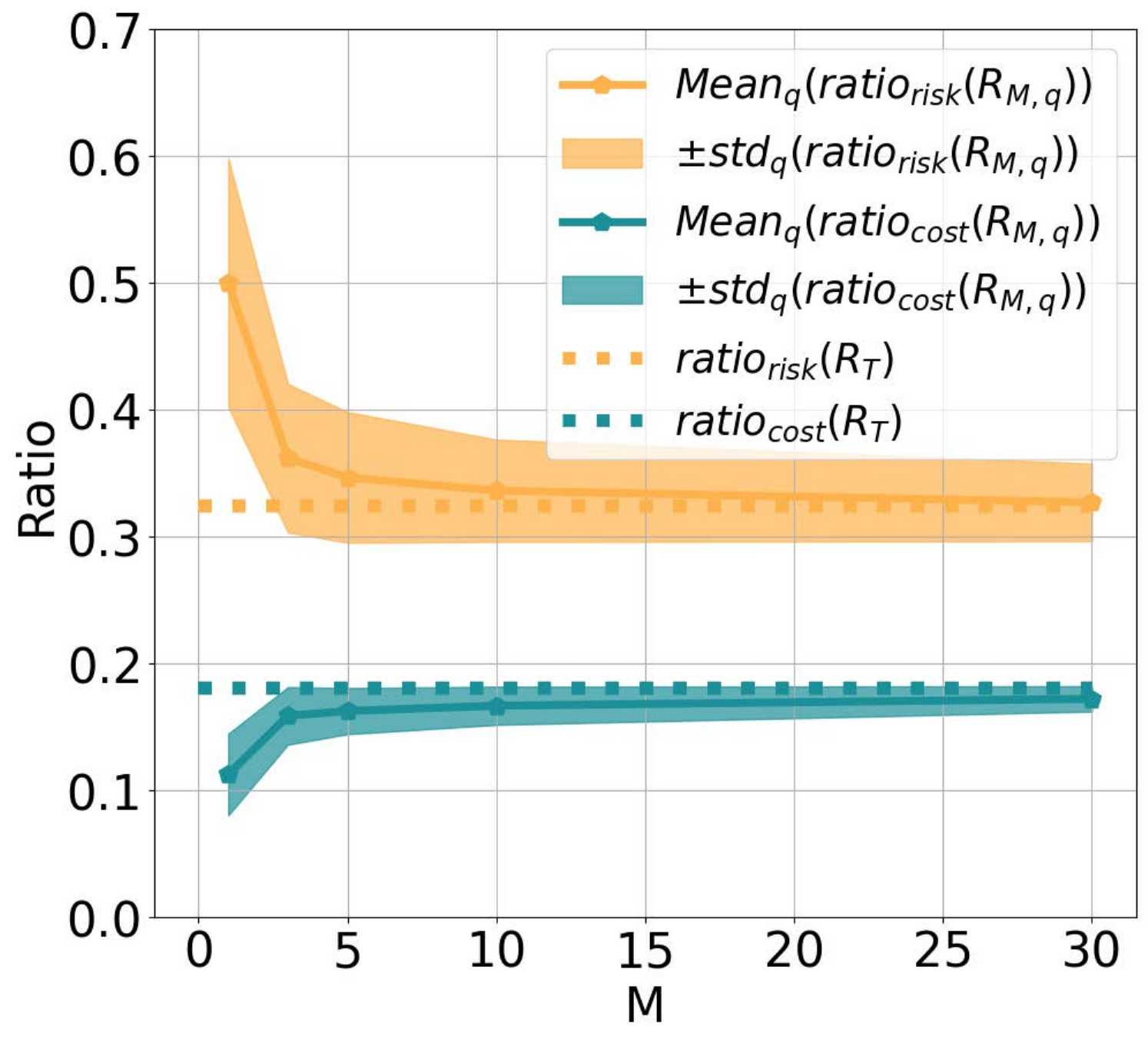}
\end{minipage}
\begin{minipage}{0.24\textwidth}
    \centering
    \includegraphics[width=\linewidth]{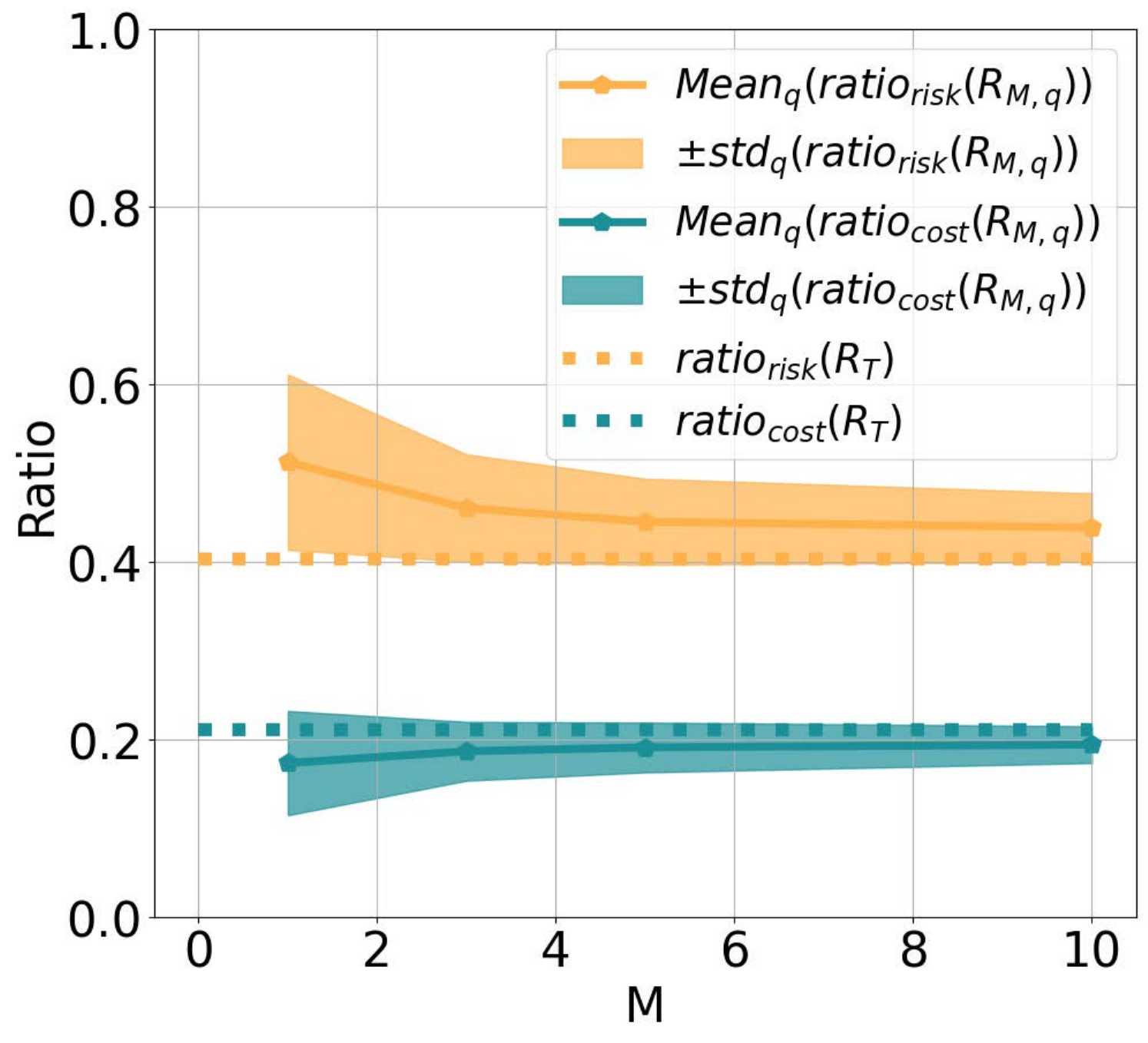}
\end{minipage}
\hfill
\begin{minipage}{0.24\textwidth}
    \centering
    \includegraphics[width=\linewidth]{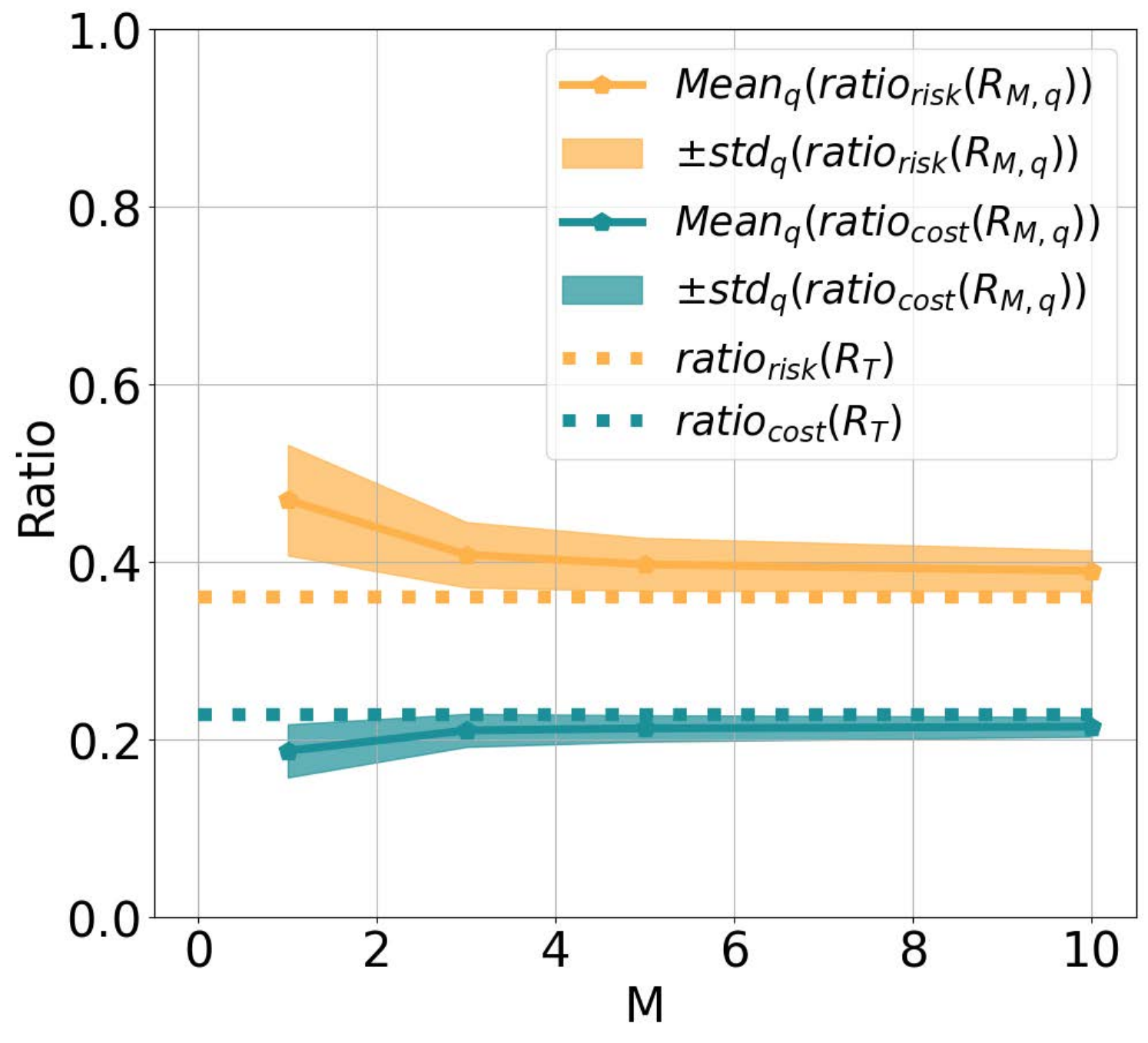}
\end{minipage}
\hfill
\caption{Audit cost ratio ($\auditcostratio$) versus audit risk ratio ($\auditriskratio$) as a function of the size of the \smallset{} ensemble. Results are shown (Left to Right) for: \textbf{BERT} on \textbf{SST2}, \textbf{RoBERTa} on \textbf{MRPC}, \textbf{BIG-SCIENCE T0} on \textbf{Bank}, \textbf{BIG-SCIENCE T0} on \textbf{Adult}. Increasing the ensemble size has a small impact in matching the \largeset{} once the ensemble is sufficiently large (around \Nretail = 10).}
\label{fig:praccostrisk}
\end{subfigure}

\vspace{1em}

\begin{subfigure}{\textwidth}
\centering
\begin{minipage}{0.24\textwidth}
    \centering
    \includegraphics[width=\linewidth]{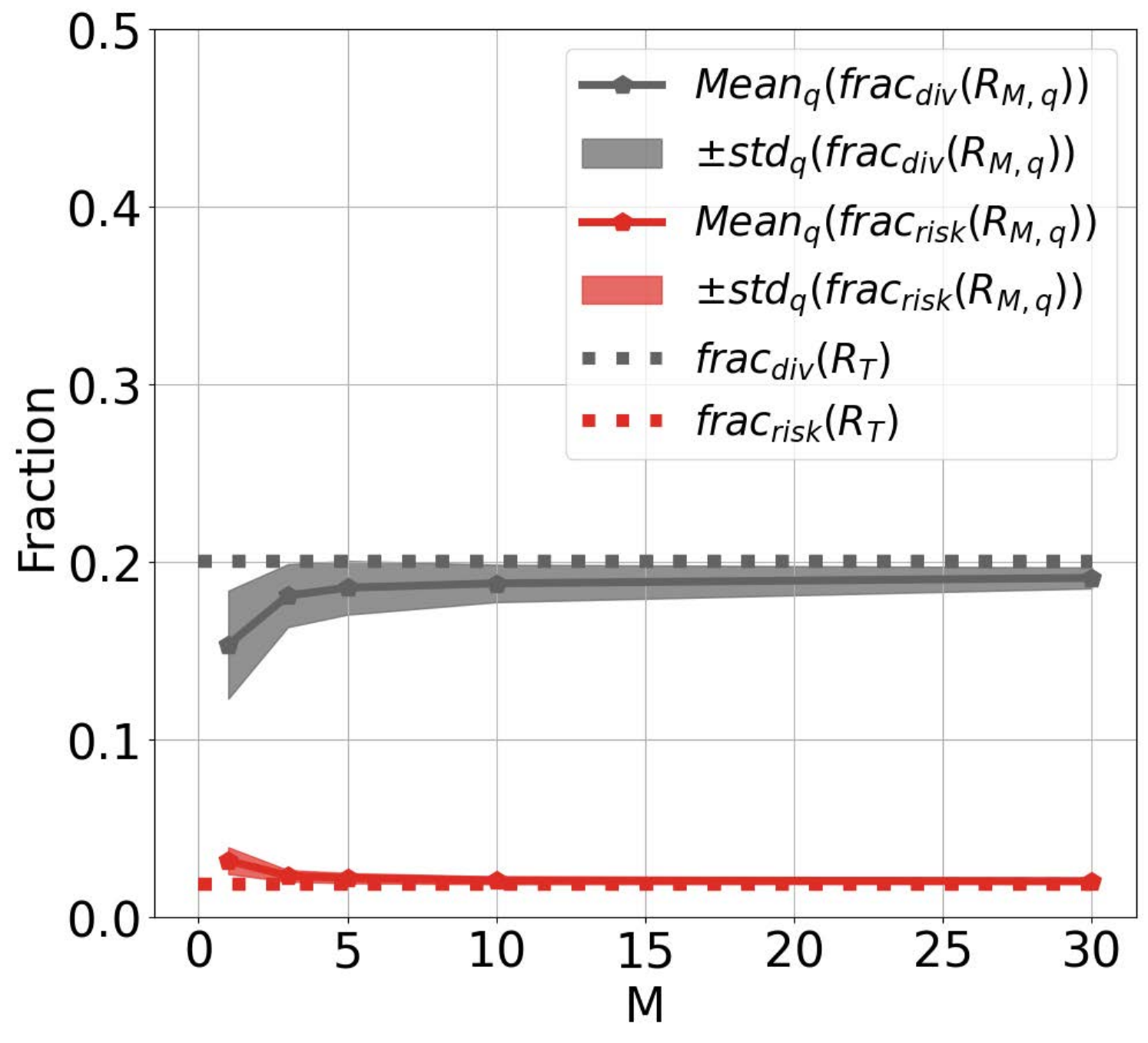}
\end{minipage}
\hfill
\begin{minipage}{0.24\textwidth}
    \centering
    \includegraphics[width=\linewidth]{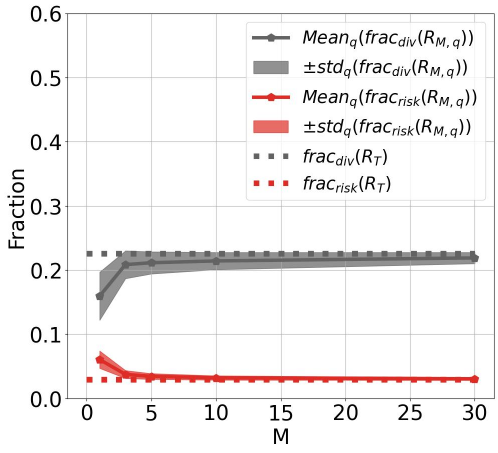}
\end{minipage}
\begin{minipage}{0.24\textwidth}
    \centering
    \includegraphics[width=\linewidth]{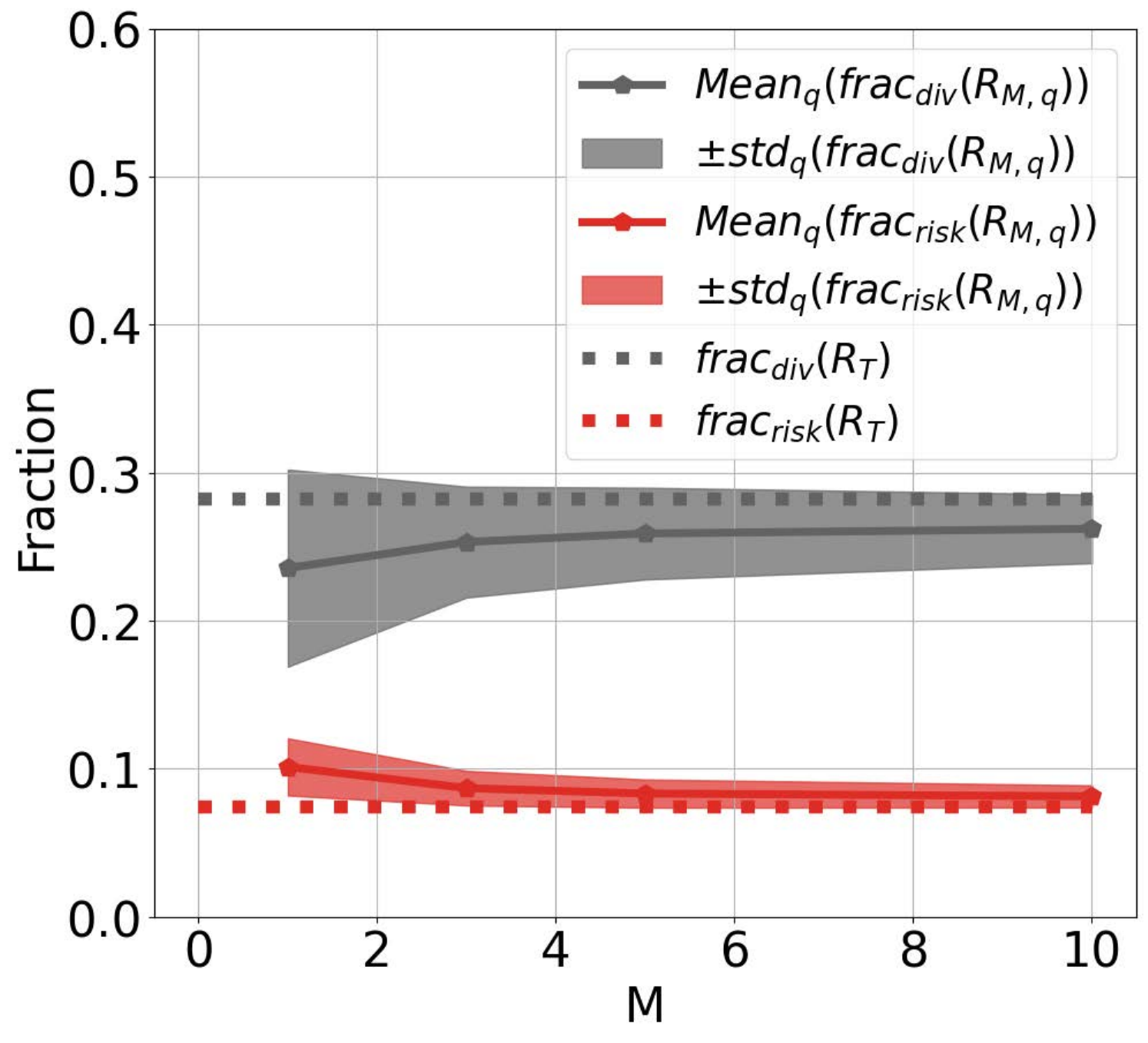}
\end{minipage}
\hfill
\begin{minipage}{0.24\textwidth}
    \centering
    \includegraphics[width=\linewidth]{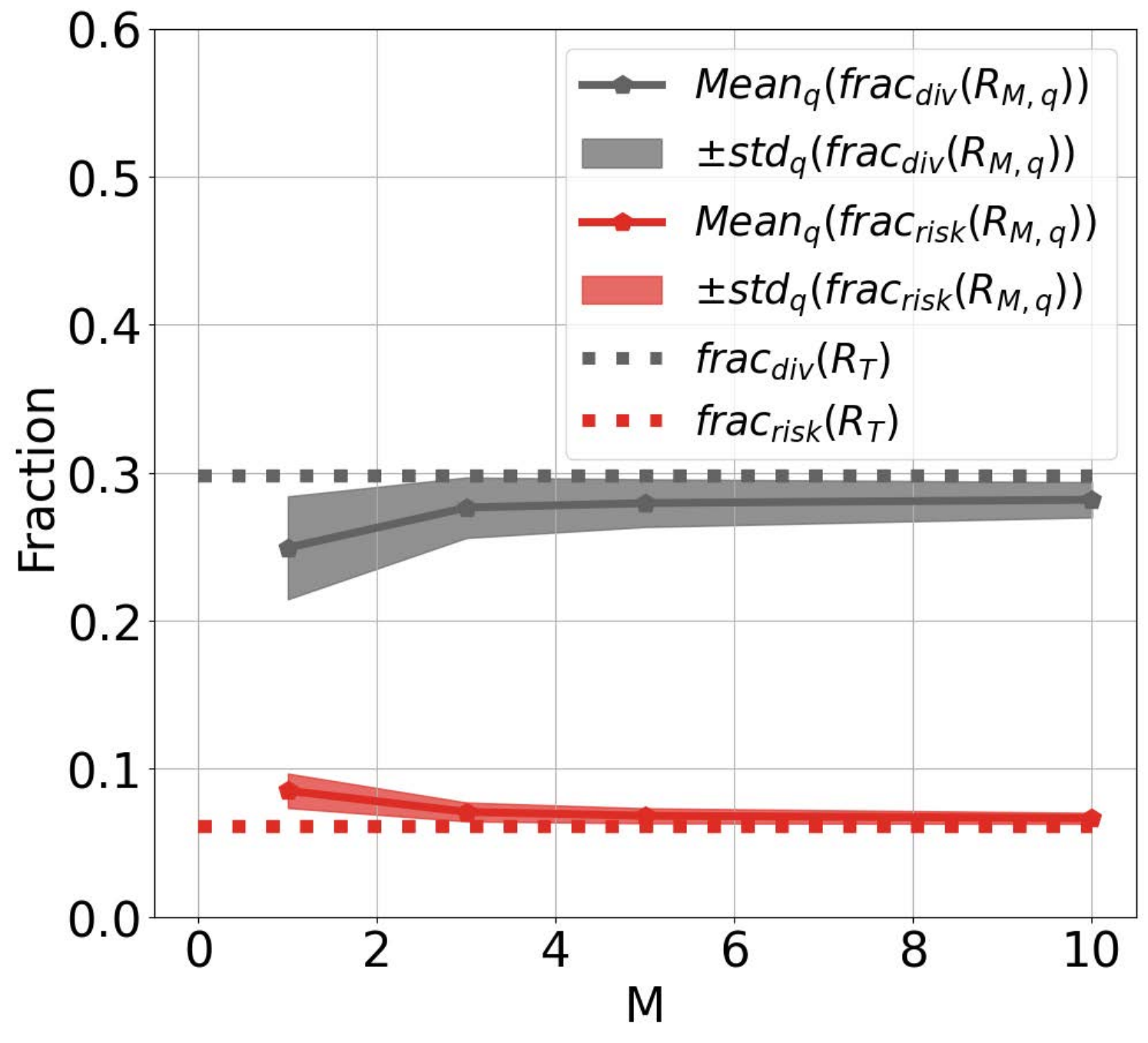}
\end{minipage}
\hfill
\caption{Diversion fraction ($\auditdivfrac$) versus audit risk fraction ($\auditriskfrac$) as a function of the size of the \smallset{} ensemble. Results are shown (Left to Right) for: \textbf{BERT} on \textbf{SST2}, \textbf{RoBERTa} on \textbf{MRPC}, \textbf{BIG-SCIENCE T0} on \textbf{Bank}, \textbf{BIG-SCIENCE T0} on \textbf{Adult}. Diversion and risk fractions nearly match those of the warehouse, even with relatively modest ensemble sizes.}
\label{fig:pracdivrisk}
\end{subfigure}

\caption{Comparing the audit metrics in Definition \ref{def:costsandrisks} for  \smallset{} ensemble models as a function of the number of models $\Nretail$. For each size $\Nretail$, $500$ ensembles of size $\Nretail$ were generated by sampling with replacement from the \largeset{} set. Figures show the mean $\pm$ standard deviation (std) of the metrics in Definition \ref{def:costsandrisks} taken over the $500$ \smallset{} ensembles.}
\label{fig:ensemble_size_metrics}
\end{figure*}

\subsubsection{Retail set audit outcomes match warehouse audit outcomes as $\Nretail$ increases}\label{sec:Addition retailauditoutcomes}
Section~\ref{sec: Retailauditoutcomes} presented the auditing results for finite sized ensembles for paraphrase detection on the MRPC dataset using ensembles of independently fine-tuned BERT models. This appendix provides the corresponding results for the remaining datasets and model architectures. Figure \ref{fig:ensemble_size_metrics} below report the audit cost ratio, audit risk ratio, diversion fraction, and audit risk fraction as functions of the retail ensemble size $\Nretail$, illustrating that the same convergence behaviour is consistently observed across all experimental settings.

\subsubsection{Understanding retail and warehouse audit disagreements}\label{sec: Combinedoutcomes}

The results in section \ref{sec: Retailauditoutcomes} compared retail and warehouse ensembles using aggregate auditing metrics from Definition~\ref{def:costsandrisks},  the audit cost ratio, audit risk ratio, diversion fraction, and audit risk fraction. While these metrics quantify the overall performance of the auditing system, they do not reveal how or why the retail and warehouse audit decisions differ.

To better understand these differences, we next examine the combined audit outcomes introduced in Table~\ref{tab: combined audit outcomes}. This allows us to identify whether disagreements arise from changes in audit cost, audit risk, or both relative to the baseline, and how these disagreement patterns evolve as the retail ensemble size increases.

Table~\ref{tab:test_instances_BERT_MRPC} summarizes how the audit outcomes produced by retail ensembles compare with those of the warehouse ensemble for paraphrase detection on the MRPC dataset using ensembles of fine-tuned BERT models. For $\Nretail=1$, corresponding to no ensembling, nearly $73\%$ of the test instances already fall into the baseline category, indicating that a single model agrees with the warehouse audit decision for the majority of inputs. At the same time, a very small fraction of test instances fall in the Neutral, Mixed, and Risk decreases category. Recall that both the mixed category and neutral category include test instances where the baseline and candidate models disagreed in their predicted labels. In the mixed category, the audit decisions for both the baseline and candidate models are incorrect: they either divert a correct prediction or automate an incorrect one. In the neutral category, on the other hand, the audit decisions for both the baseline and candidate models are correct: they either divert an incorrect prediction or automate a correct prediction. Across all experiments, we see that this combination of decisions disappears very quickly with only a small increase in ensemble size. 

We also note that more than $16.6\%$ of the test instances fall into the Cost decreases category, compared with only $4.2\%$ in the Cost increases category, indicating that a single model tends to audit fewer instances than the warehouse and therefore incurs lower audit cost. However, close to $5\%$ test instances also fall in the Risk increases category. This indicates that the reduction in audit cost achieved without ensembling comes primarily at the expense of increased audit risk.

As the retail ensemble size increases, the fraction of test instances in the Baseline category steadily rises from $73.1\%$ for $\Nretail=1$ to $90.9\%$ for $\Nretail=30$, demonstrating that retail audit outcomes increasingly match the warehouse audit outcomes with an increase in ensemble size. Correspondingly, the fractions of test instances assigned to all other categories decrease with increasing $\Nretail$. In particular, the Risk increases category decreases from almost $5\%$ to close to $1\%$, even though this comes at the expense of the Cost decreases category, which decreases from $16.6\%$ to $7.9\%$. These results show that while small retail ensembles benefit from lower audit costs, they also have a higher audit risk than the baseline, and increasing the ensemble size progressively recovers the warehouse auditing behavior. We observe a similar trend across all other architectures and datasets reported in Tables~\ref{tab:test_instances_BERT_SST2},~\ref{tab:test_instances_ROBERTA_MRPC},~\ref{tab:test_instances_BIGSCIENCE_T0_BANK}, and ~\ref{tab:test_instances_BIGSCIENCE_T0_ADULT}.

\begin{table*}[t]
\centering

\caption{Audit outcomes and aggregated audit outcomes for a paraphrase detection task on the \textbf{MRPC} dataset using ensembles of independently fine-tuned \textbf{BERT} models.}
\label{tab:test_instances_BERT_MRPC}

\begin{subtable}{\textwidth}
\centering
\renewcommand{\arraystretch}{1.5}
\setlength{\tabcolsep}{6pt}

\caption{Table reporting the fraction of test instances assigned to each outcome label $o\in\{0,\ldots,15\}$ from Table~\ref{tab: combined audit outcomes}. The baseline model is the ensemble of the Rashomon \largeset{}~$\retailset_{\Nlib}$, with model error $\modelerrorlabel(x; \retailset_{\Nlib})$, and audit decision $\riskassess(x ; \retailset_{\Nlib})$, while the candidate models are ensembles of $q=\{1,\dots,\nsampretail\}$ random draws of the Rashomon \smallset{} set $\retailset_{\Nretail,q}$, with model error $\modelerrorlabel(x ; \retailset_{\Nretail,q})$ and audit decision $\riskassess(x ; \retailset_{\Nretail,q})$. The tuple $(\modelerrorlabel(x; \retailset_{\Nlib}),\riskassess(x; \retailset_{\Nlib}),\modelerrorlabel(x; \retailset_{\Nretail,q}),\riskassess(x; \retailset_{\Nretail,q}))$ decides the combined audit outcomes, with label $o\in\{0,\ldots,15\}$. The table reports the mean $\pm$ standard deviation of the fraction of test instances assigned to each outcome label across $q=\{1,\dots,\nsampretail\}$ random retail ensemble draws for different values of $\Nretail$. A desirable auditing strategy is one that minimizes excess audit costs and risks by increasing the Baseline category.}
\label{tab:test_instances_outcomes_BERT_MRPC}

\begin{tabular}{llccccc}
\toprule
\textbf{o} &
\textbf{Meaning} &
$\mathbf{M=1}$ &
$\mathbf{M=3}$ &
$\mathbf{M=5}$ &
$\mathbf{M=10}$ &
$\mathbf{M=30}$\\
\midrule

\rowcolor{cbaseline}
0 & Baseline &
$0.586\pm0.012$ &
$0.585\pm0.010$ &
$0.591\pm0.008$ &
$0.596\pm0.005$ &
$0.601\pm0.002$\\

\rowcolor{ccost}
1 & Cost $\uparrow$ &
$0.015\pm0.011$ &
$0.018\pm0.010$ &
$0.012\pm0.008$ &
$0.007\pm0.005$ &
$0.002\pm0.002$\\

\rowcolor{crisk}
2 & Risk $\uparrow$ &
$0.000\pm0.000$ &
$0.000\pm0.000$ &
$0.000\pm0.000$ &
$0.000\pm0.000$ &
$0.000\pm0.000$\\

\rowcolor{cneutral}
3 & Neutral &
$0.002\pm0.003$ &
$0.000\pm0.000$ &
$0.000\pm0.000$ &
$0.000\pm0.000$ &
$0.000\pm0.000$\\

\rowcolor{ccostdown}
4 & Cost $\downarrow$ &
$0.133\pm0.022$ &
$0.083\pm0.019$ &
$0.076\pm0.017$ &
$0.068\pm0.014$ &
$0.062\pm0.011$\\

\rowcolor{cbaseline}
5 & Baseline &
$0.082\pm0.023$ &
$0.149\pm0.020$ &
$0.161\pm0.018$ &
$0.173\pm0.015$ &
$0.183\pm0.011$\\

\rowcolor{ctradeoff}
6 & Risk $\uparrow$, Cost $\downarrow$ &
$0.005\pm0.006$ &
$0.000\pm0.000$ &
$0.000\pm0.000$ &
$0.000\pm0.000$ &
$0.000\pm0.000$\\

\rowcolor{ccostdown}
7 & Cost $\downarrow$ &
$0.033\pm0.011$ &
$0.020\pm0.007$ &
$0.015\pm0.006$ &
$0.011\pm0.005$ &
$0.008\pm0.004$\\

\rowcolor{criskdown}
8 & Risk $\downarrow$ &
$0.000\pm0.000$ &
$0.000\pm0.000$ &
$0.000\pm0.000$ &
$0.000\pm0.000$ &
$0.000\pm0.000$\\

\rowcolor{ctradeoff}
9 & Risk $\downarrow$, Cost $\uparrow$ &
$0.000\pm0.001$ &
$0.000\pm0.000$ &
$0.000\pm0.000$ &
$0.000\pm0.000$ &
$0.000\pm0.000$\\

\rowcolor{cbaseline}
10 & Baseline &
$0.016\pm0.001$ &
$0.016\pm0.001$ &
$0.016\pm0.001$ &
$0.017\pm0.001$ &
$0.017\pm0.001$\\

\rowcolor{criskdown}
11 & Risk $\downarrow$ &
$0.001\pm0.001$ &
$0.001\pm0.001$ &
$0.001\pm0.001$ &
$0.000\pm0.001$ &
$0.000\pm0.001$\\

\rowcolor{cneutral}
12 & Neutral &
$0.006\pm0.005$ &
$0.000\pm0.001$ &
$0.000\pm0.000$ &
$0.000\pm0.000$ &
$0.000\pm0.000$\\

\rowcolor{ccost}
13 & Cost $\uparrow$ &
$0.027\pm0.010$ &
$0.023\pm0.007$ &
$0.019\pm0.006$ &
$0.015\pm0.006$ &
$0.009\pm0.004$\\

\rowcolor{crisk}
14 & Risk $\uparrow$&
$0.047\pm0.015$ &
$0.020\pm0.008$ &
$0.016\pm0.007$ &
$0.012\pm0.006$ &
$0.010\pm0.004$\\

\rowcolor{cbaseline}
15 & Baseline &
$0.047\pm0.015$ &
$0.084\pm0.010$ &
$0.092\pm0.009$ &
$0.100\pm0.007$ &
$0.108\pm0.005$\\

\bottomrule
\end{tabular}
\end{subtable}

\vspace{3mm}

\begin{subtable}{\textwidth}
\renewcommand{\arraystretch}{1.2}
\centering

\caption{Summary of the fraction of test instances assigned to each high-level audit category obtained by aggregating the corresponding outcome labels from Table~\ref{tab:test_instances_outcomes_BERT_MRPC}. Specifically, Baseline aggregates outcomes $\{0,5,10,15\}$, Risk increases aggregates $\{2,14\}$, Risk decreases aggregates $\{8,11\}$, Cost increases aggregates $\{1,13\}$, Cost decreases aggregates $\{4,7\}$, Neutral aggregates $\{3,12\}$, and Mixed aggregates $\{6,9\}$.}
\label{tab:audit_category_summary_BERT_MRPC}

\small
\begin{tabular}{lccccc}
\toprule
\textbf{Category} &
$\mathbf{M=1}$ &
$\mathbf{M=3}$ &
$\mathbf{M=5}$ &
$\mathbf{M=10}$ &
$\mathbf{M=30}$\\
\midrule
\cellcolor{cbaseline}\textbf{Baseline}
& 0.731 & 0.834 & 0.860 & 0.886 & 0.909\\

\cellcolor{crisk}\textbf{Risk increases}
& 0.047 & 0.020 & 0.016 & 0.012 & 0.010\\

\cellcolor{criskdown}\textbf{Risk decreases}
& 0.001 & 0.001 & 0.001 & 0.000 & 0.000\\

\cellcolor{ccost}\textbf{Cost increases}
& 0.042 & 0.041 & 0.031 & 0.022 & 0.011\\

\cellcolor{ccostdown}\textbf{Cost decreases}
& 0.166 & 0.103 & 0.091 & 0.079 & 0.070\\

\cellcolor{cneutral}\textbf{Neutral}
& 0.008 & 0.000 & 0.000 & 0.000 & 0.000\\

\cellcolor{ctradeoff}\textbf{Mixed}
& 0.005 & 0.000 & 0.000 & 0.000 & 0.000\\
\bottomrule
\end{tabular}

\end{subtable}

\end{table*}

\begin{table*}[t]
\centering

\caption{Audit outcomes and aggregated audit outcomes for a sentiment classification task on the \textbf{SST2} dataset using ensembles of independently fine-tuned \textbf{BERT} models.}
\label{tab:test_instances_BERT_SST2}

\begin{subtable}{\textwidth}
\centering
\renewcommand{\arraystretch}{1.5}
\setlength{\tabcolsep}{6pt}

\caption{Table reporting the fraction of test instances assigned to each outcome label $o\in\{0,\ldots,15\}$ from Table~\ref{tab: combined audit outcomes}. The baseline model is the ensemble of the Rashomon \largeset{}~$\retailset_{\Nlib}$, with model error $\modelerrorlabel(x; \retailset_{\Nlib})$, and audit decision $\riskassess(x ; \retailset_{\Nlib})$, while the candidate models are ensembles of $q=\{1,\dots,\nsampretail\}$ random draws of the Rashomon \smallset{} set $\retailset_{\Nretail,q}$, with model error $\modelerrorlabel(x ; \retailset_{\Nretail,q})$ and audit decision $\riskassess(x ; \retailset_{\Nretail,q})$. The tuple $(\modelerrorlabel(x; \retailset_{\Nlib}),\riskassess(x; \retailset_{\Nlib}),\modelerrorlabel(x; \retailset_{\Nretail,q}),\riskassess(x; \retailset_{\Nretail,q}))$ decides the combined audit outcomes, with label $o\in\{0,\ldots,15\}$. The table reports the mean $\pm$ standard deviation of the fraction of test instances assigned to each outcome label across $q=\{1,\dots,\nsampretail\}$ random retail ensemble draws for different values of $\Nretail$. A desirable auditing strategy is one that minimizes excess audit costs and risks by increasing the Baseline category.}
\label{tab:test_instances_outcomes_BERT_SST2}

\begin{tabular}{llccccc}
\toprule
\textbf{o} &
\textbf{Meaning} &
$\mathbf{M=1}$ &
$\mathbf{M=3}$ &
$\mathbf{M=5}$ &
$\mathbf{M=10}$ &
$\mathbf{M=30}$\\
\midrule

\rowcolor{cbaseline}
0 & Baseline &
$0.734\pm0.008$ &
$0.739\pm0.006$ &
$0.742\pm0.004$ &
$0.749\pm0.001$ &
$0.750\pm0.000$\\

\rowcolor{ccost}
1 & Cost $\uparrow$ &
$0.015\pm0.008$ &
$0.011\pm0.006$ &
$0.008\pm0.004$ &
$0.001\pm0.001$ &
$0.000\pm0.000$\\

\rowcolor{crisk}
2 & Risk $\uparrow$ &
$0.000\pm0.000$ &
$0.000\pm0.000$ &
$0.000\pm0.000$ &
$0.000\pm0.000$ &
$0.000\pm0.000$\\

\rowcolor{cneutral}
3 & Neutral &
$0.001\pm0.001$ &
$0.000\pm0.000$ &
$0.000\pm0.000$ &
$0.000\pm0.000$ &
$0.000\pm0.000$\\

\rowcolor{ccostdown}
4 & Cost $\downarrow$ &
$0.080\pm0.017$ &
$0.057\pm0.012$ &
$0.051\pm0.011$ &
$0.040\pm0.005$ &
$0.040\pm0.004$\\

\rowcolor{cbaseline}
5 & Baseline &
$0.077\pm0.017$ &
$0.110\pm0.012$ &
$0.120\pm0.011$ &
$0.137\pm0.005$ &
$0.138\pm0.004$\\

\rowcolor{ctradeoff}
6 & Risk $\uparrow$, Cost $\downarrow$ &
$0.002\pm0.002$ &
$0.000\pm0.000$ &
$0.000\pm0.000$ &
$0.000\pm0.000$ &
$0.000\pm0.000$\\

\rowcolor{ccostdown}
7 & Cost $\downarrow$ &
$0.023\pm0.007$ &
$0.014\pm0.004$ &
$0.010\pm0.003$ &
$0.004\pm0.002$ &
$0.004\pm0.002$\\

\rowcolor{criskdown}
8 & Risk $\downarrow$ &
$0.000\pm0.000$ &
$0.000\pm0.000$ &
$0.000\pm0.000$ &
$0.000\pm0.000$ &
$0.000\pm0.000$\\

\rowcolor{ctradeoff}
9 & Risk $\downarrow$, Cost $\uparrow$ &
$0.000\pm0.000$ &
$0.000\pm0.000$ &
$0.000\pm0.000$ &
$0.000\pm0.000$ &
$0.000\pm0.000$\\

\rowcolor{cbaseline}
10 & Baseline &
$0.012\pm0.001$ &
$0.012\pm0.001$ &
$0.012\pm0.001$ &
$0.013\pm0.000$ &
$0.013\pm0.000$\\

\rowcolor{criskdown}
11 & Risk $\downarrow$&
$0.001\pm0.001$ &
$0.001\pm0.001$ &
$0.001\pm0.001$ &
$0.000\pm0.000$ &
$0.000\pm0.000$\\

\rowcolor{cneutral}
12 & Neutral &
$0.001\pm0.002$ &
$0.000\pm0.000$ &
$0.000\pm0.000$ &
$0.000\pm0.000$ &
$0.000\pm0.000$\\

\rowcolor{ccost}
13 & Cost $\uparrow$ &
$0.011\pm0.003$ &
$0.008\pm0.003$ &
$0.006\pm0.003$ &
$0.002\pm0.001$ &
$0.001\pm0.001$\\

\rowcolor{crisk}
14 & Risk $\uparrow$&
$0.018\pm0.005$ &
$0.011\pm0.003$ &
$0.010\pm0.003$ &
$0.008\pm0.002$ &
$0.008\pm0.001$\\

\rowcolor{cbaseline}
15 & Baseline &
$0.026\pm0.006$ &
$0.037\pm0.004$ &
$0.040\pm0.004$ &
$0.047\pm0.002$ &
$0.048\pm0.002$\\

\bottomrule
\end{tabular}
\end{subtable}

\vspace{3mm}

\begin{subtable}{\textwidth}
\renewcommand{\arraystretch}{1.2}
\centering

\caption{Summary of the fraction of test instances assigned to each high-level audit category obtained by aggregating the corresponding outcome labels from Table~\ref{tab:test_instances_outcomes_BERT_SST2}. Specifically, Baseline aggregates outcomes $\{0,5,10,15\}$, Risk increases aggregates $\{2,14\}$, Risk decreases aggregates $\{8,11\}$, Cost increases aggregates $\{1,13\}$, Cost decreases aggregates $\{4,7\}$, Neutral aggregates $\{3,12\}$, and Mixed aggregates $\{6,9\}$.}
\label{tab:audit_category_summary_BERT_SST2}

\small
\small
\begin{tabular}{lccccc}
\toprule
\textbf{Category} &
$\mathbf{M=1}$ &
$\mathbf{M=3}$ &
$\mathbf{M=5}$ &
$\mathbf{M=10}$ &
$\mathbf{M=30}$\\
\midrule
\cellcolor{cbaseline}\textbf{Baseline} &
0.849 &
0.898 &
0.914 &
0.946 &
0.949\\

\cellcolor{crisk}\textbf{Risk increases} &
0.018 &
0.011 &
0.010 &
0.008 &
0.008\\

\cellcolor{criskdown}\textbf{Risk decreases} &
0.001 &
0.001 &
0.001 &
0.000 &
0.000\\

\cellcolor{ccost}\textbf{Cost increases} &
0.026 &
0.019 &
0.014 &
0.003 &
0.001\\

\cellcolor{ccostdown}\textbf{Cost decreases} &
0.103 &
0.071 &
0.061 &
0.044 &
0.044\\

\cellcolor{cneutral}\textbf{Neutral} &
0.002 &
0.000 &
0.000 &
0.000 &
0.000\\

\cellcolor{ctradeoff}\textbf{Mixed} &
0.002 &
0.000 &
0.000 &
0.000 &
0.000\\
\bottomrule
\end{tabular}

\end{subtable}

\end{table*}

\begin{table*}[t]
\centering

\caption{Audit outcomes and aggregated audit outcomes for a paraphrase detection task on the \textbf{MRPC} dataset using ensembles of independently fine-tuned \textbf{RoBERTa} models.}
\label{tab:test_instances_ROBERTA_MRPC}

\begin{subtable}{\textwidth}
\centering
\renewcommand{\arraystretch}{1.5}
\setlength{\tabcolsep}{6pt}

\caption{Table reporting the fraction of test instances assigned to each outcome label $o\in\{0,\ldots,15\}$ from Table~\ref{tab: combined audit outcomes}. The baseline model is the ensemble of the Rashomon \largeset{}~$\retailset_{\Nlib}$, with model error $\modelerrorlabel(x; \retailset_{\Nlib})$, and audit decision $\riskassess(x ; \retailset_{\Nlib})$, while the candidate models are ensembles of $q=\{1,\dots,\nsampretail\}$ random draws of the Rashomon \smallset{} set $\retailset_{\Nretail,q}$, with model error $\modelerrorlabel(x ; \retailset_{\Nretail,q})$ and audit decision $\riskassess(x ; \retailset_{\Nretail,q})$. The tuple $(\modelerrorlabel(x; \retailset_{\Nlib}),\riskassess(x; \retailset_{\Nlib}),\modelerrorlabel(x; \retailset_{\Nretail,q}),\riskassess(x; \retailset_{\Nretail,q}))$ decides the combined audit outcomes, with label $o\in\{0,\ldots,15\}$. The table reports the mean $\pm$ standard deviation of the fraction of test instances assigned to each outcome label across $q=\{1,\dots,\nsampretail\}$ random retail ensemble draws for different values of $\Nretail$. A desirable auditing strategy is one that minimizes excess audit costs and risks by increasing the Baseline category.}
\label{tab:test_instances_outcomes_ROBERTA_MRPC}

\begin{tabular}{llccccc}
\toprule
\textbf{o} &
\textbf{Meaning} &
$\mathbf{M=1}$ &
$\mathbf{M=3}$ &
$\mathbf{M=5}$ &
$\mathbf{M=10}$ &
$\mathbf{M=30}$\\
\midrule

\rowcolor{cbaseline}
0 & Baseline &
$0.683\pm0.018$ &
$0.692\pm0.010$ &
$0.699\pm0.007$ &
$0.706\pm0.005$ &
$0.712\pm0.002$\\

\rowcolor{ccost}
1 & Cost $\uparrow$ &
$0.025\pm0.016$ &
$0.021\pm0.010$ &
$0.014\pm0.007$ &
$0.007\pm0.005$ &
$0.002\pm0.002$\\

\rowcolor{crisk}
2 & Risk $\uparrow$ &
$0.001\pm0.001$ &
$0.000\pm0.000$ &
$0.000\pm0.000$ &
$0.000\pm0.000$ &
$0.000\pm0.000$\\

\rowcolor{cneutral}
3 & Neutral &
$0.005\pm0.004$ &
$0.000\pm0.001$ &
$0.000\pm0.000$ &
$0.000\pm0.000$ &
$0.000\pm0.000$\\

\rowcolor{ccostdown}
4 & Cost $\downarrow$ &
$0.090\pm0.017$ &
$0.062\pm0.014$ &
$0.055\pm0.013$ &
$0.047\pm0.011$ &
$0.039\pm0.008$\\

\rowcolor{cbaseline}
5 & Baseline &
$0.061\pm0.019$ &
$0.107\pm0.018$ &
$0.120\pm0.016$ &
$0.135\pm0.013$ &
$0.149\pm0.009$\\

\rowcolor{ctradeoff}
6 & Risk $\uparrow$, Cost $\downarrow$ &
$0.013\pm0.009$ &
$0.001\pm0.002$ &
$0.000\pm0.001$ &
$0.000\pm0.000$ &
$0.000\pm0.000$\\

\rowcolor{ccostdown}
7 & Cost $\downarrow$ &
$0.032\pm0.008$ &
$0.026\pm0.008$ &
$0.021\pm0.007$ &
$0.014\pm0.005$ &
$0.008\pm0.004$\\

\rowcolor{criskdown}
8 & Risk $\downarrow$ &
$0.000\pm0.000$ &
$0.000\pm0.000$ &
$0.000\pm0.000$ &
$0.000\pm0.000$ &
$0.000\pm0.000$\\

\rowcolor{ctradeoff}
9 & Risk $\downarrow$, Cost $\uparrow$ &
$0.000\pm0.001$ &
$0.000\pm0.000$ &
$0.000\pm0.000$ &
$0.000\pm0.000$ &
$0.000\pm0.000$\\

\rowcolor{cbaseline}
10 & Baseline &
$0.024\pm0.003$ &
$0.025\pm0.002$ &
$0.025\pm0.002$ &
$0.026\pm0.001$ &
$0.027\pm0.000$\\

\rowcolor{criskdown}
11 & Risk $\downarrow$&
$0.003\pm0.002$ &
$0.002\pm0.002$ &
$0.002\pm0.002$ &
$0.001\pm0.001$ &
$0.000\pm0.000$\\

\rowcolor{cneutral}
12 &Neutral&
$0.007\pm0.005$ &
$0.001\pm0.002$ &
$0.000\pm0.001$ &
$0.000\pm0.000$ &
$0.000\pm0.000$\\

\rowcolor{ccost}
13 & Cost $\uparrow$ &
$0.013\pm0.005$ &
$0.013\pm0.005$ &
$0.012\pm0.005$ &
$0.009\pm0.004$ &
$0.005\pm0.003$\\

\rowcolor{crisk}
14 & Risk $\uparrow$&
$0.023\pm0.006$ &
$0.012\pm0.005$ &
$0.009\pm0.004$ &
$0.006\pm0.003$ &
$0.004\pm0.002$\\

\rowcolor{cbaseline}
15 & Baseline &
$0.021\pm0.008$ &
$0.038\pm0.007$ &
$0.043\pm0.007$ &
$0.049\pm0.005$ &
$0.055\pm0.004$\\

\bottomrule
\end{tabular}
\end{subtable}

\vspace{3mm}

\begin{subtable}{\textwidth}
\renewcommand{\arraystretch}{1.2}
\centering

\caption{Summary of the fraction of test instances assigned to each high-level audit category obtained by aggregating the corresponding outcome labels from Table~\ref{tab:test_instances_outcomes_ROBERTA_MRPC}. Specifically, Baseline aggregates outcomes $\{0,5,10,15\}$, Risk increases aggregates $\{2,14\}$, Risk decreases aggregates $\{8,11\}$, Cost increases aggregates $\{1,13\}$, Cost decreases aggregates $\{4,7\}$, Neutral aggregates $\{3,12\}$, and Mixed aggregates $\{6,9\}$.}
\label{tab:audit_category_summary_ROBERTA_MRPC}

\small
\begin{tabular}{lccccc}
\toprule
\textbf{Category} &
$\mathbf{M=1}$ &
$\mathbf{M=3}$ &
$\mathbf{M=5}$ &
$\mathbf{M=10}$ &
$\mathbf{M=30}$\\
\midrule
\cellcolor{cbaseline}\textbf{Baseline}
& 0.789 & 0.862 & 0.887 & 0.916 & 0.943\\

\cellcolor{crisk}\textbf{Risk increases}
& 0.024 & 0.012 & 0.009 & 0.006 & 0.004\\

\cellcolor{criskdown}\textbf{Risk decreases}
& 0.003 & 0.002 & 0.002 & 0.001 & 0.000\\

\cellcolor{ccost}\textbf{Cost increases}
& 0.038 & 0.034 & 0.026 & 0.016 & 0.007\\

\cellcolor{ccostdown}\textbf{Cost decreases}
& 0.122 & 0.088 & 0.076 & 0.061 & 0.047\\

\cellcolor{cneutral}\textbf{Neutral}
& 0.012 & 0.001 & 0.000 & 0.000 & 0.000\\

\cellcolor{ctradeoff}\textbf{Mixed}
& 0.013 & 0.001 & 0.000 & 0.000 & 0.000\\
\bottomrule
\end{tabular}

\end{subtable}

\end{table*}

\begin{table*}[t]
\centering

\caption{Audit outcomes and aggregated audit outcomes for a binary classification task on the \textbf{Bank} dataset using ensembles of independently fine-tuned \textbf{BIG-SCIENCE T0} models.}
\label{tab:test_instances_BIGSCIENCE_T0_BANK}

\begin{subtable}{\textwidth}
\centering
\renewcommand{\arraystretch}{1.5}
\setlength{\tabcolsep}{6pt}

\caption{Table reporting the fraction of test instances assigned to each outcome label $o\in\{0,\ldots,15\}$ from Table~\ref{tab: combined audit outcomes}. The baseline model is the ensemble of the Rashomon \largeset{}~$\retailset_{\Nlib}$, with model error $\modelerrorlabel(x; \retailset_{\Nlib})$, and audit decision $\riskassess(x ; \retailset_{\Nlib})$, while the candidate models are ensembles of $q=\{1,\dots,\nsampretail\}$ random draws of the Rashomon \smallset{} set $\retailset_{\Nretail,q}$, with model error $\modelerrorlabel(x ; \retailset_{\Nretail,q})$ and audit decision $\riskassess(x ; \retailset_{\Nretail,q})$. The tuple $(\modelerrorlabel(x; \retailset_{\Nlib}),\riskassess(x; \retailset_{\Nlib}),\modelerrorlabel(x; \retailset_{\Nretail,q}),\riskassess(x; \retailset_{\Nretail,q}))$ decides the combined audit outcomes, with label $o\in\{0,\ldots,15\}$. The table reports the mean $\pm$ standard deviation of the fraction of test instances assigned to each outcome label across $q=\{1,\dots,\nsampretail\}$ random retail ensemble draws for different values of $\Nretail$. A desirable auditing strategy is one that minimizes excess audit costs and risks by increasing the Baseline category.}
\label{tab:test_instances_outcomes_BIGSCIENCE_T0_BANK}

\begin{tabular}{llcccc}
\toprule
\textbf{o} &
\textbf{Meaning} &
$\mathbf{M=1}$ &
$\mathbf{M=3}$ &
$\mathbf{M=5}$ &
$\mathbf{M=10}$\\
\midrule

\rowcolor{cbaseline}
0 & Baseline &
$0.562\pm0.024$ &
$0.579\pm0.009$ &
$0.583\pm0.007$ &
$0.586\pm0.003$\\

\rowcolor{ccost}
1 & Cost $\uparrow$ &
$0.024\pm0.021$ &
$0.009\pm0.009$ &
$0.005\pm0.007$ &
$0.002\pm0.003$\\

\rowcolor{crisk}
2 & Risk $\uparrow$ &
$0.000\pm0.000$ &
$0.000\pm0.000$ &
$0.000\pm0.000$ &
$0.000\pm0.000$\\

\rowcolor{cneutral}
3 & Neutral&
$0.002\pm0.003$ &
$0.000\pm0.000$ &
$0.000\pm0.000$ &
$0.000\pm0.000$\\

\rowcolor{ccostdown}
4 & Cost $\downarrow$ &
$0.097\pm0.029$ &
$0.081\pm0.021$ &
$0.075\pm0.018$ &
$0.071\pm0.015$\\

\rowcolor{cbaseline}
5 & Baseline &
$0.089\pm0.026$ &
$0.124\pm0.021$ &
$0.135\pm0.018$ &
$0.145\pm0.015$\\

\rowcolor{ctradeoff}
6 & Risk $\uparrow$, Cost $\downarrow$ &
$0.005\pm0.004$ &
$0.000\pm0.001$ &
$0.000\pm0.000$ &
$0.000\pm0.000$\\

\rowcolor{ccostdown}
7 & Cost $\downarrow$ &
$0.036\pm0.010$ &
$0.022\pm0.006$ &
$0.017\pm0.005$ &
$0.011\pm0.003$\\

\rowcolor{criskdown}
8 & Risk $\downarrow$ &
$0.000\pm0.000$ &
$0.000\pm0.000$ &
$0.000\pm0.000$ &
$0.000\pm0.000$\\

\rowcolor{ctradeoff}
9 & Risk $\downarrow$, Cost $\uparrow$ &
$0.000\pm0.001$ &
$0.000\pm0.000$ &
$0.000\pm0.000$ &
$0.000\pm0.000$\\

\rowcolor{cbaseline}
10 & Baseline &
$0.050\pm0.004$ &
$0.052\pm0.002$ &
$0.053\pm0.001$ &
$0.054\pm0.001$\\

\rowcolor{criskdown}
11 & Risk $\downarrow$&
$0.004\pm0.004$ &
$0.002\pm0.002$ &
$0.001\pm0.001$ &
$0.000\pm0.001$\\

\rowcolor{cneutral}
12 & Neutral&
$0.004\pm0.003$ &
$0.000\pm0.000$ &
$0.000\pm0.000$ &
$0.000\pm0.000$\\

\rowcolor{ccost}
13 & Cost $\uparrow$ &
$0.026\pm0.008$ &
$0.018\pm0.006$ &
$0.015\pm0.005$ &
$0.011\pm0.004$\\

\rowcolor{crisk}
14 & Risk $\uparrow$&
$0.047\pm0.015$ &
$0.034\pm0.010$ &
$0.030\pm0.009$ &
$0.028\pm0.007$\\

\rowcolor{cbaseline}
15 & Baseline &
$0.055\pm0.013$ &
$0.078\pm0.010$ &
$0.086\pm0.008$ &
$0.093\pm0.007$\\

\bottomrule
\end{tabular}
\end{subtable}

\vspace{3mm}

\begin{subtable}{\textwidth}
\renewcommand{\arraystretch}{1.2}
\centering

\caption{Summary of the fraction of test instances assigned to each high-level audit category obtained by aggregating the corresponding outcome labels from Table~\ref{tab:test_instances_outcomes_BIGSCIENCE_T0_BANK}. Specifically, Baseline aggregates outcomes $\{0,5,10,15\}$, Risk increases aggregates $\{2,14\}$, Risk decreases aggregates $\{8,11\}$, Cost increases aggregates $\{1,13\}$, Cost decreases aggregates $\{4,7\}$, Neutral aggregates $\{3,12\}$, and Mixed aggregates $\{6,9\}$.}
\label{tab:audit_category_summary_BIGSCIENCE_T0_BANK}

\small
\begin{tabular}{lcccc}
\toprule
\textbf{Category} &
$\mathbf{M=1}$ &
$\mathbf{M=3}$ &
$\mathbf{M=5}$ &
$\mathbf{M=10}$\\
\midrule
\cellcolor{cbaseline}\textbf{Baseline}
& 0.756 & 0.833 & 0.857 & 0.878\\

\cellcolor{crisk}\textbf{Risk increases}
& 0.047 & 0.034 & 0.030 & 0.028\\

\cellcolor{criskdown}\textbf{Risk decreases}
& 0.004 & 0.002 & 0.001 & 0.000\\

\cellcolor{ccost}\textbf{Cost increases}
& 0.050 & 0.027 & 0.020 & 0.013\\

\cellcolor{ccostdown}\textbf{Cost decreases}
& 0.133 & 0.103 & 0.092 & 0.082\\

\cellcolor{cneutral}\textbf{Neutral}
& 0.006 & 0.000 & 0.000 & 0.000\\

\cellcolor{ctradeoff}\textbf{Mixed}
& 0.005 & 0.000 & 0.000 & 0.000\\
\bottomrule
\end{tabular}

\end{subtable}

\end{table*}

\begin{table*}[t]
\centering

\caption{Audit outcomes and aggregated audit outcomes for a binary classification task on the \textbf{Adult} dataset using ensembles of independently fine-tuned \textbf{BIG-SCIENCE T0} models.}
\label{tab:test_instances_BIGSCIENCE_T0_ADULT}

\begin{subtable}{\textwidth}
\centering
\renewcommand{\arraystretch}{1.5}
\setlength{\tabcolsep}{6pt}

\caption{Table reporting the fraction of test instances assigned to each outcome label $o\in\{0,\ldots,15\}$ from Table~\ref{tab: combined audit outcomes}. The baseline model is the ensemble of the Rashomon \largeset{}~$\retailset_{\Nlib}$, with model error $\modelerrorlabel(x; \retailset_{\Nlib})$, and audit decision $\riskassess(x ; \retailset_{\Nlib})$, while the candidate models are ensembles of $q=\{1,\dots,\nsampretail\}$ random draws of the Rashomon \smallset{} set $\retailset_{\Nretail,q}$, with model error $\modelerrorlabel(x ; \retailset_{\Nretail,q})$ and audit decision $\riskassess(x ; \retailset_{\Nretail,q})$. The tuple $(\modelerrorlabel(x; \retailset_{\Nlib}),\riskassess(x; \retailset_{\Nlib}),\modelerrorlabel(x; \retailset_{\Nretail,q}),\riskassess(x; \retailset_{\Nretail,q}))$ decides the combined audit outcomes, with label $o\in\{0,\ldots,15\}$. The table reports the mean $\pm$ standard deviation of the fraction of test instances assigned to each outcome label across $q=\{1,\dots,\nsampretail\}$ random retail ensemble draws for different values of $\Nretail$. A desirable auditing strategy is one that minimizes excess audit costs and risks by increasing the Baseline category.}
\label{tab:test_instances_outcomes_BIGSCIENCE_T0_ADULT}

\begin{tabular}{llcccc}
\toprule
\textbf{o} &
\textbf{Meaning} &
$\mathbf{M=1}$ &
$\mathbf{M=3}$ &
$\mathbf{M=5}$ &
$\mathbf{M=10}$\\
\midrule

\rowcolor{cbaseline}
0 & Baseline &
$0.580\pm0.011$ &
$0.590\pm0.006$ &
$0.595\pm0.004$ &
$0.598\pm0.002$\\

\rowcolor{ccost}
1 & Cost $\uparrow$ &
$0.019\pm0.010$ &
$0.010\pm0.005$ &
$0.005\pm0.004$ &
$0.002\pm0.002$\\

\rowcolor{crisk}
2 & Risk $\uparrow$ &
$0.000\pm0.000$ &
$0.000\pm0.000$ &
$0.000\pm0.000$ &
$0.000\pm0.000$\\

\rowcolor{cneutral}
3 & Neutral&
$0.001\pm0.002$ &
$0.000\pm0.000$ &
$0.000\pm0.000$ &
$0.000\pm0.000$\\

\rowcolor{ccostdown}
4 & Cost $\downarrow$ &
$0.082\pm0.017$ &
$0.062\pm0.012$ &
$0.058\pm0.010$ &
$0.054\pm0.008$\\

\rowcolor{cbaseline}
5 & Baseline &
$0.108\pm0.019$ &
$0.144\pm0.014$ &
$0.154\pm0.012$ &
$0.163\pm0.009$\\

\rowcolor{ctradeoff}
6 & Risk $\uparrow$, Cost $\downarrow$ &
$0.005\pm0.004$ &
$0.000\pm0.001$ &
$0.000\pm0.000$ &
$0.000\pm0.000$\\

\rowcolor{ccostdown}
7 & Cost $\downarrow$ &
$0.035\pm0.008$ &
$0.023\pm0.005$ &
$0.018\pm0.004$ &
$0.013\pm0.003$\\

\rowcolor{criskdown}
8 & Risk $\downarrow$ &
$0.000\pm0.000$ &
$0.000\pm0.000$ &
$0.000\pm0.000$ &
$0.000\pm0.000$\\

\rowcolor{ctradeoff}
9 & Risk $\downarrow$, Cost $\uparrow$ &
$0.000\pm0.000$ &
$0.000\pm0.000$ &
$0.000\pm0.000$ &
$0.000\pm0.000$\\

\rowcolor{cbaseline}
10 & Baseline &
$0.046\pm0.002$ &
$0.048\pm0.001$ &
$0.048\pm0.001$ &
$0.049\pm0.001$\\

\rowcolor{criskdown}
11 & Risk $\downarrow$&
$0.004\pm0.002$ &
$0.002\pm0.001$ &
$0.001\pm0.001$ &
$0.001\pm0.001$\\

\rowcolor{cneutral}
12 & Neutral&
$0.004\pm0.003$ &
$0.000\pm0.000$ &
$0.000\pm0.000$ &
$0.000\pm0.000$\\

\rowcolor{ccost}
13 & Cost $\uparrow$ &
$0.026\pm0.005$ &
$0.020\pm0.004$ &
$0.016\pm0.004$ &
$0.012\pm0.003$\\

\rowcolor{crisk}
14 & Risk $\uparrow$&
$0.035\pm0.009$ &
$0.023\pm0.006$ &
$0.020\pm0.005$ &
$0.017\pm0.004$\\

\rowcolor{cbaseline}
15 & Baseline &
$0.056\pm0.010$ &
$0.078\pm0.007$ &
$0.084\pm0.007$ &
$0.091\pm0.005$\\

\bottomrule
\end{tabular}
\end{subtable}

\vspace{3mm}

\begin{subtable}{\textwidth}
\renewcommand{\arraystretch}{1.2}
\centering

\caption{Summary of the fraction of test instances assigned to each high-level audit category obtained by aggregating the corresponding outcome labels from Table~\ref{tab:test_instances_outcomes_BIGSCIENCE_T0_ADULT}. Specifically, Baseline aggregates outcomes $\{0,5,10,15\}$, Risk increases aggregates $\{2,14\}$, Risk decreases aggregates $\{8,11\}$, Cost increases aggregates $\{1,13\}$, Cost decreases aggregates $\{4,7\}$, Neutral aggregates $\{3,12\}$, and Mixed aggregates $\{6,9\}$.}
\label{tab:audit_category_summary_BIGSCIENCE_T0_ADULT}

\small
\begin{tabular}{lcccc}
\toprule
\textbf{Category} &
$\mathbf{M=1}$ &
$\mathbf{M=3}$ &
$\mathbf{M=5}$ &
$\mathbf{M=10}$\\
\midrule
\cellcolor{cbaseline}\textbf{Baseline}
& 0.790
& 0.860
& 0.881
& 0.901\\

\cellcolor{crisk}\textbf{Risk increases}
& 0.035
& 0.023
& 0.020
& 0.017\\

\cellcolor{criskdown}\textbf{Risk decreases}
& 0.004
& 0.002
& 0.001
& 0.001\\

\cellcolor{ccost}\textbf{Cost increases}
& 0.045
& 0.030
& 0.021
& 0.014\\

\cellcolor{ccostdown}\textbf{Cost decreases}
& 0.117
& 0.085
& 0.076
& 0.067\\

\cellcolor{cneutral}\textbf{Neutral}
& 0.005
& 0.000
& 0.000
& 0.000\\

\cellcolor{ctradeoff}\textbf{Mixed}
& 0.005
& 0.000
& 0.000
& 0.000\\
\bottomrule
\end{tabular}

\end{subtable}

\end{table*}

\subsubsection{Visualizing relation between proposed consistency measure and different multiplicity metrics}\label{sec: MM vs cons scatter plot}

Section~\ref{sec: Comp consistency multiplicity} showed that the proposed consistency measure exhibits strong correlations with the multiplicity metrics computed using the Rashomon warehouse. We also looked at the behavior of $\hammancons(x)$ proposed by \citet{hamman2024quantifying} and compared the behavior of this measure and our proposed consistency measure with respect to one multiplicity metric, pairwise disagreement, over randomly sampled test instances (Figure \ref{fig:bert_mrpc_hammancomp}). Results were reported for a paraphrase detection task on the MRPC dataset using an ensemble of independently fine-tuned BERT models. Figure~\ref{fig:bert_mrpc_all_metrics} and Figure~\ref{fig:bert_mrpc_all_metrics_hamman} provide additional visualizations, for other multiplicity metrics. In Figure~\ref{fig:bert_mrpc_all_metrics}, we observe the same trend: the proposed consistency measure shows a strong negative correlation with all multiplicity metrics, which increases with retail ensemble size.  As noted before, our measure assigns only high consistency scores to test instances with low multiplicity. On the other hand,  $\hammancons(x)$ identifies several consistent instances as inconsistent by assigning low scores to some test instances with low multiplicity, as observed in Figure~\ref{fig:bert_mrpc_all_metrics_hamman}.
\begin{figure*}[t]
\centering

\begin{minipage}[c]{0.001\textwidth}
\centering
\rotatebox{90}{\scriptsize Discrepancy}
\end{minipage}
\begin{minipage}[c]{0.95\textwidth}
\centering
\includegraphics[width=0.24\linewidth]{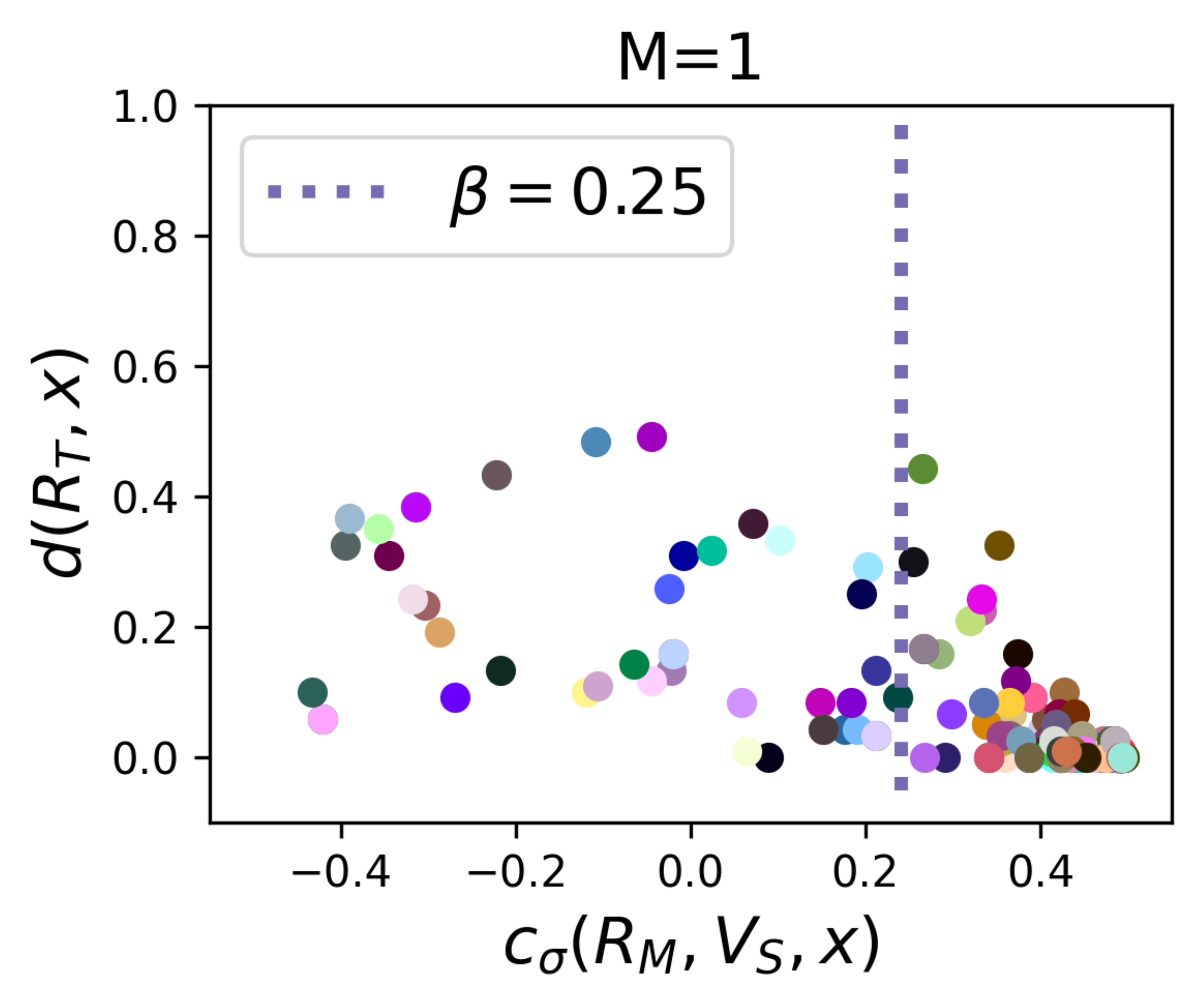}
\includegraphics[width=0.24\linewidth]{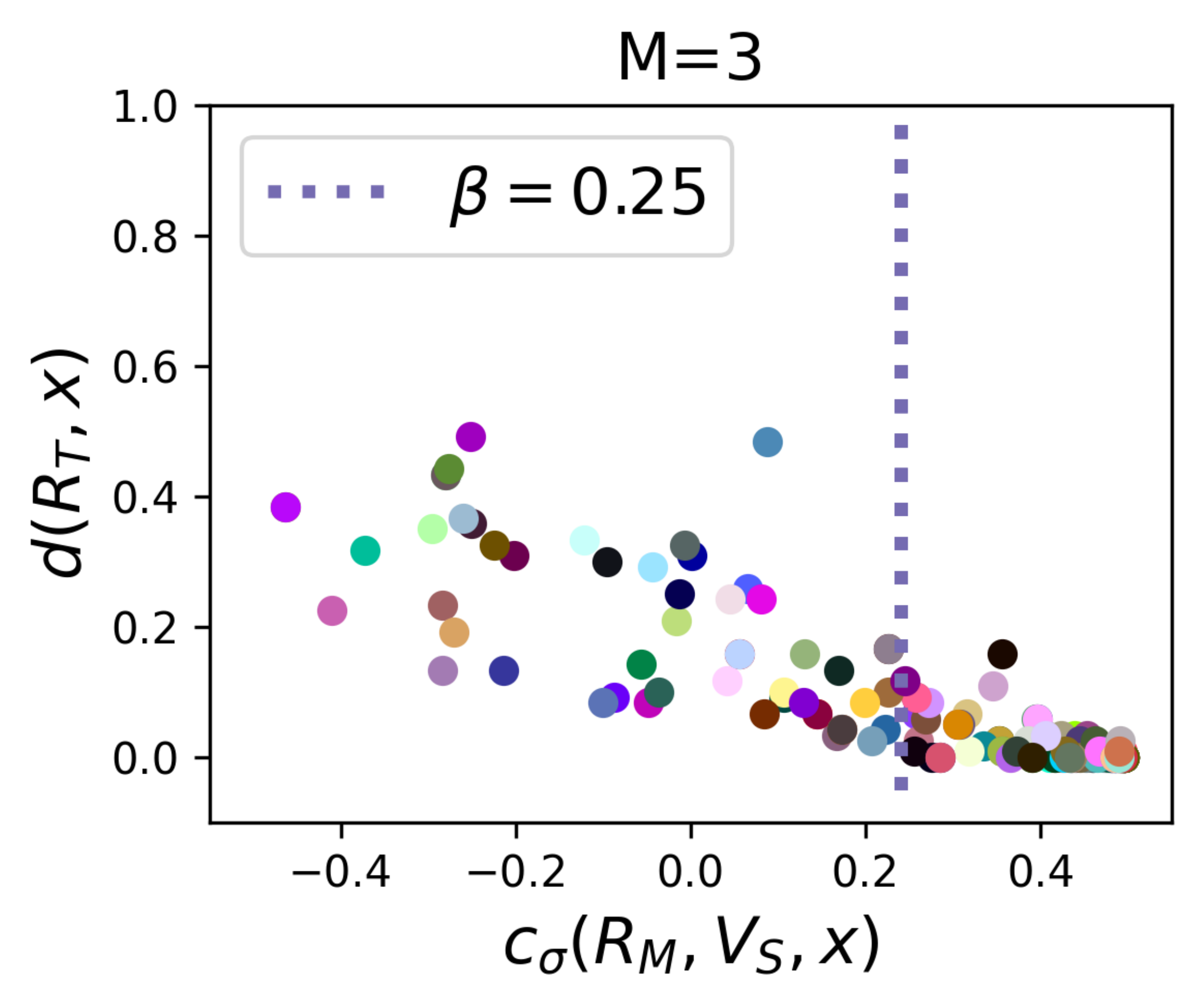}
\includegraphics[width=0.24\linewidth]{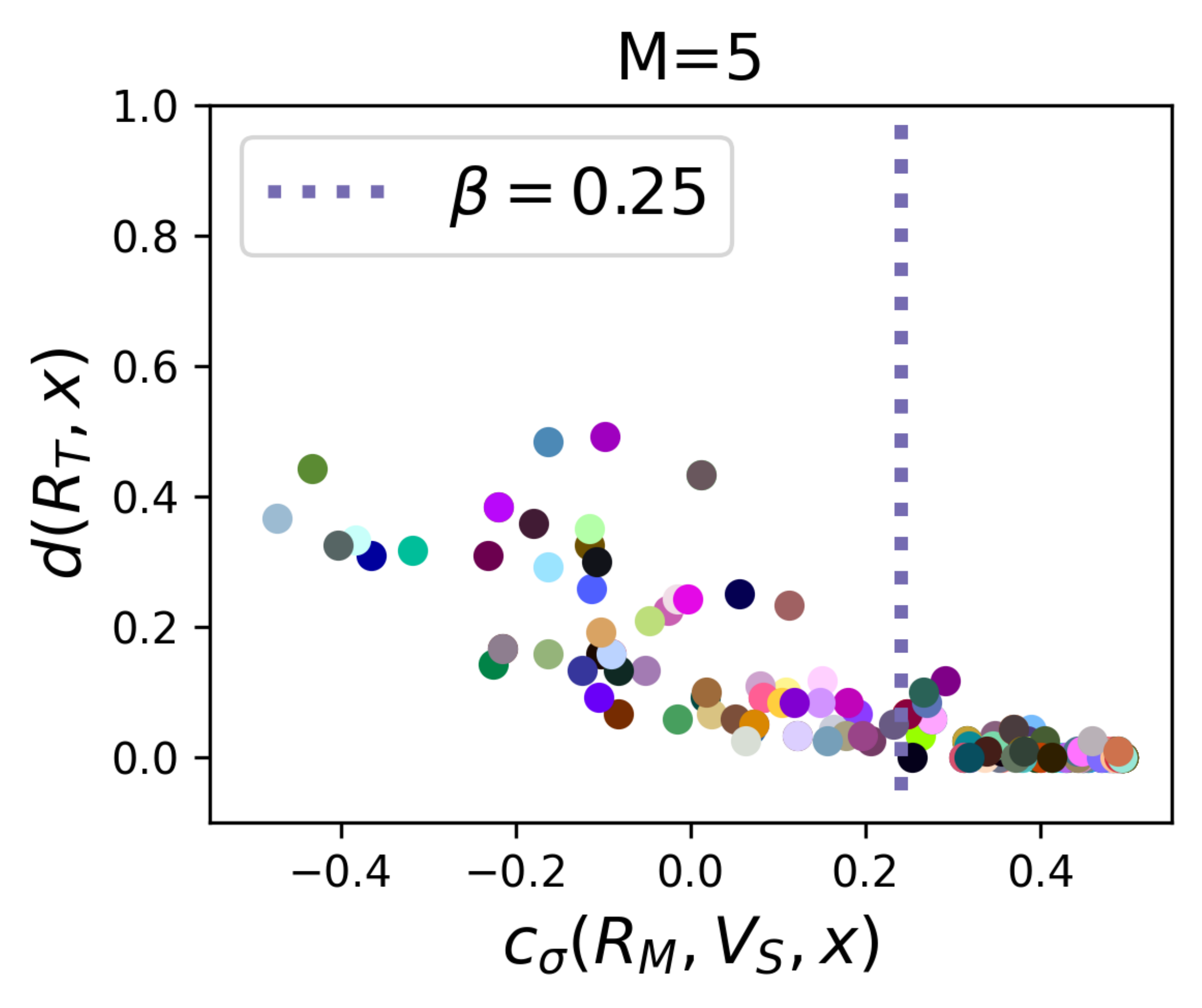}
\includegraphics[width=0.24\linewidth]{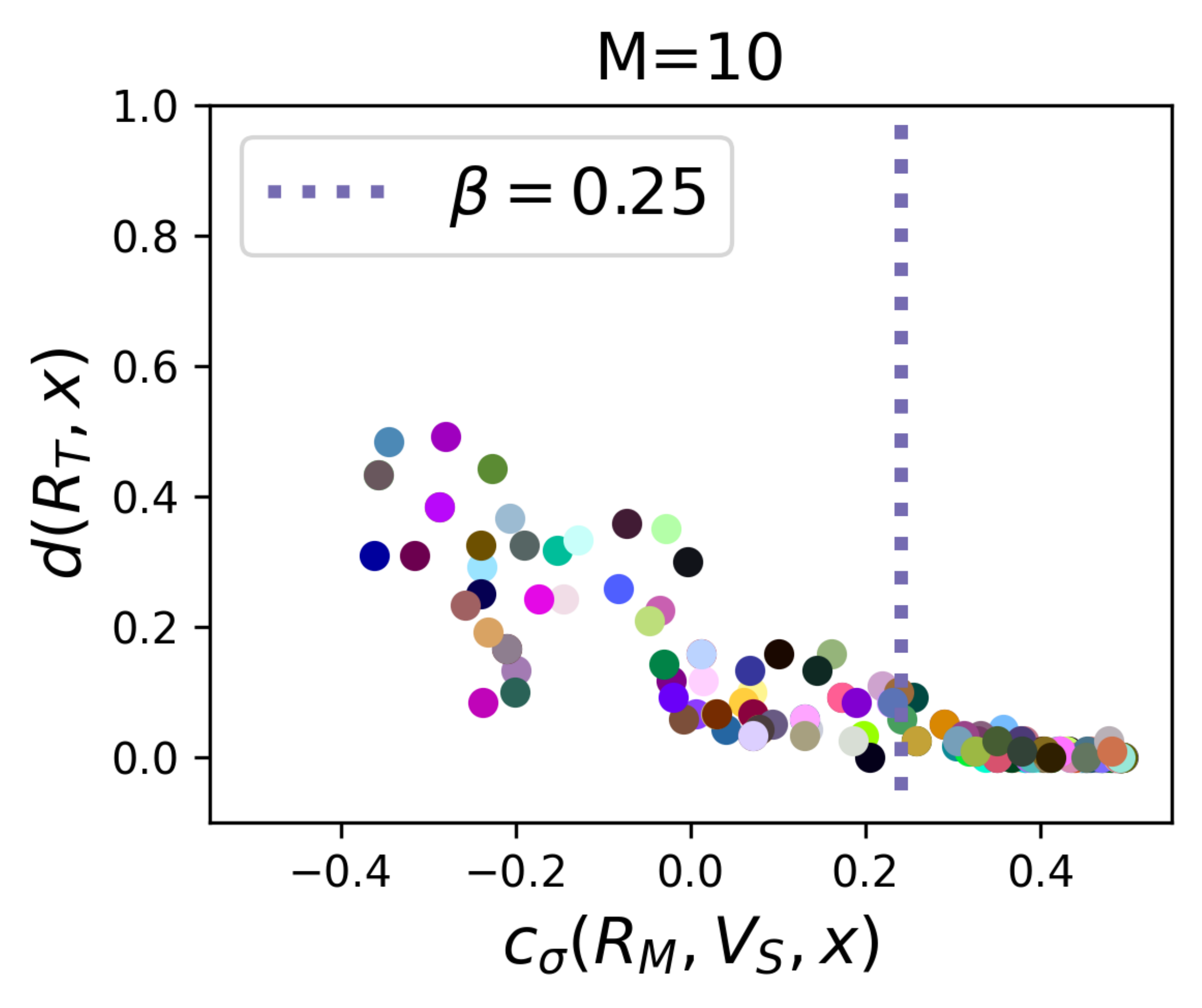}
\end{minipage}

\vspace{1ex}

\begin{minipage}[c]{0.001\textwidth}
\centering
\rotatebox{90}{\scriptsize Local discrepancy}
\end{minipage}
\begin{minipage}[c]{0.95\textwidth}
\centering
\includegraphics[width=0.24\linewidth]{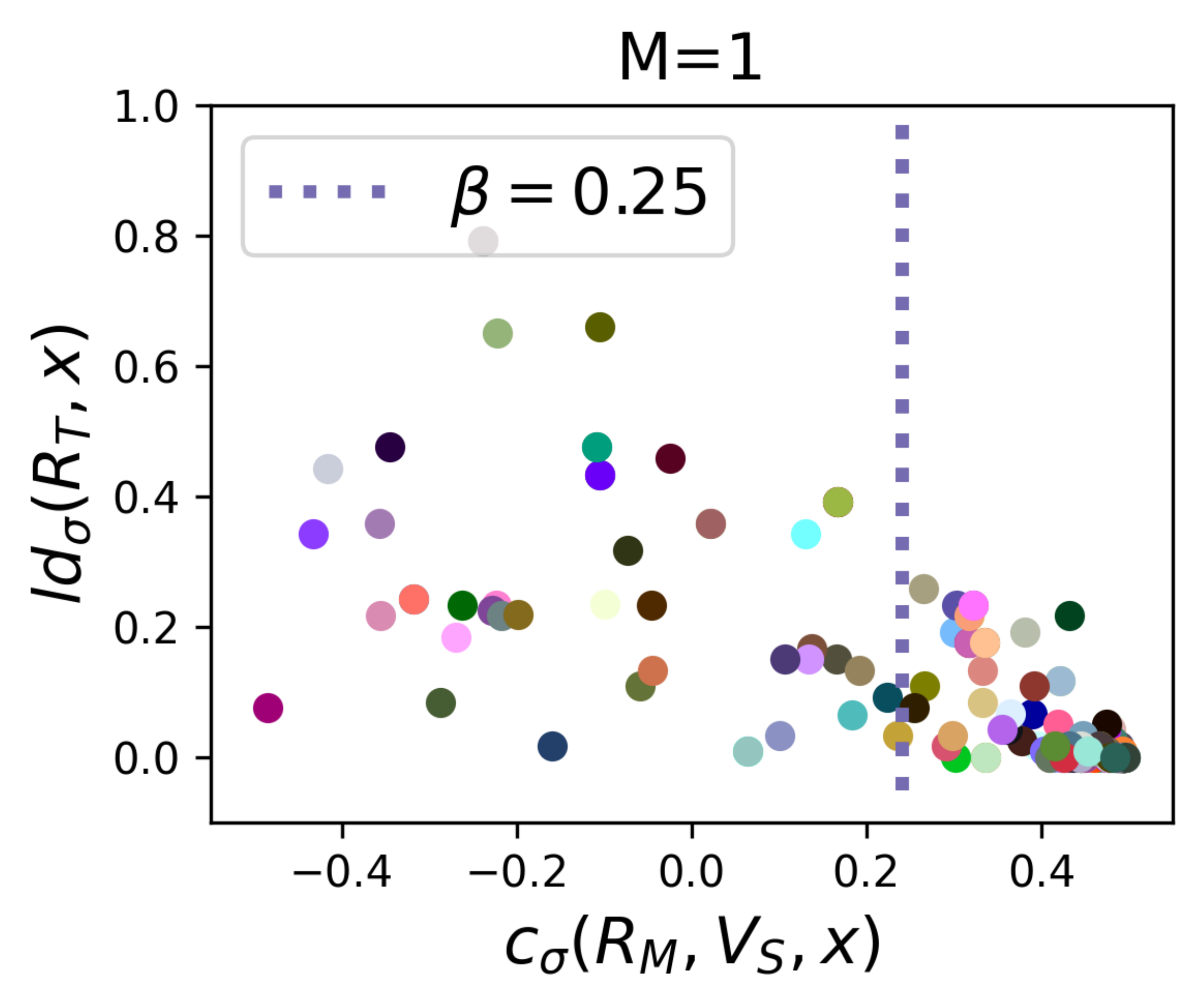}
\includegraphics[width=0.24\linewidth]{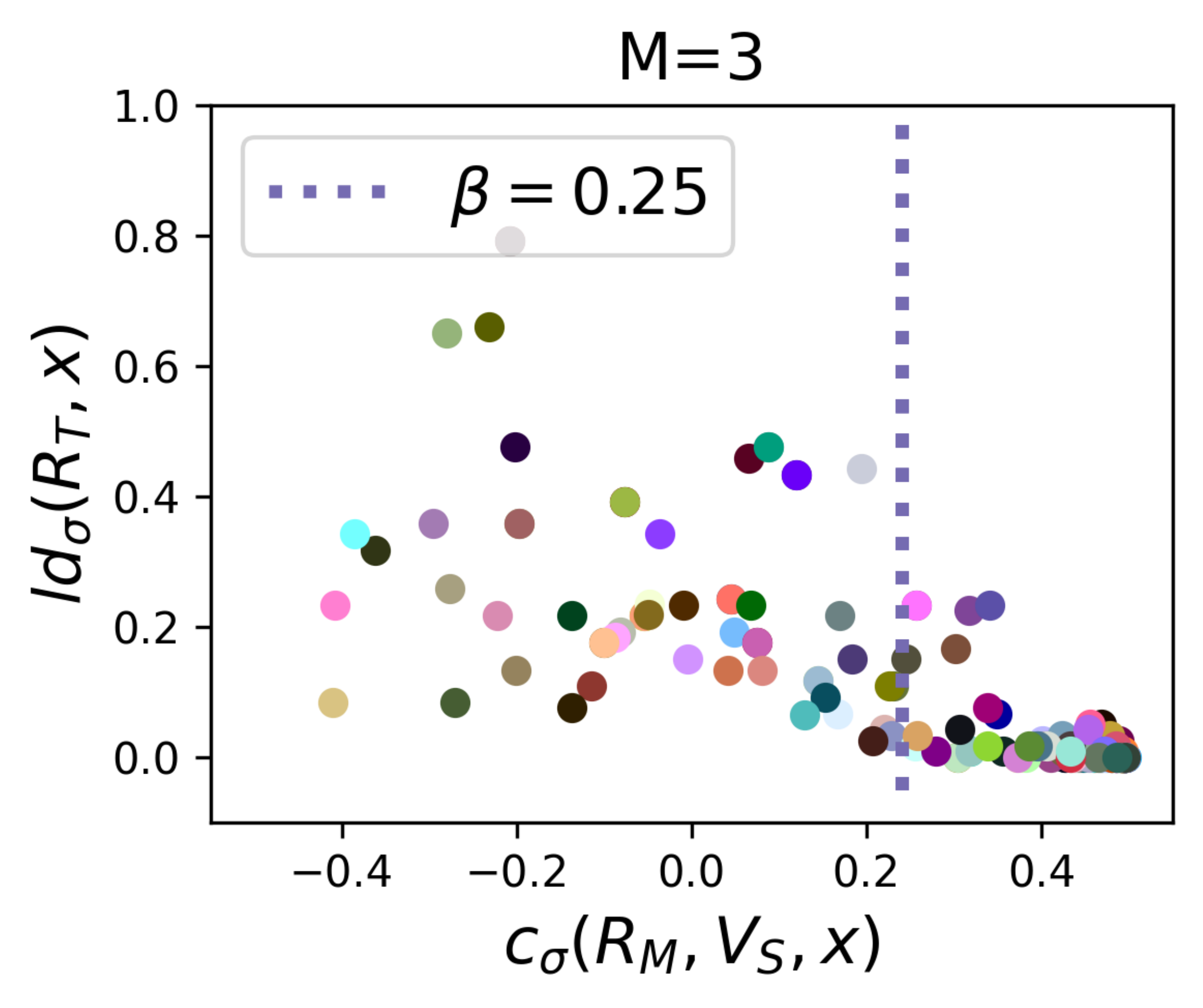}
\includegraphics[width=0.24\linewidth]{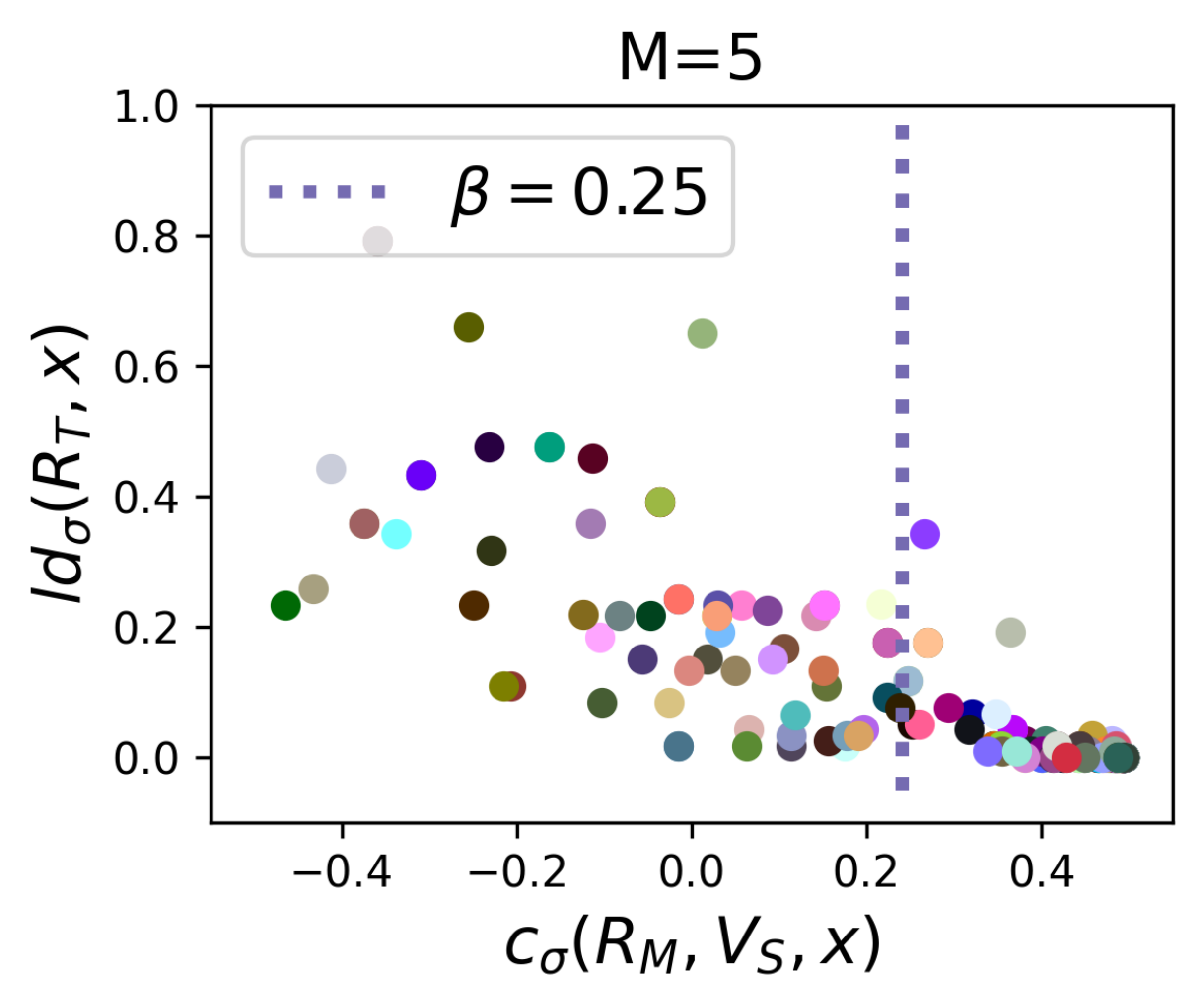}
\includegraphics[width=0.24\linewidth]{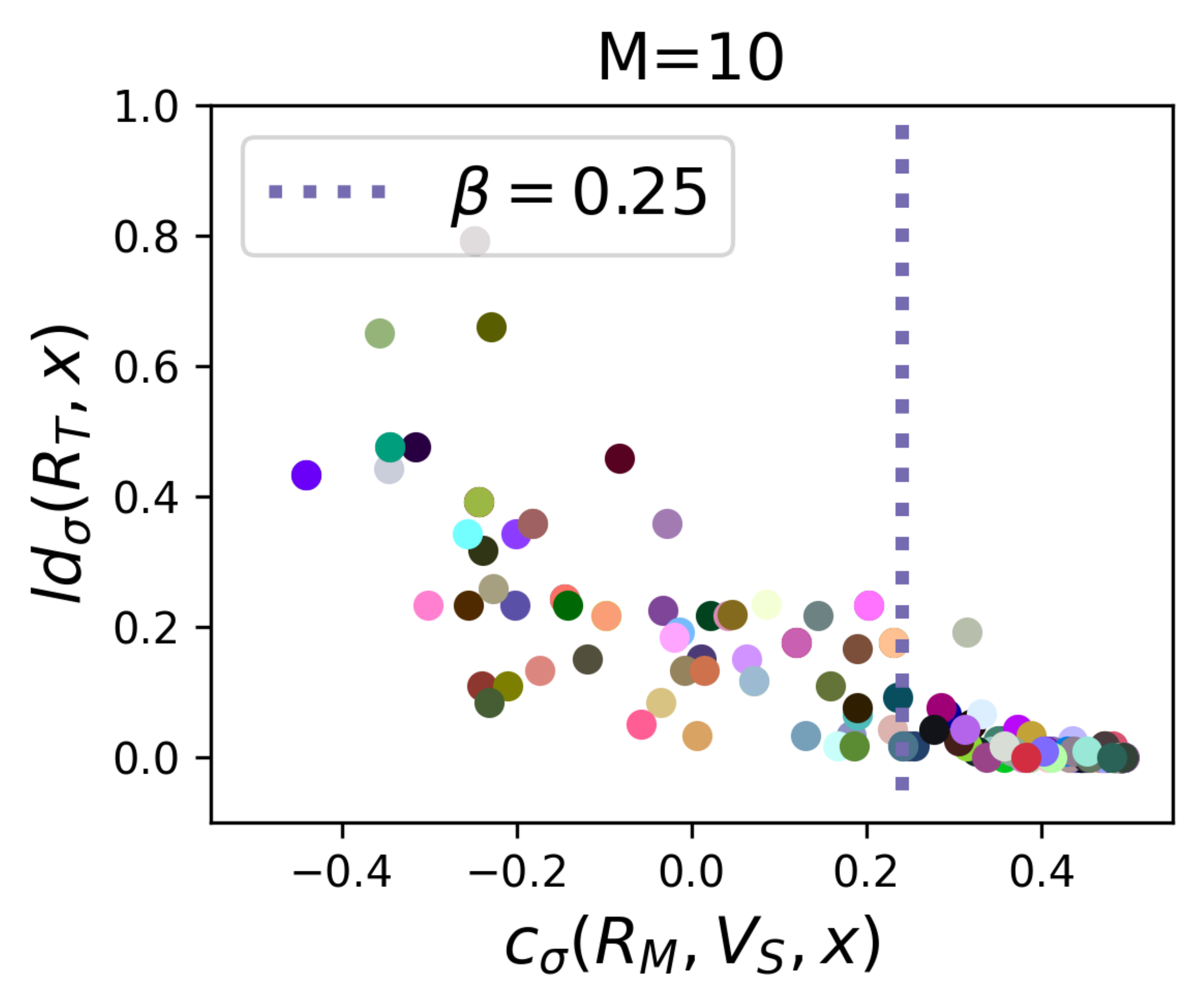}
\end{minipage}

\vspace{1ex}

\begin{minipage}[c]{0.001\textwidth}
\centering
\rotatebox{90}{\scriptsize Prediction range}
\end{minipage}
\begin{minipage}[c]{0.95\textwidth}
\centering
\includegraphics[width=0.24\linewidth]{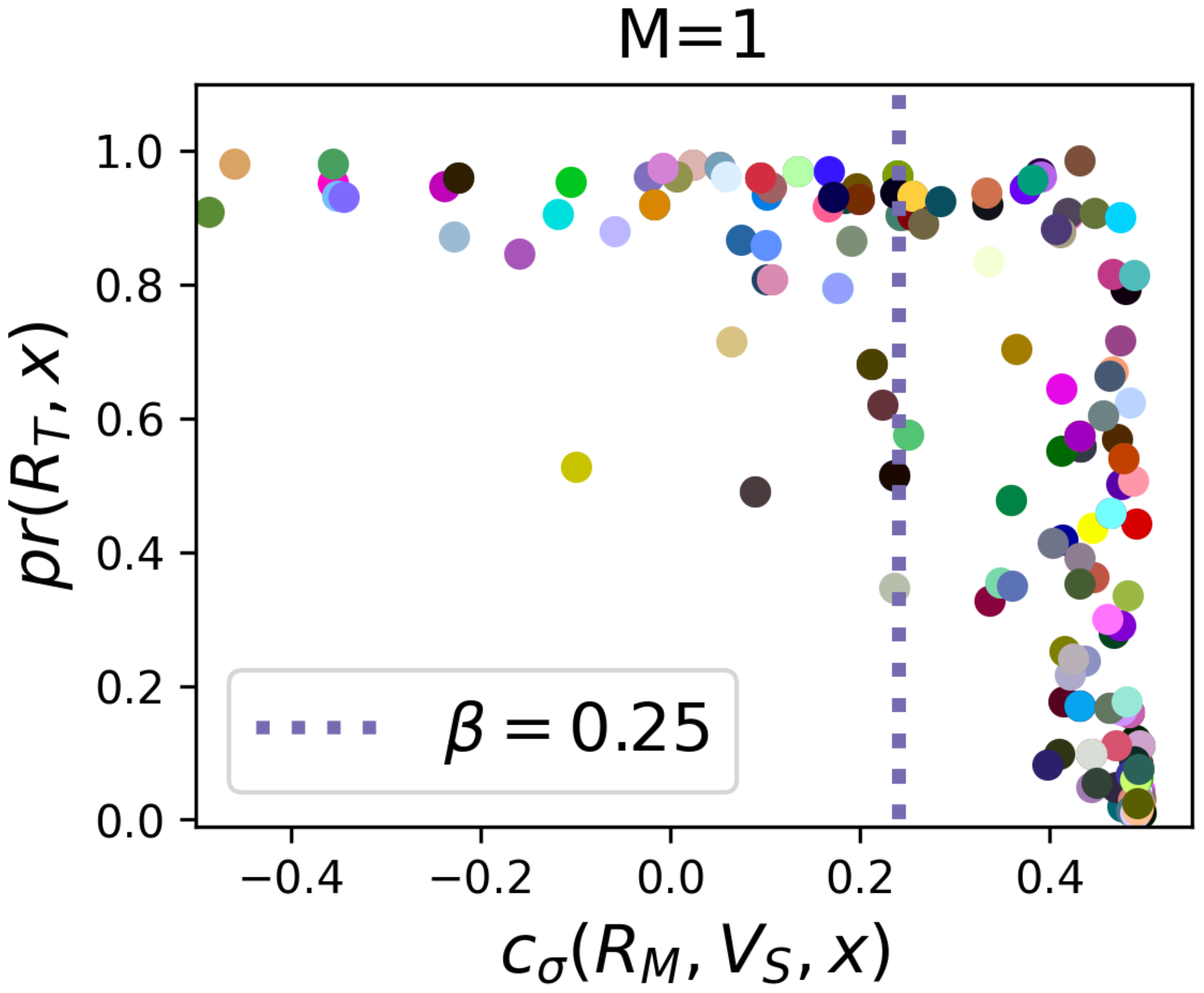}
\includegraphics[width=0.24\linewidth]{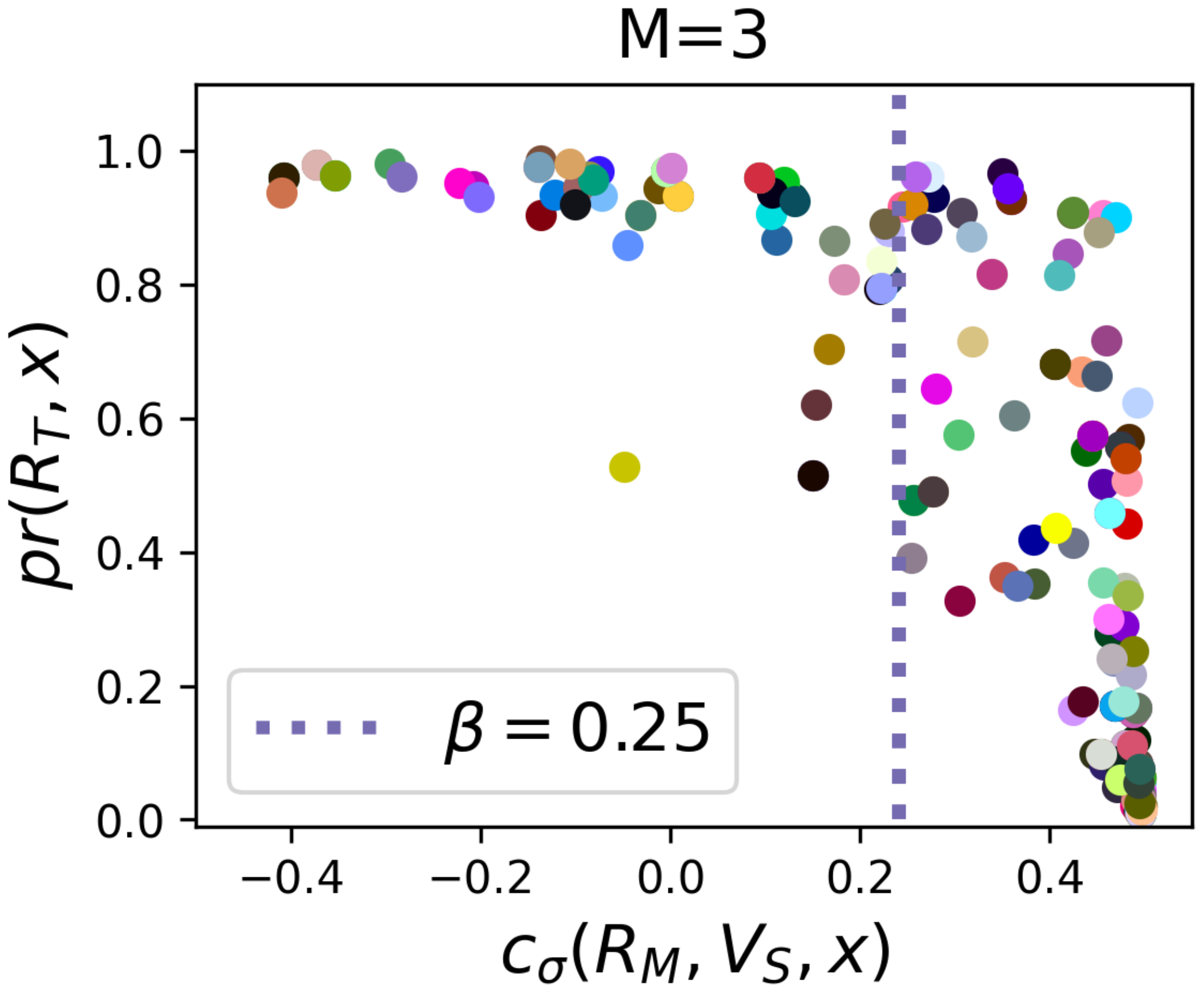}
\includegraphics[width=0.24\linewidth]{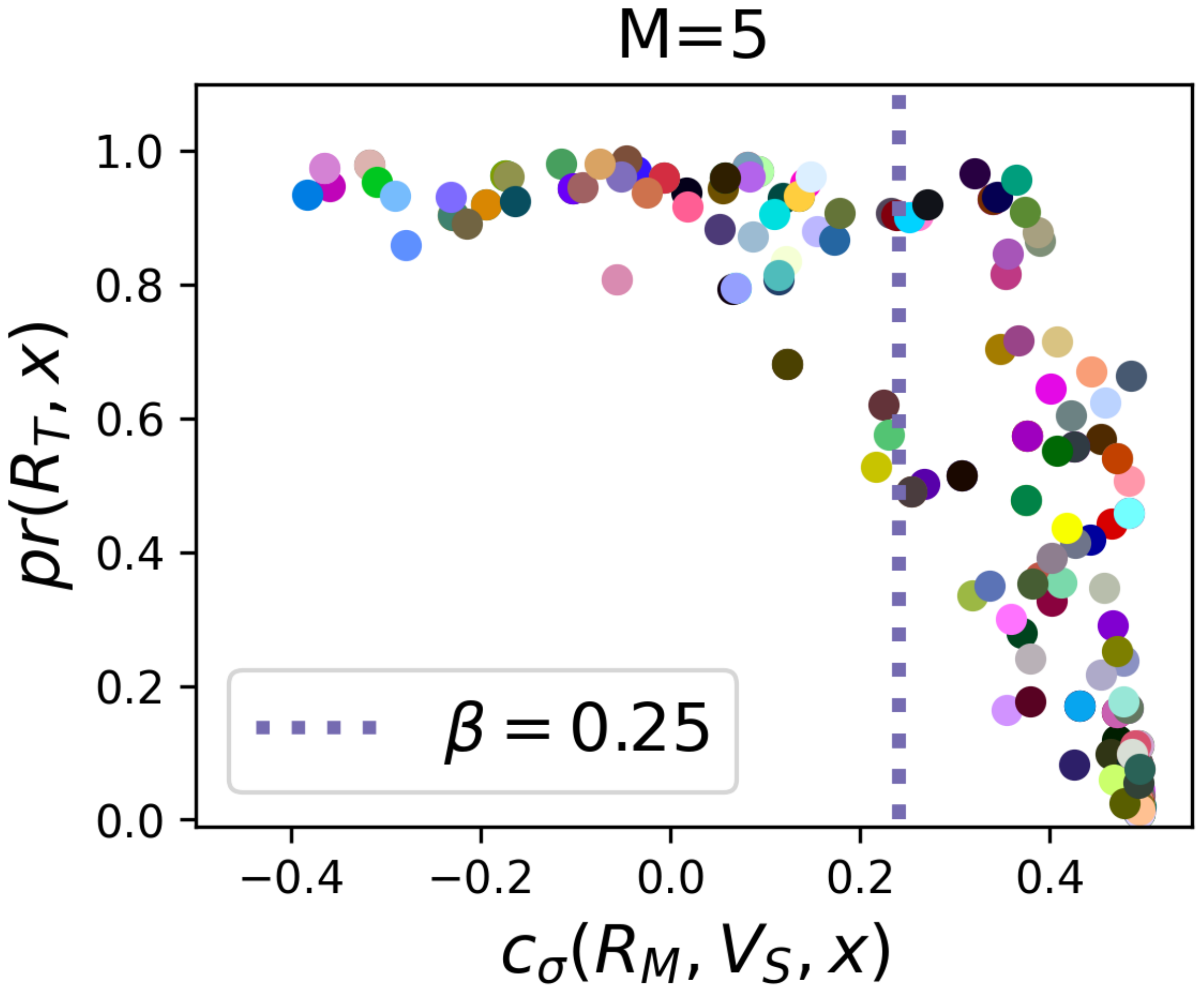}
\includegraphics[width=0.24\linewidth]{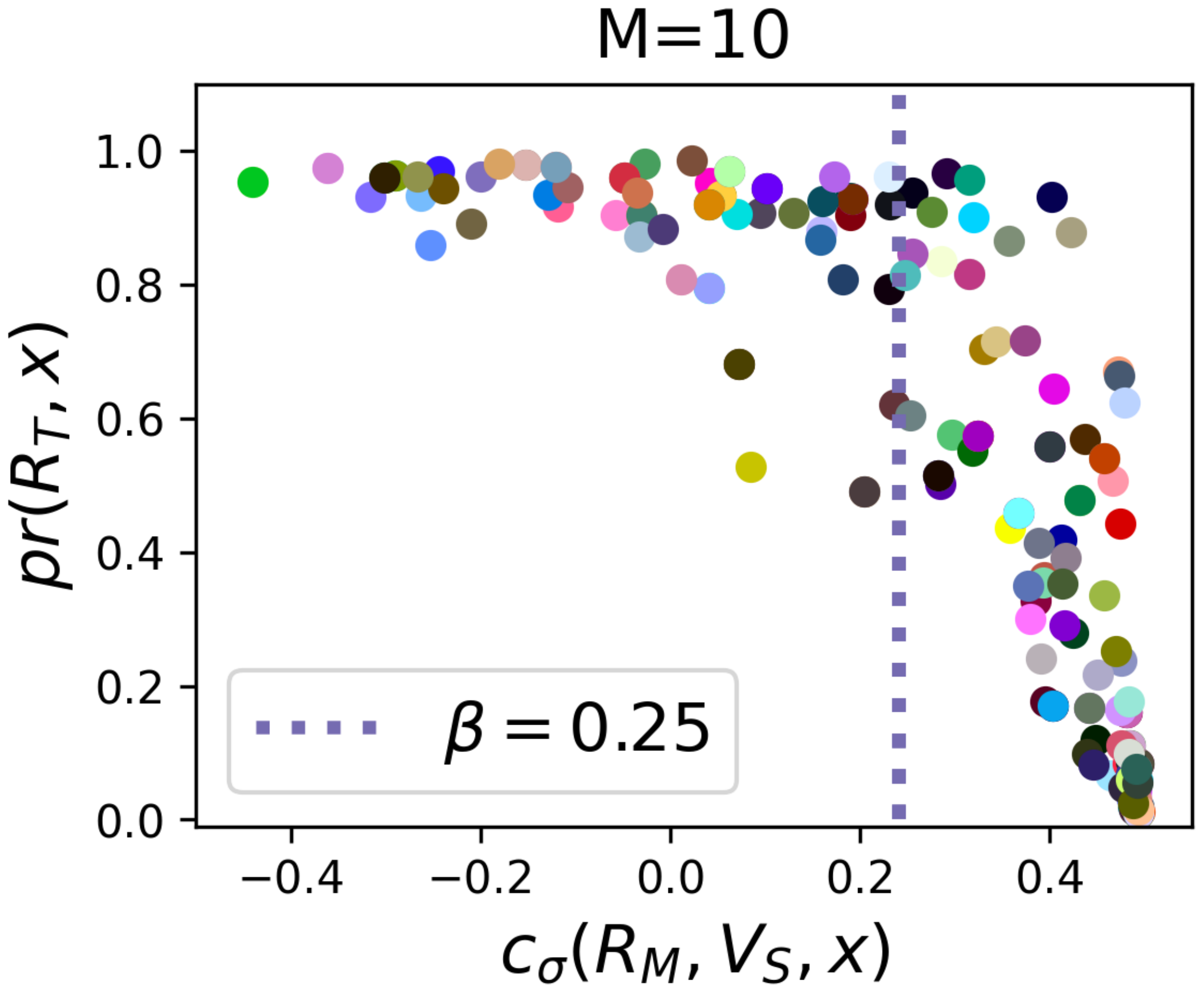}
\end{minipage}

\vspace{1ex}

\begin{minipage}[c]{0.001\textwidth}
\centering
\rotatebox{90}{\scriptsize Prediction variance}
\end{minipage}
\begin{minipage}[c]{0.95\textwidth}
\centering
\includegraphics[width=0.24\linewidth]{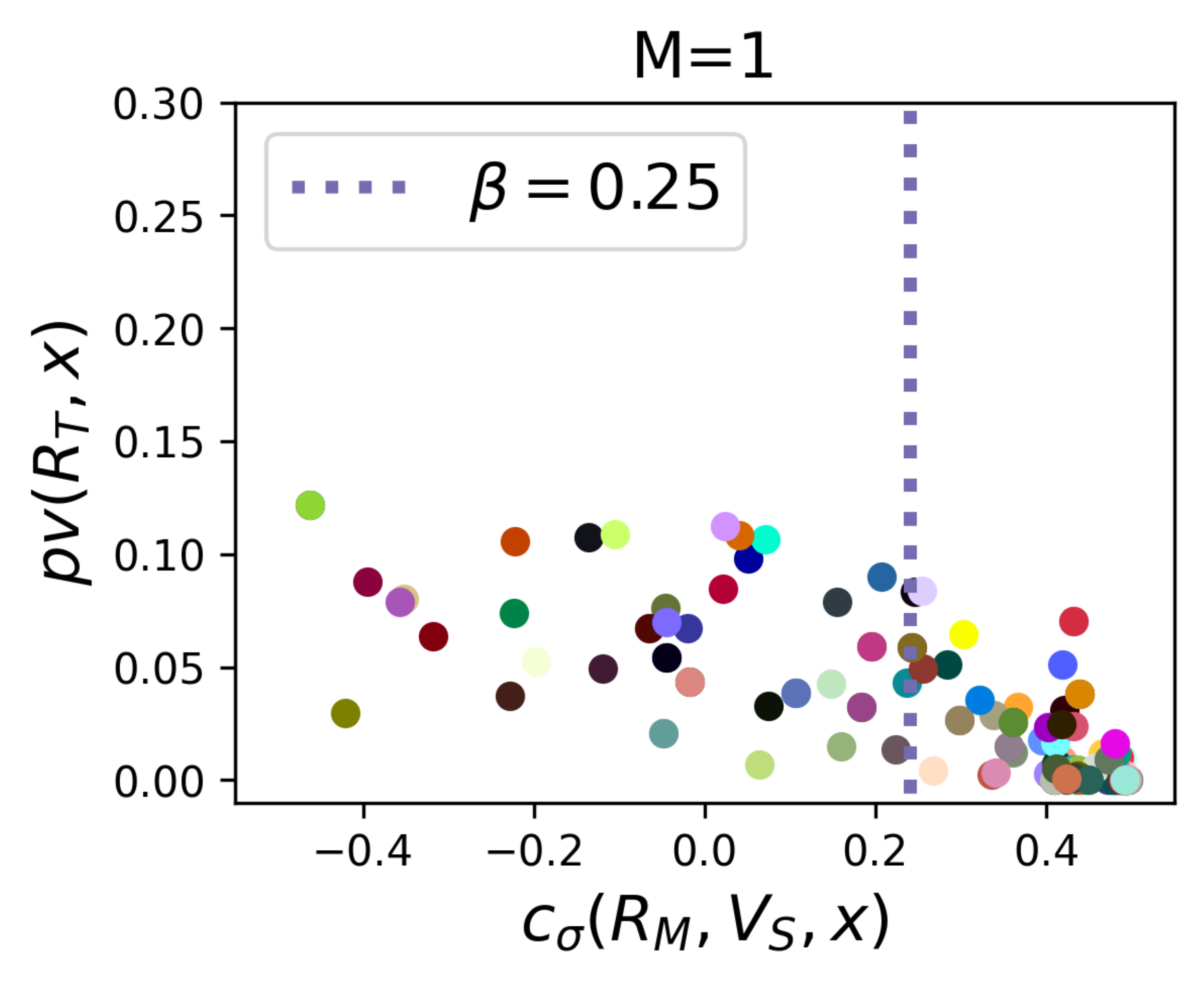}
\includegraphics[width=0.24\linewidth]{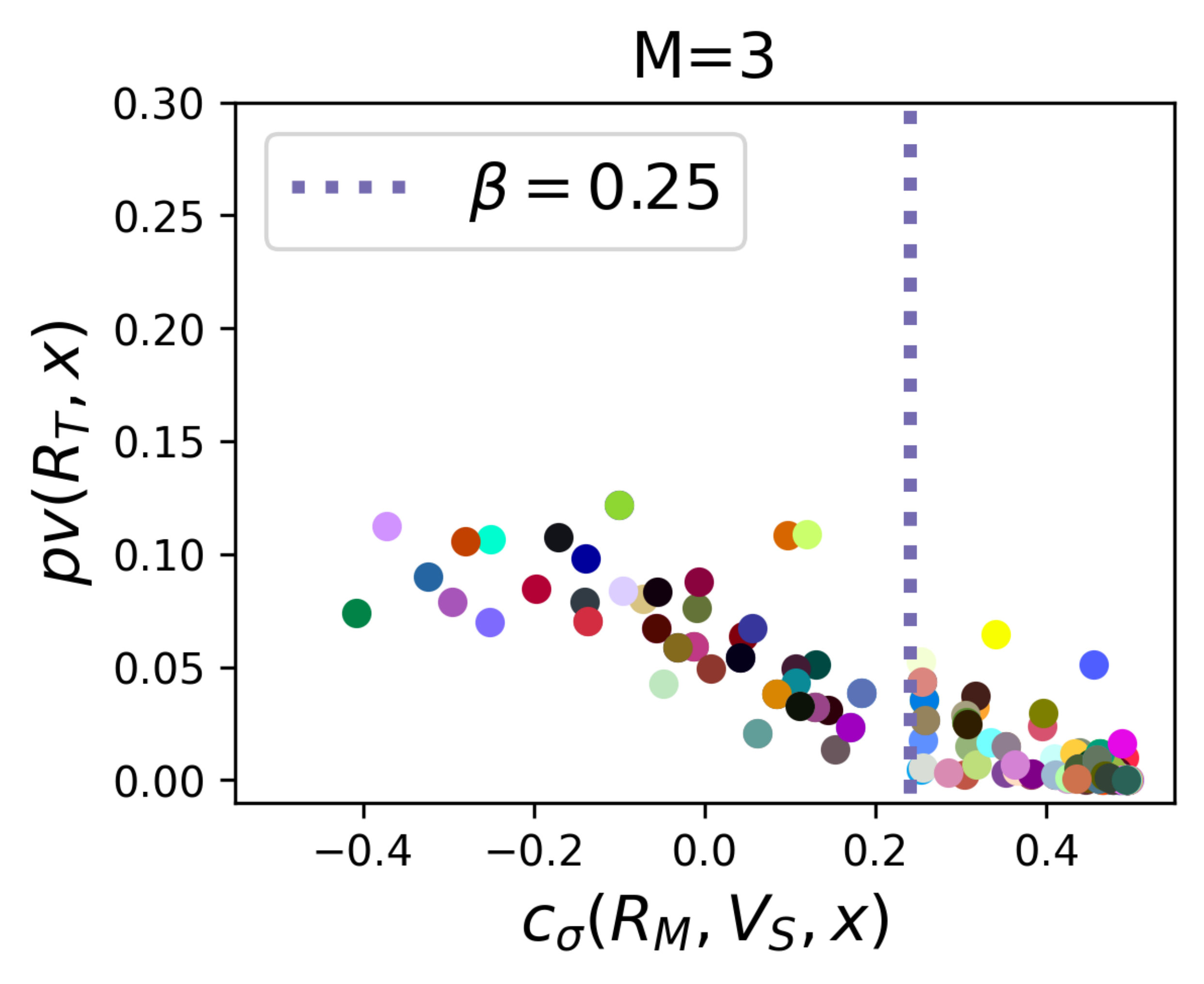}
\includegraphics[width=0.24\linewidth]{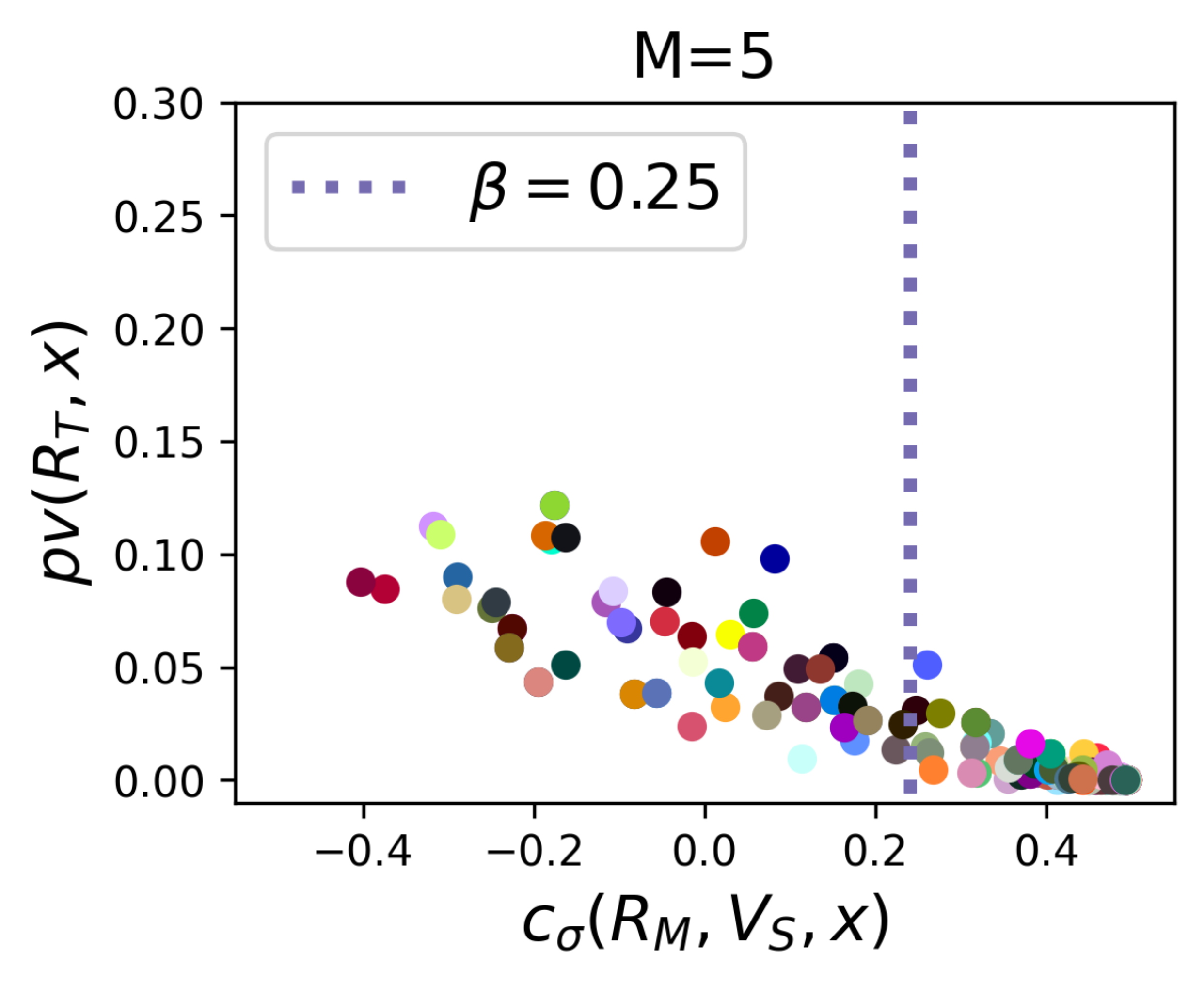}
\includegraphics[width=0.24\linewidth]{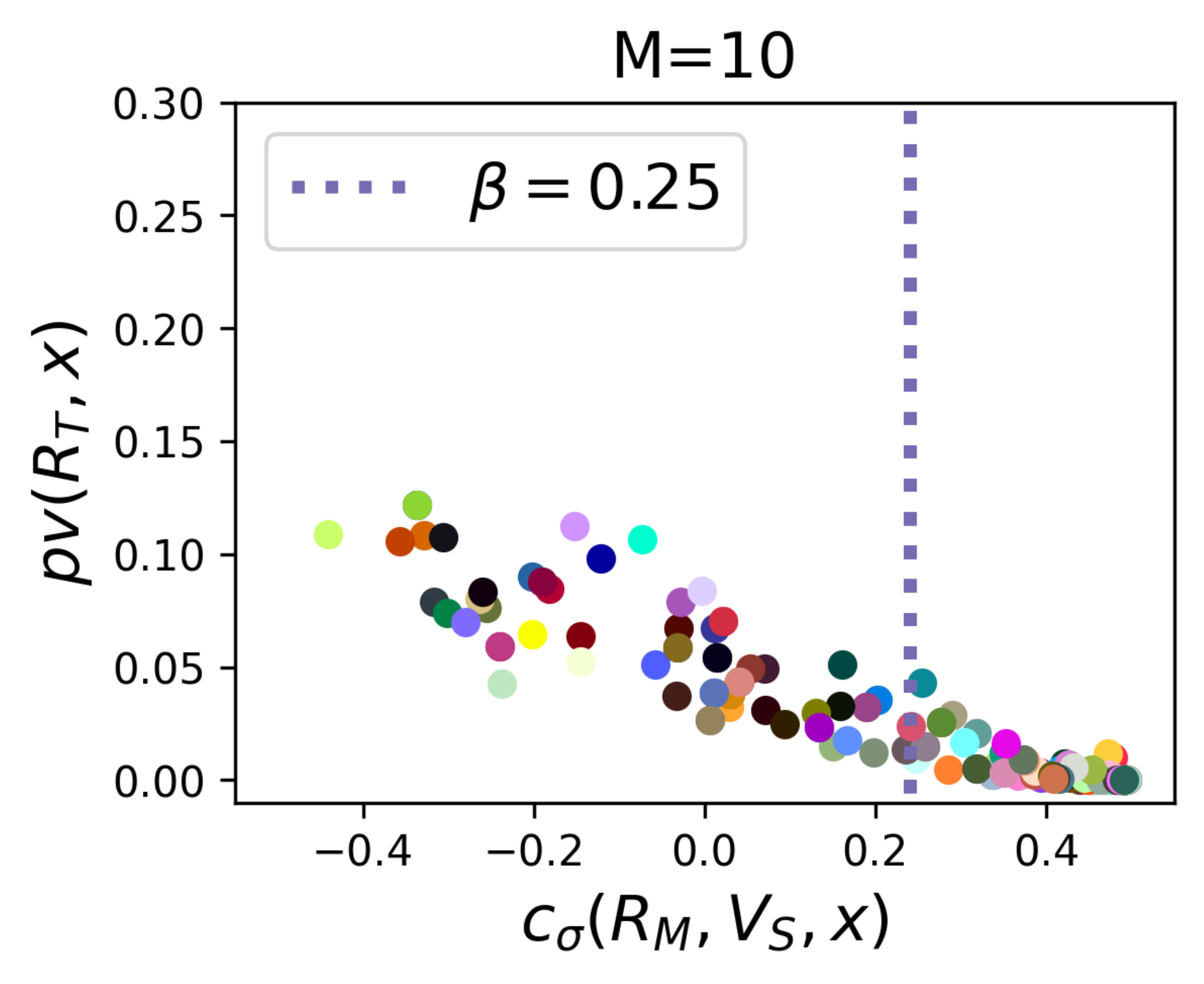}
\end{minipage}

\caption{
Scatter plot showing the relationship between our proposed consistency measure $\consistency$ at a test instance $x$ (x-axis), and four different multiplicity metrics (rows) computed over the Rashomon \largeset{}, and defined in Section~\ref{sec: Multiplicity metrics}, for different retail set sizes (columns). Each colored dot represents one of 200 randomly sampled test instances. Our proposed measure assigns higher scores to test instances with low multiplicity, and this trend becomes clearer as the ensemble size increases. Results are shown for paraphrase detection on the \textbf{MRPC} dataset using an ensemble of independently fine-tuned \textbf{BERT} models. 
}
\label{fig:bert_mrpc_all_metrics}
\end{figure*}

\begin{figure*}[t]

\centering

\begin{minipage}[c]{0.001\textwidth}
\centering
\rotatebox{90}{\scriptsize Discrepancy}
\end{minipage}
\begin{minipage}[c]{0.95\textwidth}
\centering
\includegraphics[width=0.24\linewidth]{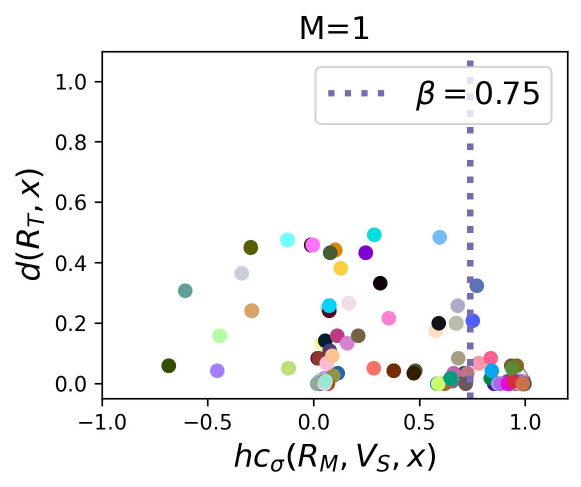}
\includegraphics[width=0.24\linewidth]{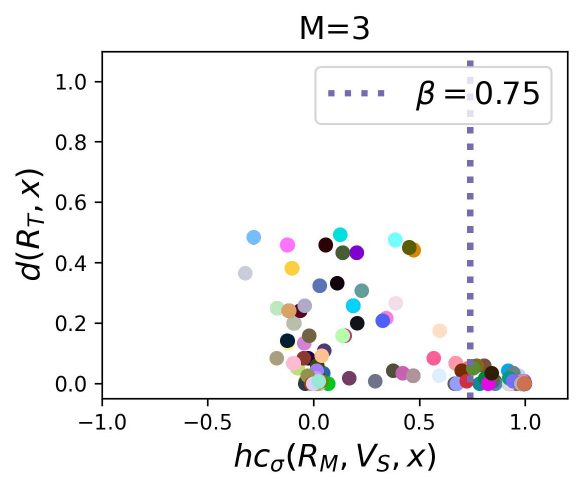}
\includegraphics[width=0.24\linewidth]{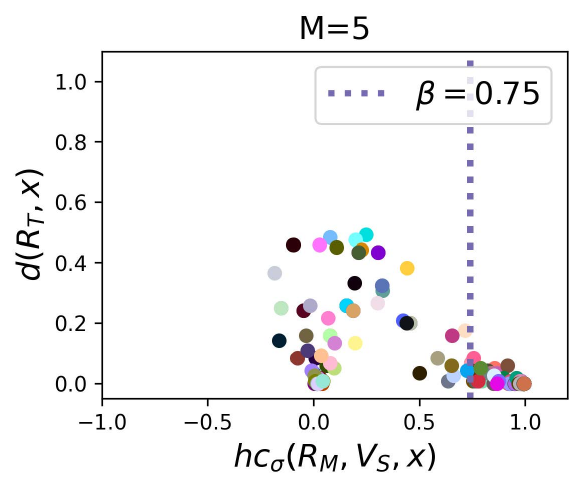}
\includegraphics[width=0.24\linewidth]{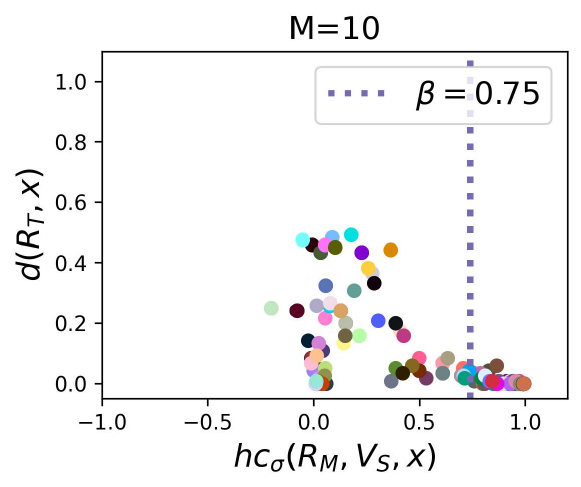}
\end{minipage}

\vspace{1ex}

\begin{minipage}[c]{0.001\textwidth}
\centering
\rotatebox{90}{\scriptsize Local discrepancy}
\end{minipage}
\begin{minipage}[c]{0.95\textwidth}
\centering
\includegraphics[width=0.24\linewidth]{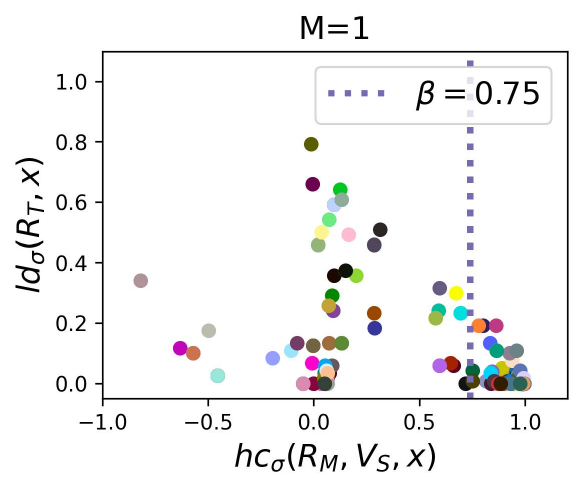}
\includegraphics[width=0.24\linewidth]{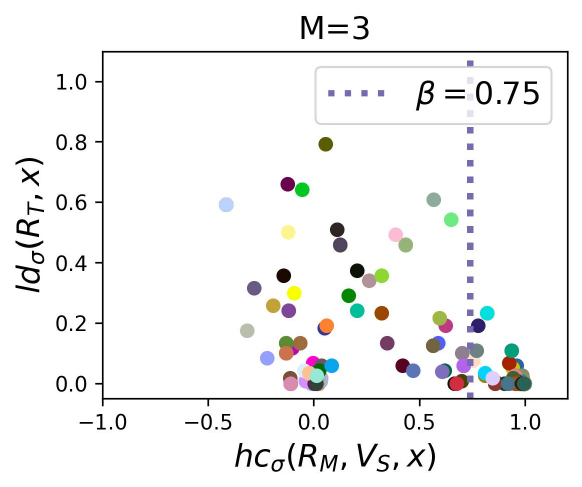}
\includegraphics[width=0.24\linewidth]{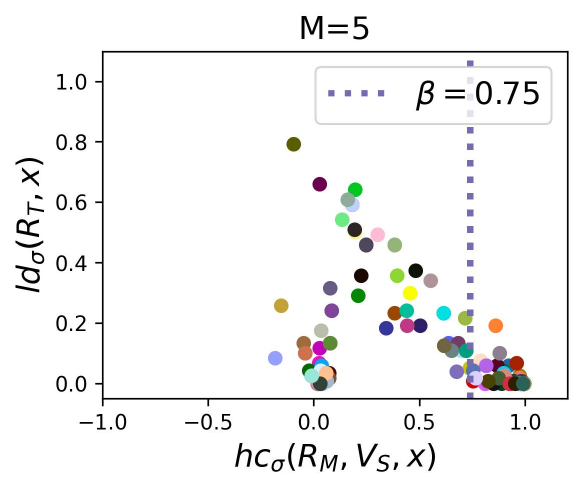}
\includegraphics[width=0.24\linewidth]{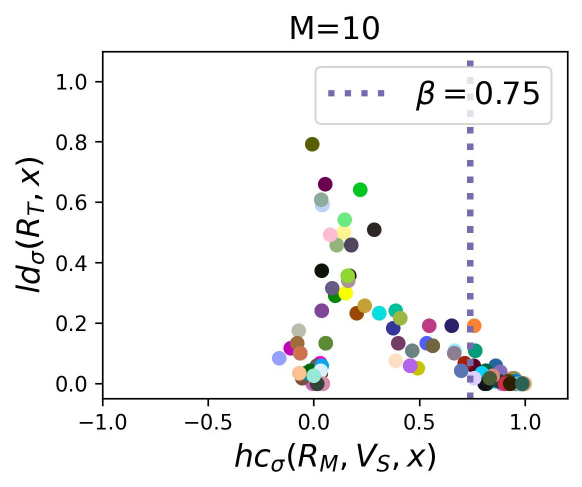}
\end{minipage}

\vspace{1ex}

\begin{minipage}[c]{0.001\textwidth}
\centering
\rotatebox{90}{\scriptsize Prediction range}
\end{minipage}
\begin{minipage}[c]{0.95\textwidth}
\centering
\includegraphics[width=0.24\linewidth]{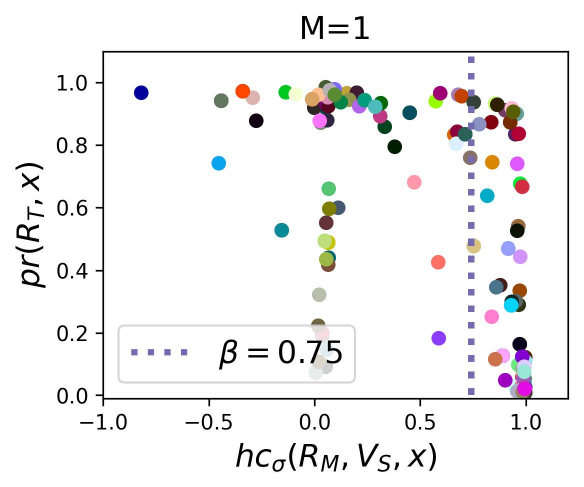}
\includegraphics[width=0.24\linewidth]{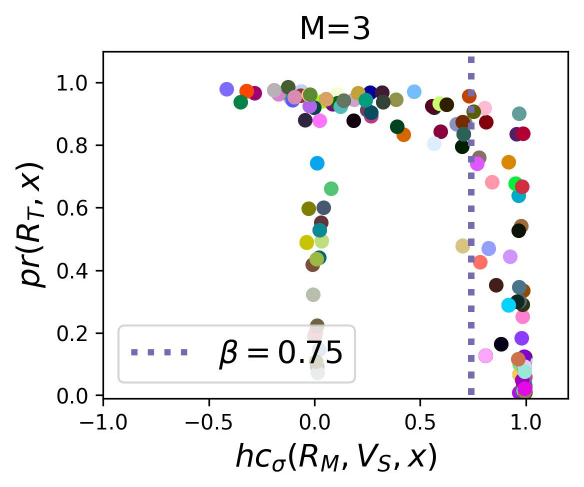}
\includegraphics[width=0.24\linewidth]{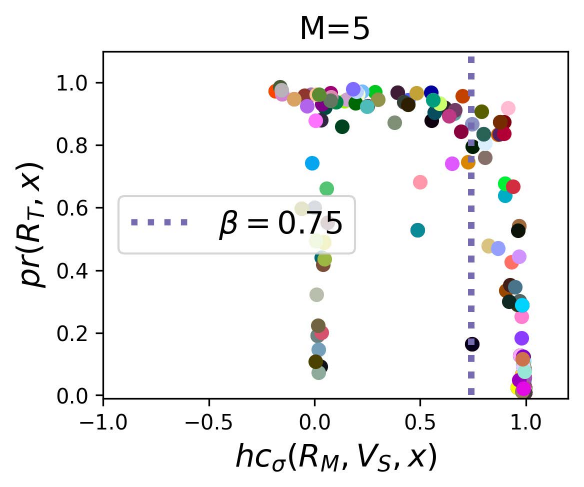}
\includegraphics[width=0.24\linewidth]{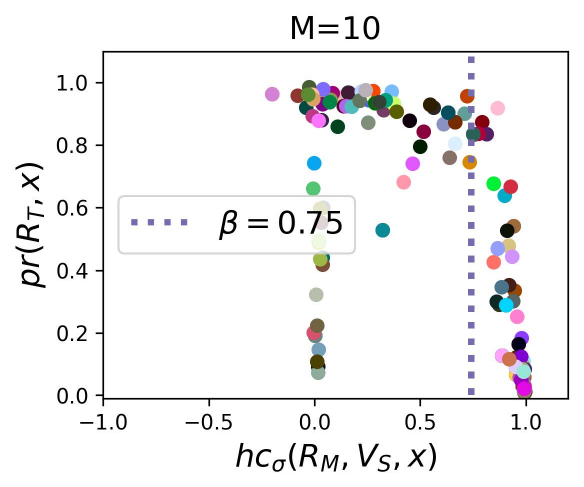}
\end{minipage}

\vspace{1ex}

\begin{minipage}[c]{0.001\textwidth}
\centering
\rotatebox{90}{\scriptsize Prediction variance}
\end{minipage}
\begin{minipage}[c]{0.95\textwidth}
\centering
\includegraphics[width=0.24\linewidth]{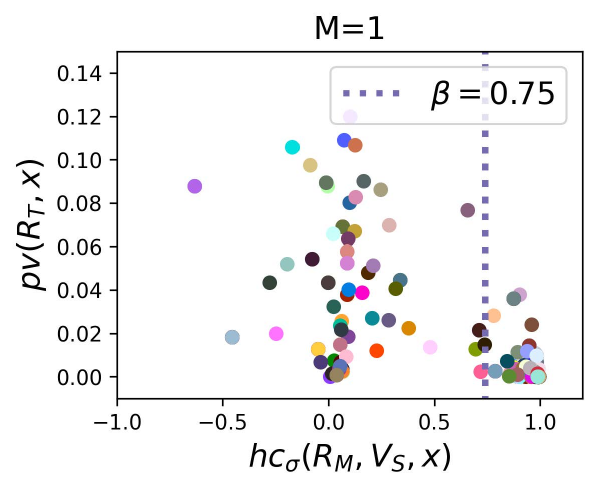}
\includegraphics[width=0.24\linewidth]{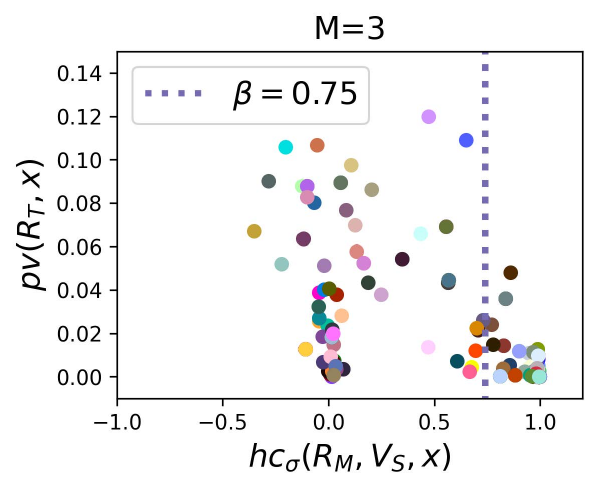}
\includegraphics[width=0.24\linewidth]{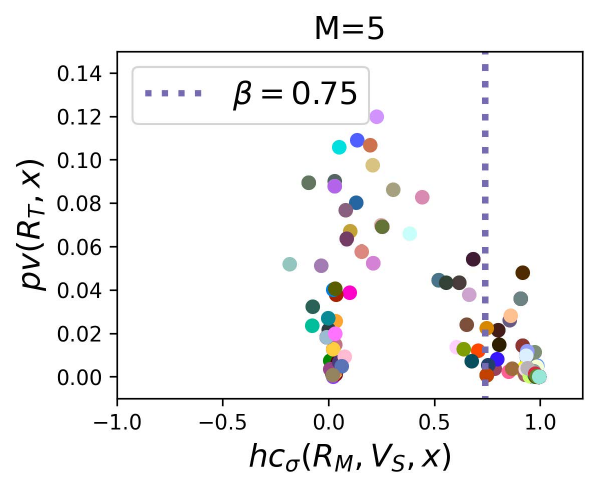}
\includegraphics[width=0.24\linewidth]{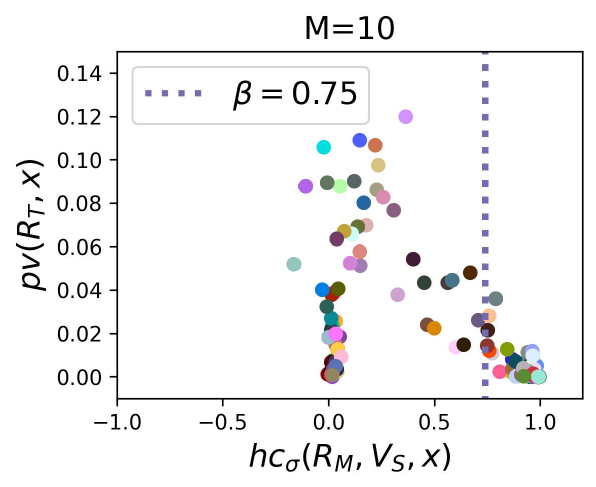}
\end{minipage}
\caption{
Scatter plot showing the relationship between $\hammancons$ from~\citet{hamman2024quantifying}, and four different multiplicity metrics (rows) computed over the Rashomon \largeset, and defined in Section~\ref{sec: Multiplicity metrics}, for different retail set sizes (columns). Each colored dot represents one of 200 randomly sampled test instances. Unlike our proposed measure $\consistency$, the $\hammancons$ measure assigns both high and low scores to test instances with low multiplicity, incorrectly identifying consistent test instances as inconsistent. Results are shown for paraphrase detection on the \textbf{MRPC} dataset using an ensemble of independently fine-tuned \textbf{BERT} models. 
}
\label{fig:bert_mrpc_all_metrics_hamman}
\end{figure*}

\subsubsection{Additional correlation tables from Section \ref{sec: Comp consistency multiplicity}}\label{sec: Comp consistency multiplicity app}

In Section~\ref{sec: Comp consistency multiplicity}, we evaluated whether consistency measures computed from a finite retail ensemble do well in capturing predictive multiplicity metrics computed using the Rashomon \largeset{}. We compared the proposed measure with related measures like margin ($\margin(x)$, Definition~\ref{def: margin}), local consistency ($\localmarginconsistency(c)$, Definition~\ref{def: local consistency}), local margin ($\localmargin(x)$, Definition~\ref{def: Local margin}), local variability ($\localvar(x)$, Definition~\ref{def: Local variability}), and the consistency measure $\hammancons(x)$ of~\citet{hamman2024quantifying}, using Spearman correlation with five warehouse multiplicity metrics: pairwise disagreement \citep{black2022model}, discrepancy \citep{marx2020predictive}, local discrepancy, prediction variance, and prediction range \citep{hamman2024quantifying,watson2023predictive}. Tables~\ref{tab:roberta_mrpc_corr},~\ref{tab:roberta_sst2_corr},~\ref{tab:bank_corr}, and~\ref{tab:adult_corr} extend this analysis to a paraphrase detection task on the MRPC dataset and sentiment classification on the SST2 dataset using independently fine-tuned RoBERTa models, and to a binary classification task on the Bank and Adult datasets using independently fine-tuned BIG-SCIENCE T0 models. Consistency measures are computed from a single retail ensemble of size $\Nretail=1,5,10$, while multiplicity metrics are computed using the full Rashomon warehouse.

Across these settings, we observe the same qualitative trends as in Section~\ref{sec: Comp consistency multiplicity}. The proposed consistency measure exhibits strong negative correlations with warehouse multiplicity metrics, and the magnitudes of these correlations increase with the ensemble size. Thus, higher consistency corresponds to lower predictive multiplicity, while lower consistency identifies instances with greater disagreement across the Rashomon warehouse. The strongest correlations are observed with prediction variance and prediction range, which capture variability in model confidence.

The results also highlight the benefit of combining ensemble margin with local prediction variability. Measures based only on margin or local variability generally exhibit weaker correlations than the proposed consistency measure and local consistency, suggesting that both proximity to the decision boundary and local stability are important for characterizing predictive multiplicity.

\begin{table*}[t]
\caption{Absolute Spearman correlation between consistency measures and multiplicity metrics when applied to the \eqref{tab:roberta_mrpc_corr} paraphrase detection task on the \textbf{MRPC} dataset, and \eqref{tab:roberta_sst2_corr} sentiment classification task on the \textbf{SST2} dataset using an ensemble of independently fine-tuned \textbf{RoBERTa} models. Higher values indicate greater agreement between the consistency measures (defined in Appendix~\ref{sec: Related measures}) computed for small retail ensembles and the multiplicity metrics (defined in Appendix~\ref{sec: Multiplicity metrics}) computed from the entire Rashomon warehouse set. Rows highlighted in light blue correspond to the proposed consistency measures, $\consistency(x)$ and $\localmarginconsistency(x)$.} 
\label{tab:correlation_with_multiplicity_metrics_Roberta}
\centering

\begin{subtable}[t]{0.56\textwidth}
\centering
\scriptsize
\setlength{\tabcolsep}{2.5pt}

\caption{\textbf{MRPC}}
\label{tab:roberta_mrpc_corr}

\resizebox{\linewidth}{!}{
\begin{tabular}{|p{1.6cm}|c|c|c|c|c|}
\hline
\multirow{4}{1.6cm}{\centering
\textbf{Consistency}\\[-0.5mm]
\textbf{measures}\\[-0.5mm]
\textbf{using \smallset}\\[-0.5mm]
\textbf{ensembles}}
&
\multicolumn{5}{c|}{}\\
&
\multicolumn{5}{c|}{\textbf{Multiplicity metrics}}\\
&
\multicolumn{5}{c|}{\textbf{using Rashomon \largeset{}}}\\
\cline{2-6}
&
$\pairwisedisag(x)$ &
$\disagens(x)$ &
$\localdisagens(x)$ &
$\predvar(x)$ &
$\predrange(x)$\\
\hline
\rowcolor{ourgrey}
\multicolumn{6}{|c|}{$\Nretail=1$}\\
\hline
\rowcolor{ourblue}
$\consistency(x)$ & $\mathbf{0.777}$ & $\mathbf{0.777}$ & $\mathbf{0.783}$ & $\mathbf{0.864}$ & $\mathbf{0.844}$\\
\hline
$\margin(x)$ & $0.763$ & $0.763$ & $0.769$ & $0.857$ & $0.835$\\
\hline
\rowcolor{ourblue}
$\localmarginconsistency(x)$ & $0.765$ & $0.765$ & $0.777$ & $0.849$ & $0.829$\\
\hline
$\localmargin(x)$ & $0.749$ & $0.749$ & $0.765$ & $0.838$ & $0.817$\\
\hline
$\localvar(x)$ & $0.734$ & $0.734$ & $0.739$ & $0.800$ & $0.782$\\
\hline
$\hammancons(x)$ & $0.558$ & $0.558$ & $0.561$ & $0.698$ & $0.689$\\
\hline
\rowcolor{ourgrey}
\multicolumn{6}{|c|}{$\Nretail=5$}\\
\hline
\rowcolor{ourblue}
$\consistency(x)$ & $\mathbf{0.875}$ & $\mathbf{0.875}$ & $0.888$ & $\mathbf{0.952}$ & $\mathbf{0.927}$\\
\hline
$\margin(x)$ & $0.873$ & $0.873$ & $0.883$ & $0.951$ & $0.924$\\
\hline
\rowcolor{ourblue}
$\localmarginconsistency(x)$ & $0.864$ & $0.864$ & $0.888$ & $0.942$ & $0.916$\\
\hline
$\localmargin(x)$ & $0.860$ & $0.860$ & $\mathbf{0.888}$ & $0.938$ & $0.911$\\
\hline
$\localvar(x)$ & $0.808$ & $0.808$ & $0.829$ & $0.889$ & $0.872$\\
\hline
$\hammancons(x)$ & $0.582$ & $0.582$ & $0.595$ & $0.733$ & $0.723$\\
\hline
\rowcolor{ourgrey}
\multicolumn{6}{|c|}{$\Nretail=10$}\\
\hline
\rowcolor{ourblue}
$\consistency(x)$ & $\mathbf{0.891}$ & $\mathbf{0.891}$ & $0.896$ & $\mathbf{0.968}$ & $\mathbf{0.944}$\\
\hline
$\margin(x)$ & $0.887$ & $0.887$ & $0.887$ & $0.966$ & $0.941$\\
\hline
\rowcolor{ourblue}
$\localmarginconsistency(x)$ & $0.882$ & $0.882$ & $\mathbf{0.902}$ & $0.960$ & $0.933$\\
\hline
$\localmargin(x)$ & $0.877$ & $0.877$ & $0.899$ & $0.956$ & $0.928$\\
\hline
$\localvar(x)$ & $0.858$ & $0.858$ & $0.878$ & $0.937$ & $0.918$\\
\hline
$\hammancons(x)$ & $0.588$ & $0.588$ & $0.599$ & $0.743$ & $0.734$\\
\hline
\end{tabular}
}
\end{subtable}
\begin{subtable}[t]{0.39\textwidth}
\centering
\scriptsize
\setlength{\tabcolsep}{2.5pt}

\caption{\textbf{SST2}}
\label{tab:roberta_sst2_corr}

\resizebox{\linewidth}{!}{%
\begin{tabular}{|c|c|c|c|c|}
\hline
\multicolumn{5}{|c|}{}\\
\multicolumn{5}{|c|}{\textbf{Multiplicity metrics}}\\
\multicolumn{5}{|c|}{\textbf{using Rashomon \largeset{}}}\\
\hline
$\pairwisedisag(x)$ &
$\disagens(x)$ &
$\localdisagens(x)$ &
$\predvar(x)$ &
$\predrange(x)$\\
\hline
\rowcolor{ourgrey}
\multicolumn{5}{|c|}{$\Nretail=1$}\\
\hline
\rowcolor{ourblue}
$\mathbf{0.661}$ & $\mathbf{0.661}$ & $0.679$ & $\mathbf{0.866}$ & $\mathbf{0.845}$\\
\hline
$0.631$ & $0.631$ & $0.646$ & $0.845$ & $0.820$\\
\hline
\rowcolor{ourblue}
$\mathbf{0.665}$ & $\mathbf{0.665}$ & $\mathbf{0.683}$ & $0.852$ & $0.832$\\
\hline
$0.639$ & $0.639$ & $0.666$ & $0.821$ & $0.802$\\
\hline
$0.647$ & $0.647$ & $0.664$ & $0.819$ & $0.797$\\
\hline
$0.287$ & $0.287$ & $0.295$ & $0.530$ & $0.507$\\
\hline
\rowcolor{ourgrey}
\multicolumn{5}{|c|}{$\Nretail=5$}\\
\hline
\rowcolor{ourblue}
$\mathbf{0.711}$ & $\mathbf{0.711}$ & $0.725$ & $\mathbf{0.941}$ & $\mathbf{0.922}$\\
\hline
$0.698$ & $0.698$ & $0.703$ & $0.897$ & $0.872$\\
\hline
\rowcolor{ourblue}
$0.705$ & $0.705$ & $\mathbf{0.726}$ & $0.937$ & $0.922$\\
\hline
$0.697$ & $0.697$ & $0.721$ & $0.930$ & $0.914$\\
\hline
$0.703$ & $0.703$ & $0.720$ & $0.925$ & $0.910$\\
\hline
$0.210$ & $0.210$ & $0.210$ & $0.475$ & $0.452$\\
\hline
\rowcolor{ourgrey}
\multicolumn{5}{|c|}{$\Nretail=10$}\\
\hline
\rowcolor{ourblue}
$\mathbf{0.730}$ & $\mathbf{0.730}$ & $0.749$ & $\mathbf{0.962}$ & $0.943$\\
\hline
$0.727$ & $0.727$ & $0.734$ & $0.929$ & $0.908$\\
\hline
\rowcolor{ourblue}
$\mathbf{0.725}$ & $\mathbf{0.725}$ & $\mathbf{0.750}$ & $0.960$ & $\mathbf{0.944}$\\
\hline
$0.711$ & $0.711$ & $0.738$ & $0.947$ & $0.932$\\
\hline
$0.723$ & $0.723$ & $0.747$ & $0.959$ & $0.943$\\
\hline
$0.284$ & $0.284$ & $0.301$ & $0.563$ & $0.541$\\
\hline
\end{tabular}}
\end{subtable}
\end{table*}

\begin{table*}[t]
\caption{Absolute Spearman correlation between consistency measures and multiplicity metrics when applied to a binary classification task on the \eqref{tab:bank_corr} \textbf{Bank} dataset, and the \eqref{tab:adult_corr} \textbf{Adult} dataset using an ensemble of independently fine-tuned \textbf{BIG-SCIENCE T0} models. Higher values indicate greater agreement between the consistency measures (defined in Appendix~\ref{sec: Related measures}) computed for small retail ensembles and the multiplicity metrics (defined in Appendix~\ref{sec: Multiplicity metrics}) computed from the entire Rashomon warehouse set. Rows highlighted in light blue correspond to the proposed consistency measures, $\consistency(x)$ and $\localmarginconsistency(x)$.}

\label{tab:correlation_with_multiplicity_metrics_bigscience}
\centering

\begin{subtable}[t]{0.558\textwidth}
\centering
\scriptsize
\setlength{\tabcolsep}{2.5pt}

\caption{\textbf{Bank}}
\label{tab:bank_corr}

\resizebox{\linewidth}{!}{
\begin{tabular}{|p{1.6cm}|c|c|c|c|c|}
\hline
\multirow{4}{1.6cm}{\centering
\textbf{Consistency}\\[-0.5mm]
\textbf{measures}\\[-0.5mm]
\textbf{using \smallset}\\[-0.5mm]
\textbf{ensembles}}
&
\multicolumn{5}{c|}{}\\
&
\multicolumn{5}{c|}{\textbf{Multiplicity metrics}}\\
&
\multicolumn{5}{c|}{\textbf{using Rashomon \largeset{}}}\\
\cline{2-6}
&
$\pairwisedisag(x)$ &
$\disagens(x)$ &
$\localdisagens(x)$ &
$\predvar(x)$ &
$\predrange(x)$\\
\hline
\rowcolor{ourgrey}
\multicolumn{6}{|c|}{$\Nretail=1$}\\
\hline
\rowcolor{ourblue}
$\consistency(x)$ & $0.792$ & $0.792$ & $0.803$ & $0.892$ & $0.869$\\
\hline
$\margin(x)$ & $0.791$ & $0.791$ & $0.802$ & $0.890$ & $0.868$\\
\hline
\rowcolor{ourblue}
$\localmarginconsistency(x)$ & $\mathbf{0.795}$ & $\mathbf{0.795}$ & $\mathbf{0.807}$ & $\mathbf{0.895}$ & $\mathbf{0.873}$\\
\hline
$\localmargin(x)$ & $0.794$ & $0.794$ & $0.806$ & $0.893$ & $0.871$\\
\hline
$\localvar(x)$ & $0.530$ & $0.530$ & $0.533$ & $0.646$ & $0.625$\\
\hline
$\hammancons(x)$ & $0.068$ & $0.068$ & $0.057$ & $0.126$ & $0.118$\\
\hline
\rowcolor{ourgrey}
\multicolumn{6}{|c|}{$\Nretail=5$}\\
\hline
\rowcolor{ourblue}
$\consistency(x)$ & $0.840$ & $0.840$ & $0.857$ & $0.940$ & $0.919$\\
\hline
$\margin(x)$ & $0.840$ & $0.840$ & $0.856$ & $0.939$ & $0.918$\\
\hline
\rowcolor{ourblue}
$\localmarginconsistency(x)$ & $\mathbf{0.856}$ & $\mathbf{0.856}$ & $\mathbf{0.869}$ & $\mathbf{0.957}$ & $\mathbf{0.933}$\\
\hline
$\localmargin(x)$ & $0.856$ & $0.856$ & $0.869$ & $0.956$ & $0.932$\\
\hline
$\localvar(x)$ & $0.631$ & $0.631$ & $0.647$ & $0.763$ & $0.745$\\
\hline
$\hammancons(x)$ & $0.061$ & $0.061$ & $0.048$ & $0.120$ & $0.112$\\
\hline
\rowcolor{ourgrey}
\multicolumn{6}{|c|}{$\Nretail=10$}\\
\hline
\rowcolor{ourblue}
$\consistency(x)$ & $\mathbf{0.859}$ & $\mathbf{0.859}$ & $\mathbf{0.875}$ & $\mathbf{0.950}$ & $\mathbf{0.932}$\\
\hline
$\margin(x)$ & $0.858$ & $0.858$ & $0.875$ & $0.949$ & $0.932$\\
\hline
\rowcolor{ourblue}
$\localmarginconsistency(x)$ & $\mathbf{0.859}$ & $\mathbf{0.859}$ & $\mathbf{0.875}$ & $\mathbf{0.950}$ & $\mathbf{0.932}$\\
\hline
$\localmargin(x)$ & $0.858$ & $0.858$ & $0.875$ & $0.948$ & $0.931$\\
\hline
$\localvar(x)$ & $0.666$ & $0.666$ & $0.681$ & $0.783$ & $0.771$\\
\hline
$\hammancons(x)$ & $0.058$ & $0.058$ & $0.046$ & $0.116$ & $0.110$\\
\hline
\end{tabular}
}
\end{subtable}%
\begin{subtable}[t]{0.39\textwidth}
\centering
\scriptsize
\setlength{\tabcolsep}{2.5pt}

\caption{\textbf{Adult}}
\label{tab:adult_corr}

\resizebox{\linewidth}{!}{%
\begin{tabular}{|c|c|c|c|c|}
\hline
\multicolumn{5}{|c|}{}\\
\multicolumn{5}{|c|}{\textbf{Multiplicity metrics}}\\
\multicolumn{5}{|c|}{\textbf{using Rashomon \largeset{}}}\\
\hline
$\pairwisedisag(x)$ &
$\disagens(x)$ &
$\localdisagens(x)$ &
$\predvar(x)$ &
$\predrange(x)$\\
\hline
\rowcolor{ourgrey}
\multicolumn{5}{|c|}{$\Nretail=1$}\\
\hline
\rowcolor{ourblue}
$\mathbf{0.734}$ & $\mathbf{0.733}$ & $\mathbf{0.779}$ & $\mathbf{0.835}$ & $\mathbf{0.828}$\\
\hline
$0.728$ & $0.728$ & $0.750$ & $0.833$ & $0.823$\\
\hline
\rowcolor{ourblue}
$0.713$ & $0.713$ & $\mathbf{0.775}$ & $0.819$ & $0.810$\\
\hline
$0.704$ & $0.704$ & $0.753$ & $0.813$ & $0.801$\\
\hline
$0.592$ & $0.592$ & $0.710$ & $0.694$ & $0.701$\\
\hline
$0.011$ & $0.011$ & $0.045$ & $0.121$ & $0.125$\\
\hline
\rowcolor{ourgrey}
\multicolumn{5}{|c|}{$\Nretail=5$}\\
\hline
\rowcolor{ourblue}
$0.797$ & $0.797$ & $0.823$ & $0.876$ & $0.871$\\
\hline
$\mathbf{0.802}$ & $\mathbf{0.802}$ & $0.800$ & $\mathbf{0.884}$ & $\mathbf{0.876}$\\
\hline
\rowcolor{ourblue}
$0.796$ & $0.796$ & $\mathbf{0.839}$ & $0.870$ & $0.865$\\
\hline
$0.801$ & $0.801$ & $0.827$ & $0.876$ & $0.868$\\
\hline
$0.538$ & $0.538$ & $0.667$ & $0.653$ & $0.661$\\
\hline
$0.009$ & $0.009$ & $0.047$ & $0.132$ & $0.137$\\
\hline
\rowcolor{ourgrey}
\multicolumn{5}{|c|}{$\Nretail=10$}\\
\hline
\rowcolor{ourblue}
$0.865$ & $0.865$ & $0.885$ & $0.946$ & $0.940$\\
\hline
$\mathbf{0.869}$ & $\mathbf{0.869}$ & $0.868$ & $\mathbf{0.950}$ & $\mathbf{0.941}$\\
\hline
\rowcolor{ourblue}
$0.859$ & $0.859$ & $\mathbf{0.895}$ & $0.940$ & $0.932$\\
\hline
$0.864$ & $0.864$ & $0.886$ & $0.944$ & $0.933$\\
\hline
$0.628$ & $0.628$ & $0.744$ & $0.737$ & $0.743$\\
\hline
$0.005$ & $0.005$ & $0.050$ & $0.135$ & $0.137$\\
\hline
\end{tabular}}
\end{subtable}

\end{table*}
\end{document}